\documentclass{article}

\PassOptionsToPackage{numbers, compress}{natbib}

\usepackage[preprint]{neurips_2022}

\usepackage[utf8]{inputenc} 
\usepackage[T1]{fontenc}    
\usepackage{hyperref}       
\usepackage{url}            
\usepackage{booktabs}       
\usepackage{amsfonts}       
\usepackage{nicefrac}       
\usepackage{microtype}      

\usepackage{amsmath,amssymb}
\usepackage{algorithm}
\usepackage{algpseudocode}
\usepackage{xcolor}

\usepackage{amsmath}
\usepackage{amssymb}
\usepackage{mathtools}
\usepackage{amsthm}
\usepackage{dsfont}
\usepackage{subfigure}
\usepackage[capitalize,noabbrev]{cleveref}

\usepackage{url}
\usepackage{enumitem}
\usepackage{layouts}
\usepackage{comment}
\usepackage{graphicx}
\usepackage{amsfonts}
\usepackage{xcolor}

\usepackage{amsmath,amsfonts,bm}

\def\eqref#1{equation~\ref{#1}}

\def\1{\bm{1}}

\DeclareMathAlphabet{\mathsfit}{\encodingdefault}{\sfdefault}{m}{sl}
\SetMathAlphabet{\mathsfit}{bold}{\encodingdefault}{\sfdefault}{bx}{n}

\def\gT{{\mathcal{T}}}

\def\sT{{\mathbb{T}}}

\def\sT{{\mathbb{T}}}

\theoremstyle{plain}
\newtheorem{theorem}{Theorem}
\newtheorem*{theoremstar}{Theorem}
\newtheorem{prop}{Proposition}
\newtheorem{lemma}{Lemma}
\newtheorem*{lemmastar}{Lemma}
\newtheorem{corollary}{Corollary}
\newtheorem*{corollarystar}{Corollary}
\theoremstyle{definition}
\newtheorem{definition}{Definition}
\newtheorem{assumption}{Assumption}

\newtheorem{remarkapp}{Remark}[section]

\usepackage[textsize=tiny]{todonotes}

\DeclareMathOperator{\sgn}{sign}

\DeclareMathOperator{\cA}{\mathcal{A}}
\DeclareMathOperator{\cS}{\mathcal{S}}
\DeclareMathOperator{\cD}{\mathcal{D}}

\title{Neural PPO-Clip Attains Global Optimality: \\ A Hinge Loss Perspective}

\author{%
  Nai-Chieh Huang\textsuperscript{1}, Ping-Chun Hsieh\textsuperscript{1}, Kuo-Hao Ho\textsuperscript{1}, Hsuan-Yu Yao\textsuperscript{2}, Kai-Chun Hu\textsuperscript{1},\\
  \textbf{Liang-Chun Ouyang\textsuperscript{1}, I-Chen Wu\textsuperscript{1,2}}\\
  \textsuperscript{1}Department of Computer Science, National Yang Ming Chiao Tung University, Hsinchu, Taiwan\\
  \textsuperscript{2}Research Center for Information Technology Innovation, Academia Sinica, Taipei, Taiwan\\
  \texttt{\{naich.cs09,pinghsieh\}@nycu.edu.tw}
}

\begin{document}

\maketitle

\begin{abstract}
Policy optimization is a fundamental principle for designing reinforcement learning algorithms, and one example is the proximal policy optimization algorithm with a clipped surrogate objective (PPO-Clip), which has been popularly used in deep reinforcement learning due to its simplicity and effectiveness. Despite its superior empirical performance, PPO-Clip has not been justified via theoretical proof up to date. 
In this paper, we establish the first global convergence rate of PPO-Clip under neural function approximation.
We identify the fundamental challenges of analyzing PPO-Clip and address them with the two core ideas: (i) We reinterpret PPO-Clip from the perspective of hinge loss, which connects policy improvement with solving a large-margin classification problem with hinge loss and offers a generalized version of the PPO-Clip objective. (ii) 
Based on the above viewpoint, we propose a two-step policy improvement scheme, which facilitates the convergence analysis by decoupling policy search from the complex neural policy parameterization with the help of entropic mirror descent and a regression-based policy update scheme. 
Moreover, our theoretical results provide the first characterization of the effect of the clipping mechanism on the convergence of PPO-Clip.
Through experiments, we empirically validate the reinterpretation of PPO-Clip and the generalized objective with various classifiers on various RL benchmark tasks.
\end{abstract}

\section{Introduction}
\label{section:intro}

Policy optimization is a well-known technique for solving reinforcement learning problems, which iteratively updates the parameters to optimize the designated objective. Policy gradient methods were introduced as a technique to directly solve policy optimization problems. The goal is to find an optimal policy that maximizes the total expected reward through interactions with an environment \citep{williams1992simple, sutton1999policy, kakade2001natural}. The step size is an important hyperparameter that significantly affects the performance of policy gradient algorithms.
Due to the difficulty of selecting a proper step size, Trust Region Policy Optimization (TRPO) \citep{schulman2015trust} was formulated to address this issue. 
TRPO leverages the trust-region method with the second-order approximation to attain strict policy improvement. 
In contrast to TRPO, which is computationally expensive, Proximal Policy Optimization (PPO) \citep{schulman2017proximal} enforces policy improvement only via first-order derivatives. 
PPO has two variants, PPO-KL and PPO-Clip, which augment the objective with the Kullback-Leibler divergence as a penalty and the clipping of probability ratio, respectively. 
Both of the above approaches are extraordinary in various environments, and PPO shines in achieving better computational efficiency \citep{chen2018adaptive, ye2020mastering, byun2020proximal}.

Due to the empirical successes of these policy optimization algorithms, a plethora of recent works make important progress in terms of their theoretical guarantees. In particular, \citep{agarwal2019theory, agarwal2020optimality,bhandari2019global} prove the global convergence result of the policy gradient algorithm under several different settings. Moreover, \citep{mei2020global} establishes the convergence rates of the softmax policy gradient in both the standard and the entropy-regularized settings. Moreover, it has been shown that various policy gradient algorithms also enjoy global convergence \citep{fazel2018global,liu2020improved, wang2021global}.
With regard to TRPO and PPO, \citep{shani2020adaptive} leverages the mirror descent method and establishes the convergence rate of adaptive TRPO under the standard and entropy-regularized setting.
Moreover, \citep{liu2019neural} proves the convergence rate of PPO-KL and TRPO under neural function approximation. 
By contrast, despite that PPO-Clip is a computationally efficient and empirically successful method, the following question about the theory of PPO-Clip remains largely open: \textit{Does PPO-Clip enjoy provable global convergence or have any convergence rate guarantee?}

{In this paper, we answer the above question in the affirmative by establishing the first global convergence rate guarantee for a neural variant of PPO-Clip. 
We achieve this by tackling two fundamental challenges of analyzing PPO-Clip under neural function approximation: (i) \textit{PPO-Clip does not have a simple closed-form expression for each policy update due to clipping}: The clipping mechanism essentially introduces an indicator function regarding the probability ratio, and the indicator output could keep changing during an optimization subroutine of the PPO-Clip objective. This makes it challenging to obtain a closed-form expression of the policy update, especially in neural settings. (ii) \textit{The clipping behavior is tightly coupled with the error of neural function approximation due to the design of the PPO-Clip objective:} In neural PPO-Clip, the advantage function is approximately learned by minimizing some temporal difference loss with the help of a neural network. The approximation error of the advantage function would propagate through the PPO-Clip objective and influence the policy improvement step.}

{Given the above challenges, we highlight the core ideas behind our convergence analysis:}
\vspace{-1mm}
\begin{itemize}[leftmargin=*]
    \item {We first reinterpret and generalize the idea of PPO-Clip from the perspective of \textit{hinge loss} by connecting state-wise policy improvement with solving a large-margin classification problem. Specifically, the process of policy improvement in PPO-Clip can be cast as training a binary classifier with hinge loss via empirical risk minimization, where the sign of the advantage function plays the role of the label.
    Accordingly, this reinterpretation facilitates the analysis of the effect of the estimated advantage function on the clipping behavior.
    Interestingly, this reinterpretation also provides a natural way to generalize PPO-Clip by leveraging various types of classifiers.}
    \item {To obtain a tractable expression for the policy update of PPO-Clip, we propose a two-step policy improvement framework that leverages Entropic Mirror Descent Algorithm (EMDA) \citep{beck2003mirror} to \textit{decouple the policy search and policy parameterization} and then utilizes a regression-based policy update scheme for the neural networks. This approach enables an explicit characterization of the effect of the clipping function on the convergence. Moreover, this framework allows us to easily extend the analysis of PPO-Clip to other variants with different classifiers.}
\end{itemize}


\noindent \textbf{Our Contributions.} We summarize the main contributions of this paper as follows:
\vspace{-1mm}
\begin{itemize}[leftmargin=*]
    \item {To establish the global convergence of PPO-Clip, we first identify the inherent challenges imposed by the clipping function and the neural function approximation in analyzing PPO-Clip. To tackle the challenges, we reinterpret PPO-Clip through the lens of hinge loss and present an EMDA-based two-step policy improvement framework, which makes the analysis tractable by decoupling policy search and the policy parameterization.}
    \item {We establish the first global convergence result and explicitly characterize the $O(1/\sqrt{T})$ convergence rate of PPO-Clip and hence provide an affirmative answer to one critical open question about PPO-Clip. Our theoretical results also provide a sharp characterization of the effect of the clipping mechanism on convergence. We also show that the analysis can be readily extended to other variants of PPO-Clip with different classifiers.}
    \item {We also empirically validate the reinterpretation of PPO-Clip with various classifiers through experiments on various RL benchmark tasks. The experimental results further demonstrate the promise of generalizing PPO-Clip from the hinge loss perspective.}
\end{itemize}

\section{Preliminaries}
\label{section:pre}

\textbf{Markov Decision Processes.}
Consider a discounted Markov Decision Process $(\mathcal{S}, \mathcal{A}, \mathcal{P}, R, \gamma, \mu)$, where $\mathcal{S}$ is the state space (possibly \textit{infinite}), $\mathcal{A}$ is a \textit{finite} action space, $\mathcal{P}: \mathcal{S} \times \mathcal{A} \times \mathcal{S} \rightarrow [0, 1]$ is the transition dynamic of the environment, $R: \mathcal{S} \times \mathcal{A} \rightarrow [0, R_{\max}]$ is the bounded reward function, $\gamma \in (0,1)$ is the discount factor, and $\mu$ is the initial state distribution.
Given a policy $\pi: \mathcal{S} \rightarrow \Delta(\mathcal{A})$, where $\Delta(\mathcal{A})$ is the unit simplex over $\mathcal{A}$, we define the state-action value function $Q^{\pi}(\cdot, \cdot) \coloneqq \mathbb{E}_{a_t \sim \pi(\cdot|s_t), s_{t+1} \sim \mathcal{P}(\cdot|s_t, a_t)}[\sum_{t=0}^{\infty} \gamma^t R(s_t, a_t) \rvert s_0 = s, a_0 = a]$.
Moreover, we define $V^{\pi}(s) \coloneqq \mathbb{E}_{a \sim \pi(\cdot\rvert s)}[Q^{\pi}(s, a)]$ and $A^{\pi}(s, a) \coloneqq Q^{\pi}(s, a) - V^{\pi}(s)$.
Also, we denote $\pi^*$ as an optimal policy that attains the maximum total expected reward and denote $\pi_0$ as the uniform policy. We introduce $\nu_{\pi}(s) = (1 - \gamma) \sum_{t=0}^{\infty} \gamma^t \mathbb{P}(s_t = s | s_0 \sim \mu, \pi)$ as the discounted state visitation distribution induced by $\pi$ and $\sigma_{\pi}(s, a) = \nu_{\pi}(s) \cdot \pi(a|s)$ as the state-action visitation distribution induced by $\pi$. In addition, we define the distribution $\nu^*$ and $\sigma^*$ as the discounted state visitation distribution and the state-action visitation distribution induced by the optimal policy $\pi^*$, respectively. Moreover, we define $\tilde{\sigma}_{\pi} = \nu_{\pi} \pi_0$ as the state-action distribution induced by interactions with the environment through $\pi$, sampling actions from the uniform policy $\pi_0$. 
We use $\mathbb{E}_{\nu_{\pi}}[\cdot]$ and $\mathbb{E}_{\sigma_{\pi}}[\cdot]$ as the shorthand notations of $\mathbb{E}_{s \sim \nu_{\pi}}[\cdot]$ and $\mathbb{E}_{(s,a) \sim \sigma_{\pi}}[\cdot]$, respectively.

For the convergence property, we define the total expected reward over the state distribution $\nu^*$ as
\begin{align}
\label{eq:cL}
    \mathcal{L}(\pi) := \mathbb{E}_{\nu^*}[V^{\pi}(s)].
\end{align}
Here, a maximizer of (\ref{eq:cL}) is equivalent to the original definition of the optimal policy $\pi^*$. We will prove the global convergence by analyzing the difference in $\mathcal{L}$ between our policy and the optimal policy and show that the total expected reward monotonically increases.

\vspace{1mm}
\noindent\textbf{Proximal Policy Optimization (PPO).}
PPO is an empirically successful algorithm that achieves monotonic policy improvement by maximizing a surrogate lower bound of the original objective, either through the Kullback-Leibler penalty (termed PPO-KL) or the clipped probability ratio (termed PPO-Clip). PPO-KL and PPO-Clip are the two variants of PPO. In this paper, our focus is PPO-Clip.
Let $\rho_{s, a}(\theta)$ denote the probability ratio $\frac{\pi_{\theta}(a|s)}{\pi_{\theta_{t}}(a|s)}$. PPO-Clip avoids large policy updates by applying a simple heuristic that clips the probability ratio by the clipping range $\epsilon$ and thereby removes the incentive for moving $\rho_{s,a}(\theta)$ away from 1. Specifically, the objective of PPO-Clip is
\begin{align}
\label{eq:clipobject}
    L^{\text{clip}}(\theta) = \mathbb{E}_{\sigma_{t}}[\min\{&\rho_{s, a}(\theta) A^{\pi_{\theta_{t}}}(s, a), \text{clip}(\rho_{s, a}(\theta), 1-\epsilon, 1+\epsilon) A^{\pi_{\theta_{t}}}(s, a)\}].
\end{align}

\noindent\textbf{Neural Networks.}
We present the notations and assumptions about neural networks. Without loss of generality, suppose $(s, a) \in \mathbb{R}^d$ for every $(s, a) \in \mathcal{S} \times \mathcal{A}$. We use two-layer neural network as $\text{NN}(\alpha;m)$ to parameterize our policy $\pi_{\theta}$ and $Q$ function, the parameterized function of $\text{NN}(\alpha;m)$ is
\begin{align}
    u_{\alpha}(s, a) = \frac{1}{\sqrt{m}} \sum_{i=1}^{m} b_i \cdot \sigma([\alpha]_i^{\top} (s, a)),
\end{align}
where $m$ is the width, $\alpha = ([\alpha]_1^{\top}, \dots, [\alpha]_m^{\top})^{\top} \in \mathbb{R}^{md}$ is the input weights where $[\alpha]_i \in \mathbb{R}^d$, $b_i \in \{-1, 1\}$ are the weights of the output, and $\sigma(\cdot)$ is the Rectified Linear Unit (ReLU) activation function. The initialization of $\alpha(0)$ and $b_i$ is as follows:
\begin{align}
\label{init}
    b_i \sim \text{Unif}(\{1, -1\}), [\alpha(0)]_i \sim \mathcal{N}(0, I_d/d),
\end{align}
where both $b_i$ and $[\alpha(0)]_i$ are i.i.d. for each $i \in [m]$. We fixed the $b_i$ after the initialization, we only train for the weights $\alpha$. Also, to ensure that the local linearization properties hold, we use the projection to restrict the training weights $\alpha$ in an $\ell_2$-ball centered at $\alpha(0)$. We denote the ball as $B_{\alpha} = \{ \alpha: \lVert\alpha - \alpha(0) \rVert_2 \le R_{\alpha}\}$ and define the projection as $\prod_{B_{\alpha}}(\alpha') \coloneqq \arg \min_{\alpha \in B_{\alpha}} \lVert\alpha - \alpha'\rVert_2$.

{Our analysis of neural networks relies on the following assumptions, which are both commonly used regularity conditions for neural networks and neural tangent kernel (NTK) in the reinforcement learning literature \citep{liu2019neural, antos2007fitted, munos2008finite, farahmand2010error, farahmand2016regularized, tosatto2017boosted}:}

\begin{assumption}[Q-Value Function Class]
\label{assump:func}
    For any $R>0$, define a function class
    \begin{align}
    \label{function_class}
        \mathcal{F}_{R,m} = \Big\{\frac{1}{\sqrt{m}} \sum_{i=1}^{m} b_i \cdot \mathds{1}\{[\alpha(0)]_i^{\top} (s, a) > 0\} \cdot [\alpha]_i^{\top} (s, a) \Big\},
    \end{align}
    for all $\alpha$ satisfying $\lVert\alpha - \alpha(0)\rVert_2 \le R$, where $b_i$ and $\alpha(0)$ are initialized as (\ref{init}).
    We assume that $Q^{\pi}(s, a) \in \mathcal{F}_{R_Q, m_Q}$ for any policy $\pi$.
\end{assumption}

\begin{assumption}[Regularity of Stationary Distribution]
\label{assump:reg}
    Given any state-action visitation distribution $\sigma_{\pi}$, there exists a upper bounding constant $c > 0$ for any weight vector $z \in \mathbb{R}^d$ and $\zeta > 0$, such that $\mathbb{E}_{\sigma_{\pi}}[\mathds{1}\{|z^{\top}(s, a)| \le \zeta\} | z] \le c \cdot \zeta/ \lVert z \rVert_2$ holds almost surely.
\end{assumption}
Since $\mathcal{T}^{\pi} Q^{\pi}$ is still a $Q$ function, Assumption \ref{assump:func} give us the {closedness} of our function class under the Bellman operator $\mathcal{T}^{\pi}$. Assumption \ref{assump:reg} states that the distribution $\sigma_{\pi}$ is sufficiently regular, which is required to analyze the neural network error. {Moreover, Assumption \ref{assump:reg} can be interpreted as a variant of the commonly-used concentrability-type assumption, under which the state-action visitation  distribution always has an upper bounded distribution \citep{liu2019neural}}. 

\noindent
\textbf{Notations:} We use $\langle a, b \rangle$ and $a \circ b$ to denote the inner product and the Hadamard product, respectively. Let $\mathds{1}[\cdot]$ denote the indicator function. Let $I_d$ denote the $d \times d$ identity matrix.

\section{{Reinterpreting PPO-Clip via Hinge Loss}}
\label{section:HPO}
{To begin with, we first motivate the connection between hinge loss and PPO-Clip by stating the following lemma about state-wise policy improvement \citep{kakade2002,hu2020rethinking}.}

\begin{lemma}
\label{prop:second}
    Given policies $\pi_{1}$ and $\pi_{2}$, $V^{\pi_{1}}(s)\geq V^{\pi_{2}}(s)$ for all $s\in \cS$ if the following holds:
    \begin{equation}
        (\pi_{1}(a|s)-\pi_{2}(a|s))A^{\pi_{2}}(s,a)\geq0,\ \forall(s,a)\in\mathcal{S}\times\mathcal{A}.
        \label{eq:prop 2 condition}
    \end{equation}
\end{lemma}


Notably, Lemma \ref{prop:second} offers a useful insight that policy improvement can be achieved by simply adjusting the action distribution based solely on the \textit{sign of the advantage} of the state-action pairs, regardless of their magnitude. 
Interestingly, one can draw an analogy between (\ref{eq:prop 2 condition}) in Lemma \ref{prop:second} and learning a linear binary classifier: (i) \textit{Features}: The state-action representation can be viewed as the feature vector of a training sample; (ii)
\textit{Labels}: The sign of $A^{\pi_2}(s,a)$ resembles a binary label; (iii) \textit{Classifiers}: $\pi_{1}(a|s)-\pi_{2}(a|s)$ serves as the 
prediction of a linear classifier. 
{Next, we substantiate this insight and rethink PPO-Clip via hinge loss.}

\vspace{-1mm}
\subsection{Connecting PPO-Clip and Hinge Loss}
In PPO-Clip, the policy stops being updated when the probability ratio is out of the clipping range. 
This behavior coincides with the large-margin classification where the classifier intends to “push” the predicted label out of a margin \citep{pi2020low}.
Specifically, the gradient of the clipped objective is indeed the negative of the gradient of hinge loss objective, i.e.,
\begin{align}
    \nonumber&\frac{\partial}{\partial\theta}   \min\{\rho_{s,a}(\theta){A}^{\pi}(s,a),\text{clip}(\rho_{s,a}(\theta),1-\epsilon,1+\epsilon){A}^{\pi}(s,a)\} \\
    &=-\frac{\partial}{\partial\theta}\ \lvert {A}^{\pi}(s,a)\rvert\ \ell(\sgn({A}^{\pi}(s,a)), \rho_{s,a}(\theta)-1, \epsilon),
\end{align}
where $\ell(y_{i}, f_{\theta}(x_i), \epsilon)$ is the hinge loss defined as $\max\{0, \epsilon-y_{i} \cdot f_{\theta}(x_i)\}$, $\epsilon$ is the margin, $y_{i}\in\{-1,1\}$ the label corresponding to the data $x_{i}$, and $f_\theta(x_{i})$ serves as the binary classifier. Please see Appendix \ref{app:compare PPO-Clip and HPO} for a detailed comparison of the two objectives.
Once $y_{i}f_{\theta}(x_i)$ is larger than the margin, $\ell(y_{i}, f_{\theta}(x_i), \epsilon)$ will equal zero, which reflects the sample clipping mechanism in PPO-Clip.
Note that hinge loss has been commonly used for large-margin classification, most notably for support vector machines \citep{freund1999large}.
From the above, maximizing the objective in (\ref{eq:clipobject}) can be rewritten as minimizing the following loss: 
\begin{align}
    \label{eq:hingeobject}
    L(\theta) = \sum_{s\in\mathcal{S}}d_{\mu}^{\pi}(s)\sum_{a\in\mathcal{A}}&\Big(\pi(a|s)\lvert {A}^{\pi}(s,a)\rvert\cdot \ell(\sgn({A}^{\pi}(s,a)), \rho_{s,a}(\theta)-1, \epsilon)\Big). 
\end{align}
In practice, we draw a batch of state-action pairs and use the sample average to approximately minimize the loss in (\ref{eq:hingeobject}).

\vspace{-1mm}
\subsection{{A Generalized PPO-Clip Objective}}
\label{section:generalized obj}
{Based on the above reinterpretation of PPO-Clip, we provide a general form of the PPO-Clip loss function from a hinge loss perspective as follows,}
\begin{equation}
    {L_{\text{Hinge}}(\theta)}=\frac{1}{\lvert\mathcal{D}\rvert}\sum_{(s,a)\in\mathcal{D}}\text{weight}\times\ell(\text{label},\text{classifier},\text{margin}).\label{eq:HPO loss}
\end{equation}
Different combinations of classifiers, margins, and weights lead to different loss functions and hence represents different algorithms{, where PPO-Clip is a special case of (\ref{eq:HPO loss}) with a specific classifier $\rho_{s,a}(\theta)-1$. As another example, a variant of PPO-Clip can be obtained by using a subtraction classifier, i.e., $\pi_{\theta}(a|s) - \pi_{\theta_t}(a|s)$ (termed as PPO-Clip-sub in this paper). 
We demonstrate the empirical evaluation of these variants in Section \ref{section:Discussions}. Given the above examples, the proposed objective provides to generalizing PPO-Clip via various classifiers.
}

{From this perspective, we are able to prove almost-sure asymptotic convergence for PPO-Clip in the tabular setting as a side product. As our main focus is the theoretical guarantee for PPO-Clip under neural function approximation, we defer the results of the tabular case to Appendix \ref{app:mini-batch thm}.}

\section{Neural PPO-Clip}
\label{section:Neural}
{In this section, we provide the details of the variant of PPO-Clip of interest {under neural function approximation}.}

\vspace{-1mm}
\subsection{{EMDA-Based Policy Search}}
\label{section:NeuralHPO:NAPS}
{Given the hinge loss perspective, we proceed to present our two-step policy improvement scheme based on EMDA \citep{beck2003mirror}, which is a classic algorithm for optimizing an objective function subject to simplex constraints, and we call it EMDA-based Policy Search.
Specifically, this scheme consists of two subroutines: }
\begin{itemize}[leftmargin=*]
    \item {\textbf{Direct policy search}:
In this step, we directly search for an improved policy in the policy space by EMDA.
More specifically, in each iteration $t$, we do a policy search by applying EMDA with direct parameterization to minimize the generalized PPO-Clip objective in (\ref{eq:HPO loss}) for finitely many iterations $K$ and thereby obtain an improved policy $\widehat{\pi}_{t+1}$ as the target policy. The pseudo code of EMDA is provided in Algorithm \ref{algo:2}. Notably, under EMDA, we can obtain an explicit expression of the target policy $\widehat{\pi}_{t+1}$. } 
\item {\textbf{Neural approximation for the target policy}: Given the target policy $\widehat{\pi}_{t+1}$ obtained by EMDA, we then approximate it in the parameter space by utilizing the representation power of neural networks via a regression-based policy update scheme (e.g., by using the mean-squared error loss).
The detailed neural parameterization will be described in the next subsection.}
\end{itemize}
{The motivations and benefits of using EMDA with direct parameterization are mainly two-fold:}
\begin{itemize}[leftmargin=*]
    \item {\textbf{Decoupling improvement and approximation:} One major goal of this paper is to provide rigorous theoretical guarantees for PPO-Clip under neural function approximation. To make the analysis tractable and general, we would like to decouple policy improvement and function approximation of the policy. To achieve this, we take the EMDA-based two-step approach, as described above.}
    \item {\textbf{EMDA-induced closed-form expression of the target policy:} For the analysis of a policy optimization method, in general, we would like to derive a closed-form optimal solution to the policy improvement objective as the ideal target policy. However, such a closed-form optimal solution of an \textit{arbitrary} objective function does not always exist. In particular, the loss function of PPO-Clip is an example that does not have a simple closed-form optimal solution. From this view, EMDA, which enjoys closed-form updates, substantially facilitates the convergence analysis.}
\end{itemize}

\subsection{{Neural PPO-Clip}}

\noindent \textbf{Parameterization Setting.} We parameterize our policy at each iteration $t$ as an energy-based policy $\pi_{\theta_t}(a|s) \propto \exp\{\tau_t^{-1} f_{\theta_t}(s, a)\}$, where $\tau_t$ is the temperature parameter and $f_{\theta_t}(s, a) = \text{NN}(\theta_t;m_f)$ are the energy functions, where $m_f$ is the width of neural network $f_\theta$ defined in Section \ref{section:pre}.
Similarly, we parameterize our state-action value function $Q_{\omega}(s, a) = \text{NN}(\omega;m_Q)$, where $m_Q$ is the width of neural network $Q_\omega$. We define $V_{\omega}(s)$ as the value function obtained using the Bellman Expectation Equation. Moreover, we define $A_{\omega}(s, a) := Q_{\omega}(s, a) - V_{\omega}(s)$.

\noindent \textbf{Policy Improvement.}
According to the {EMDA-based Policy Search} framework presented above, we first give the closed-form of the obtained target policy of Neural PPO-Clip as follows. The detailed proof is in Appendix \ref{app:B}.

\begin{prop}[EMDA Target Policy]
\label{pp:PI}
    For the target policy obtained by the EMDA subroutine at the $t$-th iteration, we have
    \begin{align}
        \log \widehat{\pi}_{t+1}(a|s) \propto C_t(s, a) A_{\omega_t}(s, a) + \tau_{t}^{-1} f_{\theta_t}(s, a), 
    \end{align}
    where $C_t(s, a) A_{\omega_t}(s, a) = -\sum_{k=0}^{K-1} \eta g_{s,a}^{(k)}$ as given in Algorithm \ref{algo:2}.
\end{prop}

Recall that the target policy $\widehat{\pi}$ is the direct parameterization in the policy space, but our policy $\pi_{\theta}$ is an energy-based (softmax) policy that is proportional to the exponentiated energy function. 
This explains why we consider the $\log \widehat{\pi}_{t+1}(a|s)$ in Proposition \ref{pp:PI}.
{Another benefit of using EMDA is that it closely matches the energy-based policies considered in Neural PPO-Clip due to the inherent exponentiated gradient update.}

Then, we discuss the details of the neural function approximation of our policy. After obtaining the target policy by Proposition \ref{pp:PI}, we solve the Mean Squared Error (MSE) subproblem with respect to $\theta$ to approximate the target policy as follows:
\begin{align}
    \mathbb{E}_{\tilde{\sigma}_t}[(f_{\theta}(s, a) - \tau_{t+1} (C_t(s, a) A_{\omega_t}(s, a) + \tau_{t}^{-1} f_{\theta_t}(s, a)))^2].
\end{align}
Notice that we consider the state-action distribution $\tilde{\sigma}_t$ which samples the action through a uniform policy $\pi_0$. In this manner, we use more exploratory data to improve our current policy. In particular, we use the SGD to tackle the above subproblem, and the pseudo code is provided in Appendix \ref{app:pseudo_code}.

\noindent \textbf{Policy Evaluation.} To evaluate $Q$, we use a neural network to approximate the true state-action value function $Q^{\pi_{\theta_t}}$ by solving the Mean Square Bellman Error (MSBE) subproblem. The MSBE subproblem is to minimize the following objective with respect to $\omega$ at each iteration $t$:
\begin{align}
    \mathbb{E}_{\sigma_t}[(Q_{\omega}(s, a) - [\mathcal{T}^{\pi_{\theta_t}}  Q_{\omega}](s, a))^2],
\end{align}
where $\mathcal{T}^{\pi_{\theta_t}} $ is the Bellman operator of policy $\pi_{\theta_t}$ such that
\begin{align}
    &[\mathcal{T}^{\pi_{\theta_t}} Q_{\omega}](s, a)\nonumber\\&= \mathbb{E}[r(s, a) + \gamma Q_{\omega}(s',a') \mid s' \sim \mathcal{P}(\cdot|s, a), a' \sim \pi_{\theta_t}(\cdot|s')].
\end{align}
The pseudo code of neural TD update for state-action value function $Q_{\omega}$ is in Appendix \ref{app:pseudo_code}.
It is worth mentioning that this variant of Neural PPO-Clip is not a fully on-policy algorithm. Although we interact with the environment by our current policy, we sample the actions by the uniform policy $\pi_0$ for policy improvement. {We provide the pseudo code of Neural PPO-Clip as the following Algorithm \ref{algo:1 incomplete} (please refer to Algorithm \ref{algo:1} in Appendix \ref{app:pseudo_code} for the complete version) and the pseudo code of EMDA as Algorithm \ref{algo:2}.
The pseudo code of Algorithms \ref{algo:3}-\ref{algo:4} is in Appendix \ref{app:pseudo_code}.}

\vspace{-1mm}
\begin{algorithm}[!htbp]
\caption{Neural PPO-Clip}
\label{algo:1 incomplete}
    \begin{algorithmic}[1]
        \State {\bfseries Input:} $L_{\text{Hinge}}(\theta)$, $T$, $\epsilon$, EMDA step size $\eta$, number of EMDA iterations $K$, number of SGD and TD update iterations $T_{\text{upd}}$;
        \State {\bfseries Initialization:} the policy $\pi_{\theta_0}$ as a uniform policy\;
        \For{$t=1,\cdots,T-1$}
            \State Set temperature parameter $\tau_{t+1}$\;
            \State Sample the tuple $\{s_i, a_i, a_i^0, s_i',a_i'\}_{i=1}^{T_{\text{upd}}}$\;
            \State Run EMDA as Algorithm \ref{algo:2} with $L_{\text{Hinge}}(\theta)$\;
            \State Run TD as Algorithm \ref{algo:3}: $Q_{\omega_t} = \text{NN}(\omega_t;m_{Q})$\;
            \State Calculate $V_{\omega_t}$ and the advantage $A_{\omega_t} = Q_{\omega_t} - V_{\omega_t}$\;
            \State Run SGD as Algorithm \ref{algo:4}: $f_{\theta_{t+1}} = \text{NN}(\theta_{t+1};m_f)$\;
            \State Update the policy $\pi_{\theta_{t+1}} \propto \exp \{ \tau_{t+1}^{-1} f_{\theta_{t+1}}\}$\;
        \EndFor
    \end{algorithmic}
\end{algorithm}
\vspace{-4mm}
\begin{algorithm}[!htbp]
\caption{EMDA}
\label{algo:2}
    \begin{algorithmic}[1]
        \State {\bfseries Input:} $L_{\text{Hinge}}(\theta)$, EMDA step size $\eta$, number of EMDA iterations $K$, initial policy $\pi_{\theta_{t}}$, sample batch $\{s_i\}_{i=1}^{T_{\text{upd}}}$\;
        \State {\bfseries Initialization:} $\tilde{\theta}^{(0)} = \pi_{\theta_{t}}$, $C_t(s, a) = 0$, for all $s, a$\;
        \For{$k=0,\cdots,K-1$}
            \For{\text{each state} $s$ \text{in the batch}}
            \State Find $g_{s,a}^{(k)} = \left.\frac{\partial L_{\text{Hinge}}(\theta)}{\partial \theta_{s, a}}\right|_{\theta = \tilde{\theta}^{(k)}}$, for each $a$\;
            \State Let $w_s = (e^{-\eta g_{s,1}}, \dots, e^{-\eta g_{s,|\mathcal{A}|}})$\;
            \State $\tilde{\theta}^{(k+1)} = \frac{1}{\langle w_s, \tilde{\theta}^{(k)} \rangle} (w_s \circ \tilde{\theta}^{(k)})$\;
            \State $C_t(s, a) \leftarrow C_t(s, a) - \eta g_{s,a}^{(k)} / A_{\omega_t}(s, a)$, for each 
            
            \ \ \ \ \ \ $a$ with $A_{\omega_t}(s, a) \neq 0$\;
            \EndFor
        \EndFor
        \State $\widehat{\pi}_{t+1} = \tilde{\theta}^{(K)}$\;
        \State {\bfseries Output:} Return the policy $\widehat{\pi}_{t+1}$, and $C_t$\;
    \end{algorithmic}
\end{algorithm}

\section{Main Results}
\label{section:analysis}
In this section, we present the convergence analysis of Neural PPO-Clip {as well as the insights about the clipping mechanism provided by our theoretical results}.

\vspace{-1mm}
\subsection{Convergence Guarantee of Neural PPO-Clip}
Inspired by the analysis of \citep{liu2019neural}, we analyze the convergence behavior of Neural PPO-Clip based on the NTK technique. Nevertheless, the analysis presents several unique technical challenges in establishing its convergence: (i) \textit{Tight coupling between function approximation error and the clipping behavior}: The clipping mechanism can be viewed as an indicator function. The function approximation for advantage would significantly influence the value of the indicator function in a highly complex manner. As a result, handling the error between the neural approximated advantage and the true advantage serves as one major challenge in the analysis; (ii) \textit{Lack of a closed-form expression of policy update}: Due to the clipping function in the hinge loss objective and the iterative updates in the EMDA subroutine, the new policy does not have a simple closed-form expression. This is one salient difference between the analysis of Neural PPO-Clip and other neural algorithms (cf. \cite{liu2019neural}); (iii) \textit{NTK technique on advantage function}: Another technicality is that the advantage function requires the NTK projection and linearization properties to characterize the approximation error. However, since we use the neural network to approximate the state-action value function instead of the advantage function, it requires additional effort to establish the error bound of the advantage function.
Throughout this section, we suppose Assumptions \ref{assump:func} and \ref{assump:reg} hold.

Given that we need to analyze the error between our approximation and the true function, we further define the target policy under the true advantage function $A^{\pi_{\theta_t}}$ as {$\pi_{t+1}(a|s) := \bar{C}_t(s, a)A^{\pi_{\theta_t}}(s, a) + \tau_t^{-1} f_{\theta_t}(s, a)$, where $\bar{C}_t(s, a)$ is the $C_t(s,a)$ obtained under $A^{\pi_{\theta_t}}$}.
Moreover, all the expectations about $A_{\omega}$ throughout the analysis are with respect to the randomness of the neural network initialization.
Below we state the convergence rate and the sufficient condition of Neural PPO-Clip, which is also the main theorem of our paper.
\begin{theorem}[{General} Convergence Rate of Neural PPO-Clip]
\label{thm:main}
    Let the followings hold for all $t$, 
    \begin{align}
        \label{suff:1}
        &\text{(i) } L_{C} \cdot |A^{\pi}(s, a)| \le \bar{C}_t(s, a) \cdot |A^{\pi}(s, a)| \le U_{C} \cdot |A^{\pi}(s, a)|, \\
        \label{suff:2}
        &\text{(ii) } L_C = \omega(T^{-1}), U_C = O(T^{-1/2})
    \end{align}
    {where $L_{C}, U_{C} > 0$ are finite values dependent on $T$}. Then, the policy sequence $\{\pi_{\theta_t}\}_{t=0}^{T}$ obtained by Neural PPO-Clip satisfies
    \begin{align}
    \label{thm:main:eq}
        &\min_{0\le t \le T} \{\mathcal{L}(\pi^*) - \mathcal{L}(\pi_{\theta_t})\} \le \frac{\log |\mathcal{A}| + \sum_{t=0}^{T-1} (\varepsilon_t + \varepsilon_t') + T U_{C}^2 (2 \psi^* + M)}{T L_{C} (1 - \gamma)},
    \end{align}
    where $\varepsilon_t = C_{\infty} \tau_{t+1}^{-1} \phi^* \epsilon_{t+1}^{1/2} + Y^{1/2} \psi^* \epsilon_t'^{1/2}$, $\varepsilon_t' = |\mathcal{A}| \cdot C_{\infty} \tau_{t+1}^{-2} \epsilon_{t+1}$, $M = 4\mathbb{E}_{\nu^*}[\max_{a} (Q_{\omega_0}(s, a))^2] + 4R_f^2$, $M' = 4\mathbb{E}_{\nu_t}[\max_{a} (Q_{\omega_0}(s, a))^2] + 4R_f^2$, and $Y = 2M' + 2(R_{\max} / (1 - \gamma))^2$.
    
\end{theorem}
{To demonstrate that our convergence analysis is general for Neural PPO-Clip with various classifiers, we choose to state Theorem \ref{thm:main} in a general form with the help of the condition (\ref{suff:1}), which is \textit{not} a technical assumption for our analysis but a classifier-specific property. Indeed, we show that (\ref{suff:1}) can be naturally satisfied by using the standard PPO-Clip classifier and other classifiers in Corollaries \ref{cor:PPO-Clip}-\ref{cor:sub}. We defer the complete statement to Appendix \ref{app:add:cor}}.
\begin{corollary}[Global Convergence of {Neural PPO-Clip},  Informal]
\label{cor:PPO-Clip}
    Consider Neural PPO-Clip with the standard PPO-Clip classifier $\rho_{s, a}(\theta) - 1$ and the objective function $L^{(t)}(\theta)$ in each iteration $t$ as 
    \begin{align}
         \mathbb{E}_{\nu_t}[\langle \pi_{\theta_t}(\cdot|s), |A^{\pi_{\theta_t}}(s, \cdot)| \circ \ell (\sgn(A^{\pi_{\theta_t}}(s, \cdot)), \rho_{s, \cdot}(\theta) - 1, \epsilon) \rangle].
    \end{align}
    We specify the EMDA step size $\eta = 1 / \sqrt{T}$ and the temperature parameter $\tau_t = \sqrt{T} / (Kt)$. Recall that $K$ is the maximum number of EMDA iterations.
    Let the minimum neural networks' widths $m_f, m_Q$, and the SGD and TD updates $T_{\text{upd}}$ be configured as in Appendix \ref{app:add:cor}, we have
    \begin{align}
        \min_{0\le t \le T} &\{\mathcal{L}(\pi^*) - \mathcal{L}(\pi_{\theta_t})\} \le \frac{\log |\mathcal{A}| + K^2 (2 \psi^* + M) + O(1)}{\sqrt{T} (1 - \gamma)}.
    \end{align}
     {Hence, Neural PPO-Clip has $O(1 / \sqrt{T})$ convergence rate.}
\end{corollary}
Notice that in (\ref{thm:main:eq}), the $\varepsilon_t$ and $\varepsilon_t'$ are the errors induced by policy improvement and policy evaluation, which can be controlled by the neural networks' widths and the number of TD, SGD iterations $T_{\text{upd}}$ and can be arbitrarily small here. Hence, the convergence rate obtained by our analysis is determined by $U_C^2 / L_C$. 
Given the condition in (\ref{suff:1}), we have that the fastest convergence rate is achieved by $L_C = U_C = T^{-1/2}$, and shall be the $O(1 / \sqrt{T})$ convergence rate. 
By Corollary \ref{cor:PPO-Clip}, we know PPO-Clip attains the fastest rate in the generalized PPO-Clip family in the neural function approximation setting. The detailed proof of Corollary \ref{cor:PPO-Clip} is in Appendix \ref{app:add:cor}, where we also provide the convergence rate of Neural PPO-Clip with another classifier.


Before we jump into the supporting lemmas and the proof, we need the assumption about distribution density for analyses. The common theory works \citep{liu2019neural, antos2007fitted, munos2008finite, farahmand2010error, farahmand2016regularized, tosatto2017boosted} have the concentrability assumption, we also have this common regularity condition which is stated as follows:
\begin{assumption}[Concentrability Coefficient and Ratio]
\label{assump:con}
    We define the density ratio between the policy-induced distributions and the policies,
    \begin{align}
        \phi^*_t = \mathbb{E}_{\tilde{\sigma}_t}\big[\left|\frac{d \pi^*}{d\pi_0} - \frac{d \pi_{\theta_t}}{d \pi_0}\right|^2\big]^{\frac{1}{2}}, \psi^*_t = \mathbb{E}_{\sigma_t}\big[\left|\frac{d \sigma^*}{d \sigma_t} - \frac{d \nu^*}{d \nu_t}\right|^2\big]^{\frac{1}{2}},
    \end{align}
    where the above fractions are the Radon–Nikodym Derivatives. Also, we define the concentrability coefficient $C_{\infty}$, which is the upper bound of the density ratio between the optimal state distribution and any state distribution, i.e. $\lVert \nu^* / \nu\rVert_{\infty}< C_{\infty}$ for any $\nu$. {We assume that there exist $\phi^*,\psi^*>0$ such that $\phi^*_t < \phi^*$ and $\psi^*_t < \psi^*$, for all $t$. Besides, the concentrability coefficient $C_{\infty}$ is bounded.}
\end{assumption}

\vspace{-1mm}
\subsection{{Proof Sketch}}

In this section, we present the supporting lemmas and then the proof of Theorem \ref{thm:main}. The detailed proofs of all the supporting lemmas are in Appendix \ref{app:main_thm}. Throughout this section, we assume $L_C, U_C$ satisfy the condition in (\ref{suff:1}), which specifies a lower bound and an upper bound of $\bar{C}_t$.

\begin{lemma}[Error Propagation]
\label{lm:ep}
    Let $\pi_{t+1}$ be the target policy obtained by EMDA with the true advantage. Suppose the policy improvement error 
    \begin{align}
    \label{lm:ep:eq1}
        \mathbb{E}_{\tilde{\sigma}_t}&[(f_{\theta_{t+1}}(s, a) - \tau_{t+1} (C_t(s, a) A_{\omega_t}(s, a) + \tau_t^{-1} f_{\theta_t}(s, a)))^2]
    \end{align}
    is upper bounded by $\epsilon_{t+1}$ and the policy evaluation error
    \begin{align}
    \label{lm:ep:eq2}
        \mathbb{E}_{\sigma_t}[(A_{\omega_t}(s, a) - A^{\pi_{\theta_t}}(s, a))^2]
    \end{align}
    is upper bounded by $\epsilon_t'$. Then,
    \begin{align}
        |\mathbb{E}_{\nu^*}[\langle \log\pi_{\theta_{t+1}}(\cdot|s) - \log  \pi_{t+1}(\cdot|s)&, \pi^*(\cdot|s) - \pi_{\theta_t}(\cdot|s) \rangle]| \le \varepsilon_t + \varepsilon_{\text{err}}
    \end{align}
    where $\varepsilon_t = C_{\infty} \tau_{t+1}^{-1} \phi^* \epsilon_{t+1}^{1/2} + U_{C} X^{1/2} \psi^* \epsilon_t'^{1/2}$ and $\varepsilon_{\text{err}} = 2 U_{C} \epsilon_{\text{err}} \psi^*$, and $X = \left[(2 / \epsilon_{\text{err}}^2)(M' + (R_{\max} / (1 - \gamma))^2 - \epsilon_t'/2)\right]$, and $M' = 4\mathbb{E}_{\nu_t}[\max_{a} (Q_{\omega_0}(s, a))^2] + 4R_f^2$.
\end{lemma}
The above $\epsilon_{t+1}, \epsilon_{t}'$ in (\ref{lm:ep:eq1}) and (\ref{lm:ep:eq2}) are the policy improvement and policy evaluation error and can always be small by properly choosing the width of the neural networks. 
Please see Appendix \ref{app:main_thm} for more details.
We continue to use the conditions in Lemma \ref{lm:ep} and present the following.
\begin{lemma}[Stepwise Energy $\ell_{\infty}$-Difference]
\label{lm:sed}
    \begin{align}
    \label{lm:sed:eq2}
        \mathbb{E}_{\nu^*}[\lVert\tau_{t+1}^{-1} f_{\theta_{t+1}}(s,\cdot) - \tau_{t}^{-1}f_{\theta_{t}}(s,\cdot)\rVert_{\infty}^2] \le 2\varepsilon'_t + 2 U_{C}^2 M,
    \end{align}
    where $\varepsilon'_t = |\mathcal{A}| C_{\infty} \tau_{t+1}^{-2} \epsilon_{t+1}$ and $M = 4\mathbb{E}_{\nu^*}[\max_{a} (Q_{\omega_0}(s, a))^2] + 4 R_f^2$.
\end{lemma}

After showing the lemmas about errors, we give the following lemmas for the suboptimality gap. Roughly speaking, we use the KL divergence difference as the potential function to analyze the suboptimality gap between optimal policy $\pi^*$ and our policy $\pi_{\theta_t}$ in the proof of Theorem \ref{thm:main}.
\begin{lemma}[Stepwise KL Difference]
\label{lm:OSD}
The KL difference is as follows,
\begin{align}
    &\text{KL}(\pi^*(\cdot|s) \rVert \pi_{\theta_{t+1}}(\cdot|s)) - \text{KL}(\pi^*(\cdot|s) \rVert \pi_{\theta_t}(\cdot|s)) \\
    &\le \langle \log\pi_{\theta_{t+1}}(\cdot|s) - \log  \pi_{t+1}(\cdot|s), \pi_{\theta_t}(\cdot|s) -  \pi^*(\cdot|s) \rangle - \langle \bar{C_t}(s, \cdot) \circ A^{\pi_{\theta_t}}(s, \cdot), \pi^*(\cdot|s) - \pi_{\theta_t}(\cdot|s) \rangle \nonumber \\
    &\quad  - \frac{1}{2} \lVert\pi_{\theta_{t+1}}(\cdot|s) - \pi_{\theta_t}(\cdot|s)\rVert_1^2 - \langle \log\pi_{\theta_{t+1}}(\cdot|s) - \log  \pi_{\theta_t}(\cdot|s), \pi_{\theta_t}(\cdot|s) - \pi_{\theta_{t+1}}(\cdot|s) \rangle
\end{align}
\end{lemma}
\begin{lemma}[Performance Difference Using Advantage]
\label{lm:PDL}
Recall that $\mathcal{L}(\pi) = \mathbb{E}_{\nu^*}[V^{\pi}(s)]$. We have 
\begin{align}
    \mathcal{L}(\pi^*) - \mathcal{L}(\pi) = (1 - \gamma) ^ {-1} \mathbb{E}_{\nu^*}[\langle A^{\pi}(s, \cdot), \pi^*(\cdot|s) - \pi(\cdot|s)\rangle].
\end{align}
\end{lemma}

Next, we provide the proof sketch of Theorem \ref{thm:main}. 
Please refer to Appendix \ref{app:main_thm} for the detailed proof. 


\noindent \textbf{Proof Sketch of Theorem \ref{thm:main}.} 
By taking the expectation over $\nu^*$ on the difference between KL in Lemma \ref{lm:OSD}, the following holds by using Hölder's inequality and $2xy - x^2 \le y^2$,
\begin{align}
    &\mathbb{E}_{\nu^*}[\text{KL}(\pi^*(\cdot|s) \rVert \pi_{\theta_{t+1}}(\cdot|s)) - \text{KL}(\pi^*(\cdot|s) \rVert \pi_{\theta_t}(\cdot|s))] \\
    &\le \varepsilon_t + \varepsilon_{\text{err}}
    - \mathbb{E}_{\nu^*}[\langle \bar{C}_t(s,\cdot) \circ A^{\pi_{\theta_t}}(s, \cdot), \pi^*(\cdot | s) - \pi_{\theta_t}(\cdot|s)\rangle] \nonumber \\
    &\quad - \frac{1}{2} \mathbb{E}_{\nu^*}[\lVert\pi_{\theta_{t+1}}(\cdot|s) - \pi_{\theta_{t}}(\cdot|s)\rVert_{1}^{2}] -\mathbb{E}_{\nu^*}[\langle \tau_{t+1}^{-1} f_{\theta_{t+1}}(s, \cdot) - \tau_{t}^{-1} f_{\theta_{t}}(s, \cdot), \pi_{\theta_{t}}(\cdot|s) - \pi_{\theta_{t+1}}(\cdot|s)\rangle] \\
    &\le \varepsilon_t + \varepsilon_{\text{err}}
    - \mathbb{E}_{\nu^*}[\langle \bar{C}_t(s,\cdot) \circ A^{\pi_{\theta_t}}(s, \cdot), \pi^*(\cdot |s) - \pi_{\theta}(\cdot|s) \rangle] \nonumber \\
    &\quad + \frac{1}{2} \mathbb{E}_{\nu^*}[\lVert\tau_{t+1}^{-1} f_{\theta_{t+1}}(s, \cdot) - \tau_{t}^{-1} f_{\theta_{t}}(s, \cdot)\rVert_{\infty}^{2}].
\end{align}

By the Lemma \ref{lm:sed} and rearranging the terms, we obtain that 
\begin{align}
\label{proof_sketch:eq:each_t}
    &\mathbb{E}_{\nu^*}[\langle \bar{C}_t(s,\cdot) \circ A^{\pi_{\theta_t}}(s, \cdot), \pi^*(\cdot | s) - \pi_{\theta_t}(\cdot|s)\rangle] \\
    &\le \mathbb{E}_{\nu^*}[\text{KL}(\pi^*(\cdot|s) \rVert \pi_{\theta_{t}}(\cdot|s)) - \text{KL}(\pi^*(\cdot|s) \rVert \pi_{\theta_{t+1}}(\cdot|s))] + \varepsilon_t + \varepsilon_{\text{err}} + \varepsilon_t' + U_{C}^2 M.
\end{align}

By (\ref{suff:1}), we have $L_{C}  \mathbb{E}_{\nu^*}[\langle A^{\pi_{\theta_t}}(s, \cdot), \pi^*(\cdot|s) - \pi_{\theta_t}(\cdot|s)\rangle] \le \mathbb{E}_{\nu^*}[\langle \bar{C}_t(s,\cdot) \circ A^{\pi_{\theta_t}}(s, \cdot), \pi^*(\cdot | s) - \pi_{\theta_t}(\cdot|s)\rangle]$. 
Then, by Lemma \ref{lm:PDL} and taking the telescoping sum of (\ref{proof_sketch:eq:each_t}) from $t = 0$ to $T-1$, we obtain
\begin{flalign}
    &(1 - \gamma) L_{C} \sum_{t=0}^{T-1} (\mathcal{L}(\pi^*) - \mathcal{L}(\pi_{\theta_t})) &\\ &\le \mathbb{E}_{\nu^*}[\text{KL}(\pi^*(\cdot|s) \rVert \pi_{\theta_{0}}(\cdot|s)) - \text{KL}(\pi^*(\cdot|s) \rVert \pi_{\theta_{T}}(\cdot|s))] + \sum_{t=0}^{T-1} (\varepsilon_t + \varepsilon_{\text{err}} + \varepsilon_t') + M T U_{C}^2.
\end{flalign}
By the facts that (i) $\mathbb{E}_{\nu^*}[\text{KL}(\pi^*(\cdot|s) \rVert \pi_{\theta_{0}}(\cdot|s))] \le \log |\mathcal{A}|$, (ii) KL divergence is nonnegative, (iii) $\sum_{t=0}^{T-1} (\mathcal{L}(\pi^*) - \mathcal{L}(\pi_{\theta_t})) \ge T \cdot \min_{0\le t \le T} \{\mathcal{L}(\pi^*) - \mathcal{L}(\pi_{\theta_t})\}$, we have
\begin{align}
    \min_{0\le t \le T} &\{\mathcal{L}(\pi^*) - \mathcal{L}(\pi_{\theta_t})\} \le \frac{\log |\mathcal{A}| + \sum_{t=0}^{T-1} (\varepsilon_t + \varepsilon_t') + T (\varepsilon_{\text{err}} + M U_{C}^2)}{T L_{C} (1 - \gamma)}.
\end{align}
Since we have $\varepsilon_{\text{err}} = 2 U_{C} \epsilon_{\text{err}} \psi^*$, we set $\epsilon_{\text{err}} = U_{C}$ and thereby obtain the result of (\ref{thm:main:eq}).
\hfill\qedsymbol

\section{Discussions}
\label{section:Discussions}
\begin{figure*}[!ht]
\centering
    \subfigure[Breakout]{\includegraphics[width=0.266\linewidth]{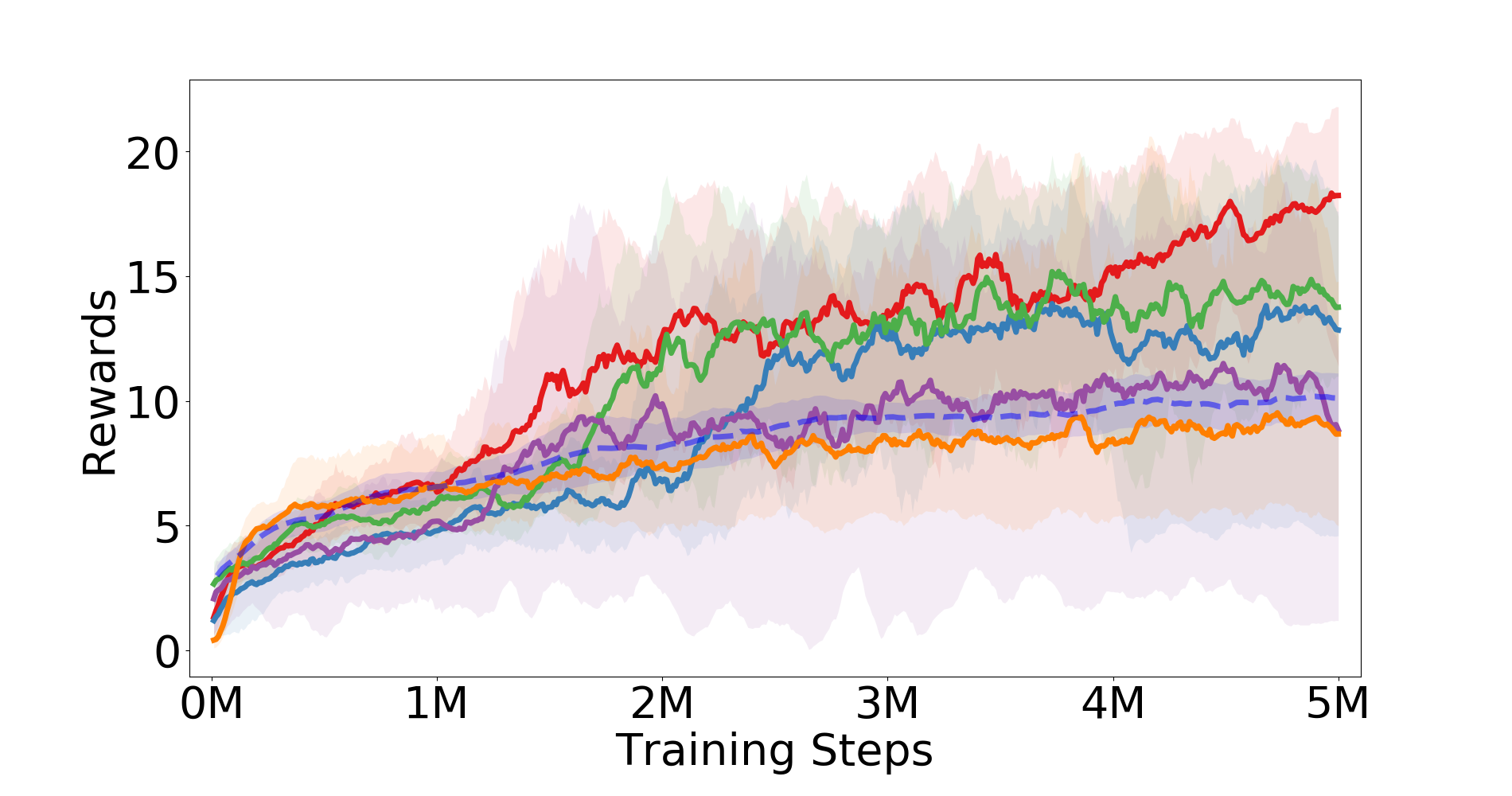}}
    \hfill
\hspace{-7mm}
    \subfigure[Space Invaders]{\includegraphics[width=0.266\linewidth]{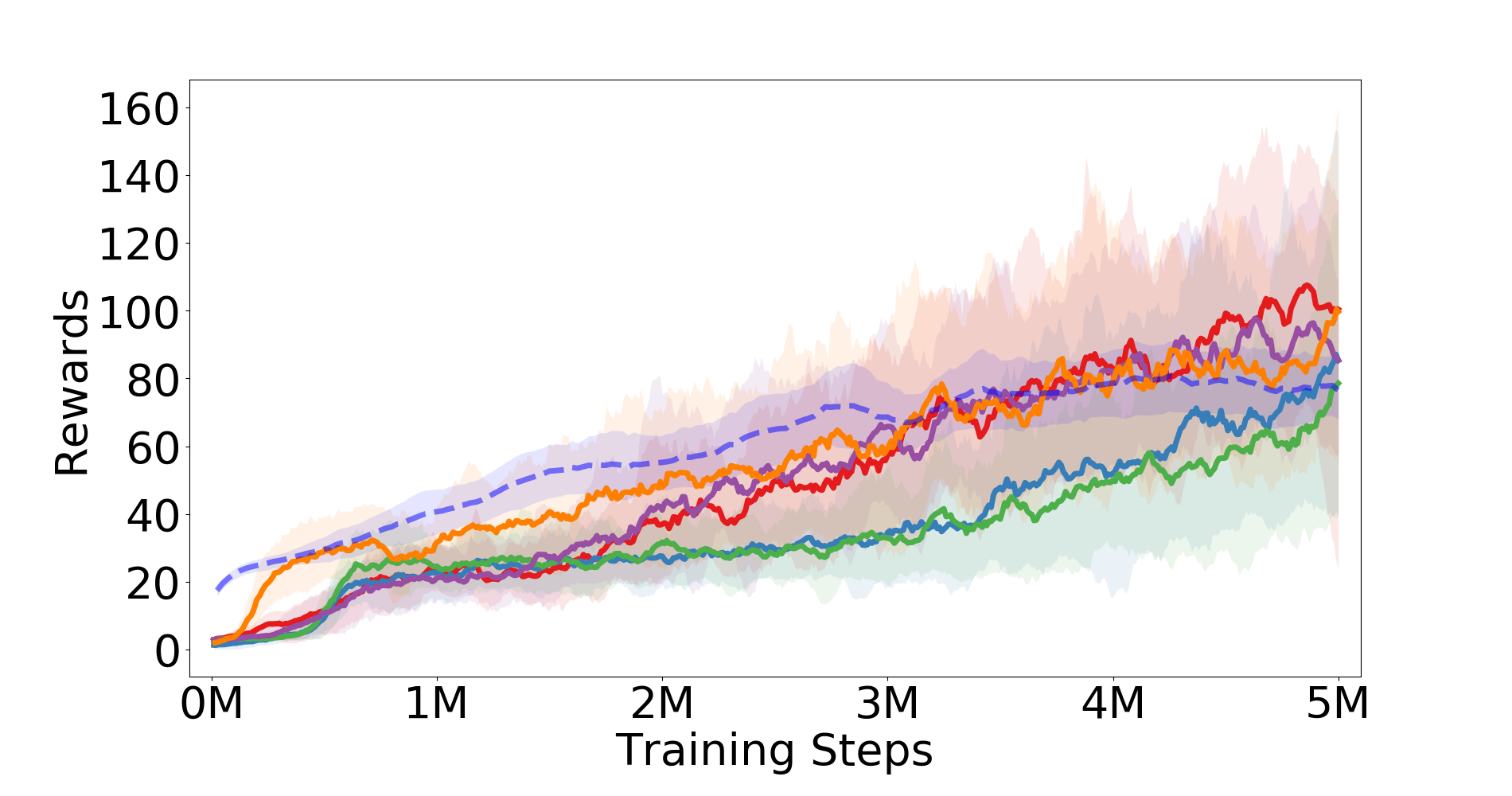}}
\hspace{-7mm}
 \hfill
    \subfigure[LunarLander]{\includegraphics[width=0.266\linewidth]{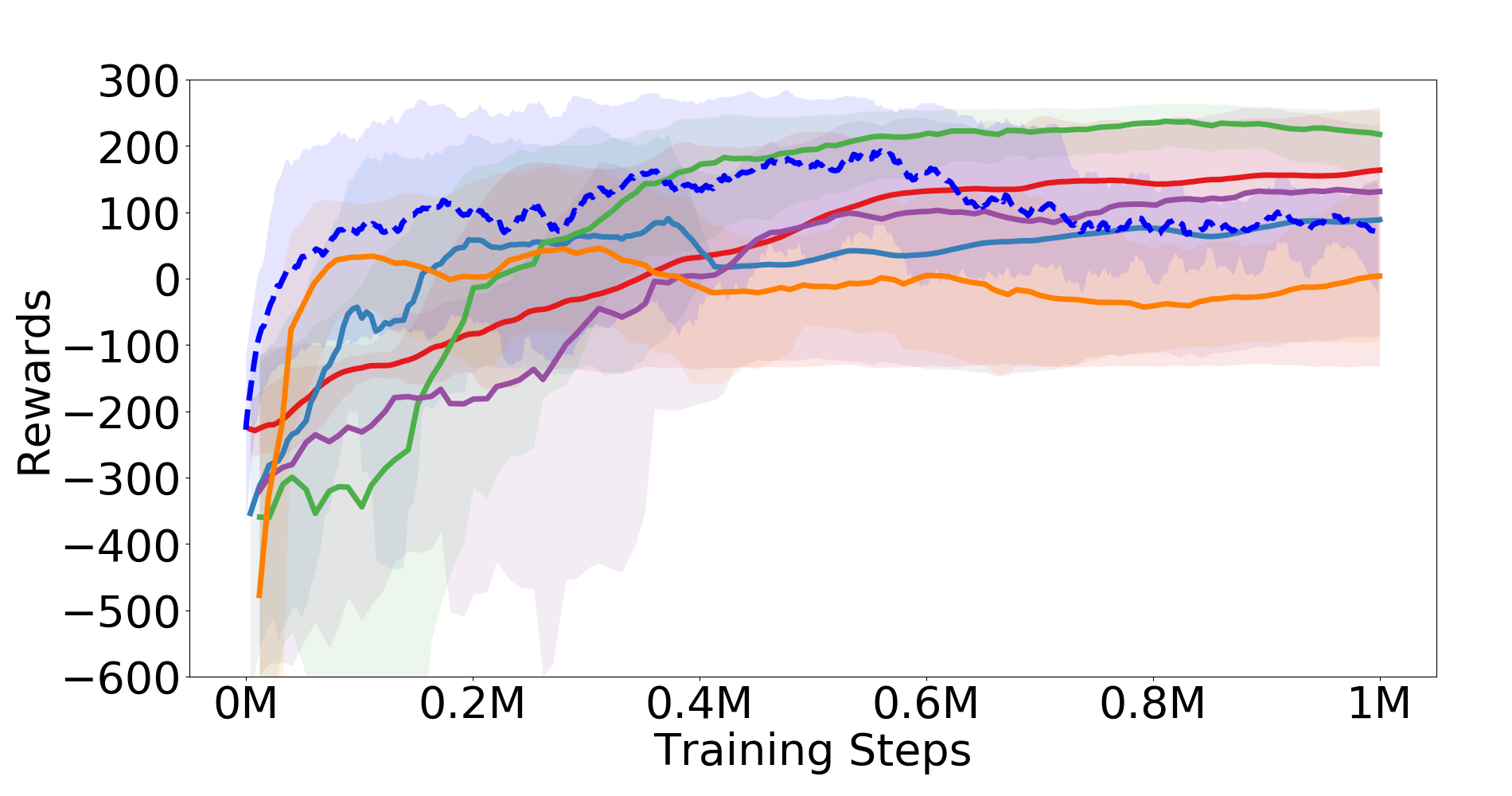}}
\hspace{-5.8mm}
 \hfill
    \subfigure[CartPole]{\includegraphics[width=0.266\linewidth]{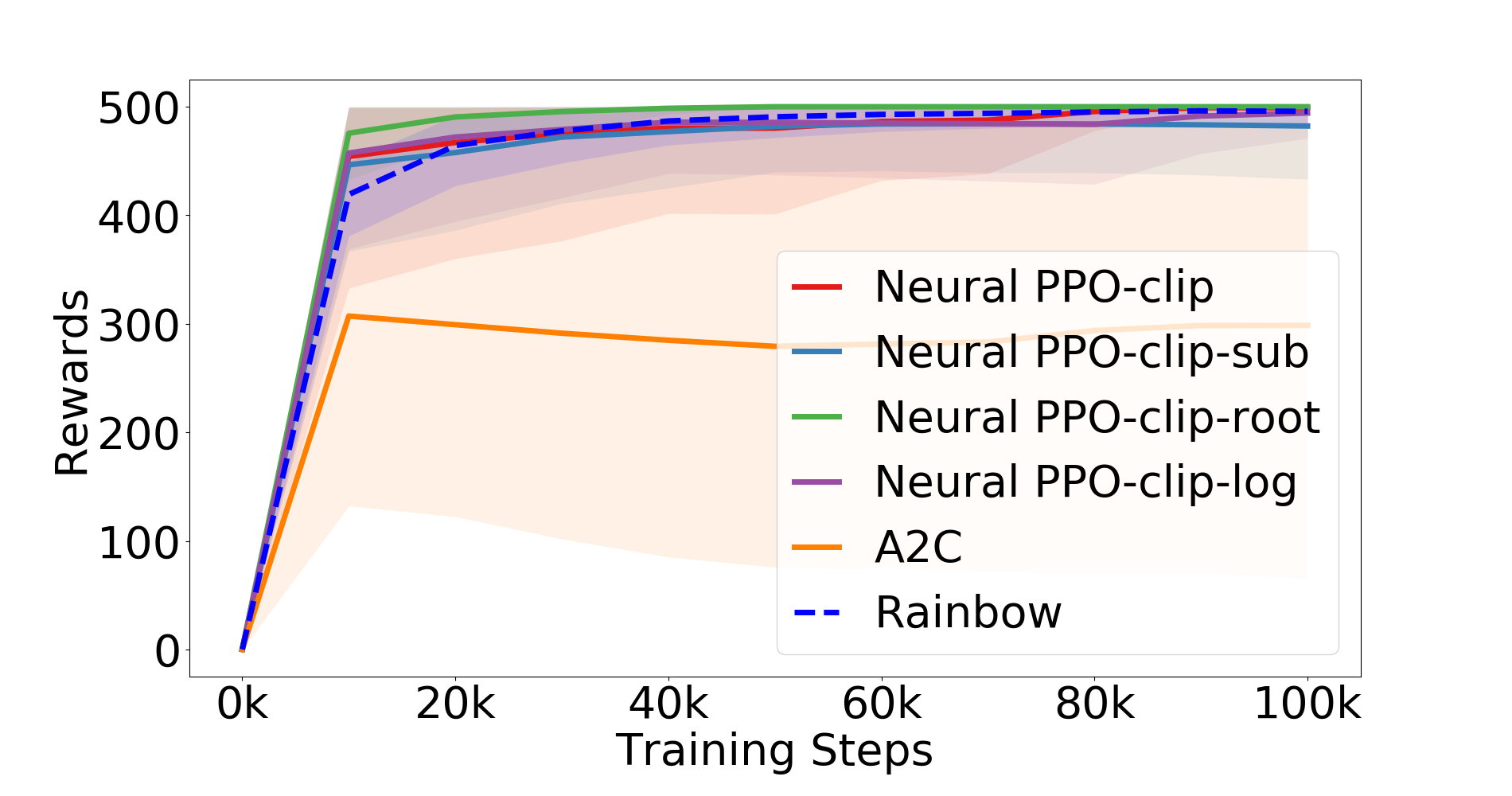}}
    \vspace{-2mm}
\caption{Evaluation of PPO-Clip with different classifiers and popular benchmark methods in MinAtar and OpenAI Gym.}
\vspace{-2mm}
\label{fig:experiments}
\end{figure*}
In this section, we take a deeper look into the convergence results and the generalized PPO-Clip objective.

\noindent\textbf{Our convergence analysis offers a sharp characterization of the effect of the clipping mechanism:} 
{Our analysis and results provide an interesting insight that the clipping range $\epsilon$ only affects the \textit{pre-constant} of the convergence rate of Neural PPO-Clip, and this is a somewhat surprising result since by intuition the clipping range $\epsilon$ is regarded as the counterpart of the penalty parameter of PPO-KL \citep{liu2019neural} which directly affects the convergence rate. On the other hand, we find that the EMDA step size $\eta$ plays an important role in determining the convergence rate, instead of the clipping range $\epsilon$. To illustrate this result, we can see that the clipping mechanism takes part in the EMDA subroutine through the indicator functions in the gradients. Moreover, as the clipping range $\epsilon$ is contained inside the indicator function, \textit{it only influences the number of effective EMDA updates but not the magnitude of each EMDA update}. Since we know that the convergence rate is determined by the magnitude of the gradient updates (i.e., $U_C, L_C$, which is $\eta$-dependent), the clipping range can only affect the pre-constant of the convergence rate and the rate would still be $O(1 / \sqrt{T})$.}

\vspace{1mm}
\noindent{\textbf{The reinterpretation via hinge loss enables variants of PPO-Clip with comparable empirical performance:}} 
{Given the convergence guarantees in Section \ref{section:analysis}, to better understand the empirical behavior of the generalized PPO-Clip objective, we further conduct experiments to evaluate Neural PPO-Clip with different classifiers. Specifically, we evaluate Neural PPO-Clip and Neural PPO-Clip-sub (as introduced in Section \ref{section:generalized obj}) as well as two additional classifiers $\log(\pi_{\theta}(a|s))-\log(\pi_{\theta_t}(a|s))$, $\sqrt{\rho_{s,a}(\theta)}-1$ (termed as Neural PPO-Clip-log and Neural PPO-Clip-root) against the benchmark approaches in several RL benchmark environments. Our implementations of Neural PPO-Clip are based on the RL Baseline3 Zoo framework \citep{rl-zoo3}. We test the algorithms in both MinAtar \citep{young19minatar} and OpenAI Gym environments \citep{brockman2016openai}. In addition, the algorithms are compared with popular baselines, including A2C and Rainbow. For A2C, we use the implementation and the default setting provided by RL Baseline3 Zoo. Regarding Rainbow, we use the same configuration of \citep{ceron2021revisiting}. Please refer to Appendix \ref{app:Experiments} for more details about our experiment settings.}

{Figure \ref{fig:experiments} shows the training curves of Neural PPO-Clip with various classifiers and the benchmark methods. Notably, we observe that Neural PPO-Clip with various classifiers can achieve comparable or better performance than the baseline methods in both RL environments. To be mentioned, the performance of Rainbow is consistent with the results reported by \citep{ceron2021revisiting}. To sum up, the above results demonstrate the applicability of the hinge loss reinterpretation of PPO-Clip in standard RL tasks.}

\section{Concluding Remarks}
\label{section:conclusion}

{The global convergence of PPO-Clip has remained a critical open problem.
This paper establishes the first global convergence result of PPO-Clip and offers new insights into PPO-Clip in two important aspects: (i) The reinterpretation via hinge loss offers a natural way to generalize the PPO-Clip method; (ii) Our convergence rate unveils the interplay between the convergence behavior of PPO-Clip with the clipping mechanism.
We expect that this work would spark further understanding of PPO-Clip in the RL community.}

{
\small
\bibliographystyle{unsrtnat}
\bibliography{aaai}
}

\appendix
\section*{Appendix}
\label{section:appendix}

\section{Pseudo Code of Algorithms}
\label{app:pseudo_code}

\begin{algorithm}[!htbp]
\caption{Neural PPO-Clip (A More Detailed Version of Algorithm \ref{algo:1 incomplete})}
\label{algo:1}
    \begin{algorithmic}[1]
        \State {\bfseries Input:} MDP $(\mathcal{S}, \mathcal{A}, \mathcal{P}, r, \gamma, \mu)$, Objective function $L$, EMDA step size $\eta$, number of EMDA iterations $K$, number of SGD and TD update iterations $T_{\text{upd}}$, number of Neural PPO-Clip iterations $T$, the clipping range $\epsilon$;
        \State {\bfseries Initialization:} the policy $\pi_{\theta_0}$ as a uniform policy\;
        \For{$t=1,\cdots,T-1$}
            \State Set temperature parameter $\tau_{t+1}$\;
            \State Sample the tuple $\{s_i, a_i, a_i^0, s_i',a_i'\}_{i=1}^{T_{\text{upd}}}$, where $(s_i, a_i) \sim \sigma_t, a_i^0 \sim \pi_0(\cdot|s_i)$, $s_i' \sim \mathcal{P}(\cdot|s_i, a_i)$ and $a_i' \sim \pi_{\theta_t}(\cdot|s_i')$\;
            \State {Use the states with nonzero advantage as the batch $\{s_i\}_{i=1}^{T_{\text{upd}}}$ for $L_{\text{Hinge}}(\theta)$ and obtain target policy $\widehat{\pi}_{t+1}$ and $C_t$ by using}
            
            {EMDA in Algorithm \ref{algo:2}}\;
            \State Solve for $Q_{\omega_t} = \text{NN}(\omega_t;m_{Q})$ by using TD update as Algorithm \ref{algo:3}\;
            \State Calculate $V_{\omega_t}$ by Bellman expectation equation and the advantage $A_{\omega_t} = Q_{\omega_t} - V_{\omega_t}$\;
            \State Solve for $f_{\theta_{t+1}} = \text{NN}(\theta_{t+1};m_f)$ by using SGD as Algorithm \ref{algo:4} based on the EMDA result\;
            \State Update the policy $\pi_{\theta_{t+1}} \propto \exp \{ \tau_{t+1}^{-1} f_{\theta_{t+1}}\}$\;
        \EndFor
    \end{algorithmic}
\end{algorithm}
\vspace{-1mm}
\begin{remarkapp}
\label{remark:choice}
{In Neural PPO-Clip, there are various types of classifiers, the choices of the EMDA step size $\eta$ and the temperature parameters $\{\tau_{t}\}$ of the neural networks are important factors to the convergence rate and hence shall be configured properly according to the properties of different classifiers. As a result, we do not specify the specific choices of $\eta$ and $\{\tau_{t}\}$ in the following pseudo code of the generic Neural PPO-Clip. Please refer to Corollaries \ref{cor:PPO-Clip}-\ref{cor:sub} in Appendix \ref{app:add:cor} for the choices of $\eta$ and $\{\tau_{t}\}$ for Neural PPO-Clip with several classifiers including the standard PPO-Clip classifier $\rho_{s, a}(\theta) - 1 = \frac{\pi_{\theta}(a|s)}{\pi_{\theta_{t}}(a|s)} - 1$.}
\end{remarkapp}

For better readability, we restate EMDA (Algorithm \ref{algo:2}) here as Algorithm \ref{algo:2 copy}.
\vspace{-1mm}
\begin{algorithm}[!htbp]
\caption{EMDA}
\label{algo:2 copy}
    \begin{algorithmic}[1]
        \State {\bfseries Input:} Objective function $L(\theta)$, EMDA step size parameter $\eta$, number of EMDA iteration $K$, initial policy $\pi_{\theta_{t}}$, sample batch $\{s_i\}_{i=1}^{T_{\text{upd}}}$\;
        \State {\bfseries Initialization:} $\tilde{\theta}^{(0)} = \pi_{\theta_{t}}$, $C_t(s, a) = 0$, for all $s, a$\;
        \For{$k=0,\cdots,K-1$}
            \For{\text{each state} $s$ \text{in the batch}}
            \State Find $g_{s,a}^{(k)} = \left.\frac{\partial L(\theta)}{\partial \theta_{s, a}}\right|_{\theta = \tilde{\theta}^{(k)}}$, for each $a$\;
            \State Let $w_s = (e^{-\eta g_{s,1}}, \dots, e^{-\eta g_{s,|\mathcal{A}|}})$\;
            \State $\tilde{\theta}^{(k+1)} = \frac{1}{\langle w_s, \tilde{\theta}^{(k)} \rangle} (w_s \circ \tilde{\theta}^{(k)})$\;
            \State $C_t(s, a) \leftarrow C_t(s, a) - \eta g_{s,a}^{(k)} / A_{\omega_t}(s, a)$, for each $a$ with $A_{\omega_t}(s, a) \neq 0$\;
            \EndFor
        \EndFor
        \State $\widehat{\pi}_{t+1} = \tilde{\theta}^{(K)}$\;
        \State {\bfseries Output:} Return the policy $\widehat{\pi}_{t+1}$, and $C_t$\;
    \end{algorithmic}
\end{algorithm}

\begin{algorithm}[!htbp]
\caption{Policy Evaluation via TD}
\label{algo:3}
    \begin{algorithmic}[1]
        \State {\bfseries Input:} MDP $(\mathcal{S}, \mathcal{A}, \mathcal{P}, r, \gamma)$, initial weights $b_i$, $[\omega(0)]_i \ (i \in [m_Q])$, number of iterations $T_{\text{upd}}$, sample $\{s_i, a_i, s_i', a_i\}_{i=1}^{T_{\text{upd}}}$\;
        \State Set the step size $\eta_{\text{upd}} \leftarrow T_{\text{upd}}^{-1/2}$\;
        \For{$t=0,\cdots,T_{\text{upd}}-1$}
            \State $(s,a,s',a') \leftarrow (s_i,a_i,s_i',a_i')$\;
            \State $\omega(t + 1/2) \leftarrow \omega(t) - \eta_{\text{upd}} \cdot (Q_{\omega(t)}(s, a) -  r(s, a) - \gamma Q_{\omega(t)}(s', a')) \cdot \nabla_{\omega} Q_{\omega(t)}(s, a)$\;
            \State $\omega(t+1) \leftarrow \arg \min_{\omega \in \mathcal{B}^0(R_Q)}\{||\omega - \omega(t+1/2)||_2\}$\;
        \EndFor
        \State Take the average over path $\bar{\omega} \leftarrow 1/T_{\text{upd}} \cdot \sum_{t=0}^{T_{\text{upd}}-1} \omega(t)$\;
        \State {\bfseries Output:} $Q_{\bar{\omega}}$\;
    \end{algorithmic}
\end{algorithm}

\begin{algorithm}[!htbp]
\caption{Policy Improvement via SGD}
\label{algo:4}
    \begin{algorithmic}[1]
        \State {\bfseries Input:} MDP $(\mathcal{S}, \mathcal{A}, \mathcal{P}, r, \gamma)$, the current energy function $f_{\theta_t}$, initial weights $b_i$, $[\theta(0)]_i \ (i \in [m_f])$, number of iterations $T_{\text{upd}}$, sample $\{s_i, a_i^0\}_{i=1}^{T_{\text{upd}}}$\;
        \State Set the step size $\eta_{\text{upd}} \leftarrow T_{\text{upd}}^{-1/2}$\;
        \For{$t=0,\cdots,T_{\text{upd}}-1$}
            \State $(s,a) \leftarrow (s_i,a_i^0)$\;
            \State $\theta(t + 1/2) \leftarrow \theta(t) - \eta_{\text{upd}} \cdot (f_{\theta_t}(s,a) - \tau_{t+1} \cdot (C_t(s, a) \cdot A_{\omega_t}(s, a) + \tau_{t}^{-1}f_{\theta_t}(s, a))) \cdot \nabla_{\theta} f_{\theta_t}(s, a)$\;
            \State $\theta(t+1) \leftarrow \arg \min_{\theta \in \mathcal{B}^0(R_f)}\{||\theta - \theta(t+1/2)||_2\}$\;
        \EndFor
        \State Take the average over path $\bar{\theta} \leftarrow 1/T_{\text{upd}} \cdot \sum_{t=0}^{T_{\text{upd}}-1} \theta(t)$\;
        \State {\bfseries Output:} $f_{\bar{\theta}}$\;
    \end{algorithmic}
\end{algorithm}

\section{Proof of Proposition \ref{pp:PI}}
\label{app:B}
We expand the closed-form of the $\log$ of the EMDA target policy,
\begin{align}
    \log \widehat{\pi}_{t+1}(a|s) &= \log \left(\prod_{k=0}^{K^{(t)}-1} \frac{\exp(-\eta g_{s, a}^{(k)})}{\langle w_s, \tilde{\theta}^{(k)} \rangle} \cdot \pi_{\theta_t}(a|s)\right) \\
    &= \sum_{k=0}^{K^{(t)}-1} -\eta g_{s,a}^{(k)} - \sum_{k=0}^{K^{(t)}-1} \log (\langle w_s, \tilde{\theta}^{(k)}  \rangle) + \log \pi_{\theta_t}(a|s) \\
    &= \sum_{k=0}^{K^{(t)}-1} -\eta g_{s,a}^{(k)}  - \sum_{k=0}^{K^{(t)}-1} \log (\langle w_s, \tilde{\theta}^{(k)}  \rangle) + \tau_{t}^{-1} f_{\theta_t}(s, a) - \log (Z_t(s)) \\
    &\propto C_t(s, a) \cdot A_{\omega_t}(s, a) + \tau_{t}^{-1} f_{\theta_t}(s, a).
\end{align}
where $Z_t(s)$ is the normalizing factor of the policy at step $t$. Since both the $\sum_{k=0}^{K^{(t)}-1} \log (\langle w_s, \tilde{\theta}^{(k)}  \rangle)$ and $\log (Z_t(s))$ are state-dependent, we can cancel it under softmax policy. We obtain $C_t(s, a)$ from Algorithm \ref{algo:2} and complete the proof.

\hfill \qedsymbol

\section{Proof of the Supporting Lemmas for Theorem \ref{thm:main}}
\label{app:main_thm}

\subsection{Additional Supporting Lemmas}
{Throughout this section, we slightly abuse the notation that we use $\mathbb{E}_{\text{init}}[\cdot]$ to denote the expectation over the initialization of neural networks. Also, we assume that Assumption \ref{assump:func} and \ref{assump:reg} hold in the following proofs.}
\begin{lemma}[Policy Evaluation Error]
\label{lm:PE_error}
The output $A_{\bar{\omega}} = Q_{\bar{\omega}} - V_{\bar{\omega}}$ of Algorithm \ref{algo:3} and Bellman expectation equation satisfies
\begin{align}
    \mathbb{E}_{\text{init,}\sigma_t}[(A_{\omega_t}(s, a) - A^{\pi_{\theta_t}}(s, a))^2] = O(R_Q^2 T_{\text{upd}}^{-1/2} + R_Q^{5/2} m_Q^{-1/4} + R_Q^3 m_Q^{-1/2}).
\end{align}
\end{lemma}

\noindent To prove Lemma \ref{lm:PE_error}, we start by stating a bound on the error of the estimated state-action value function.

\begin{lemma}[Theorem 4.6 in \citep{liu2019neural}]
\label{thm:4.6_liu}
The output $Q_{\bar{\omega}}$ of Algorithm \ref{algo:3} satisfies
\begin{align}
    \mathbb{E}_{\text{init,}\sigma_t}[(Q_{\omega_t}(s, a) - Q^{\pi_{\theta_t}}(s, a))^2] = O(R_Q^2 T_{\text{upd}}^{-1/2} + R_Q^{5/2} m_Q^{-1/4} + R_Q^3 m_Q^{-1/2}).
\end{align}
\end{lemma}

\begin{proof}[Proof of Lemma \ref{lm:PE_error}]
We are ready to show the policy evaluation error of the advantage function. First, we find the bound of $|A_{\omega_t}(s, a) - A^{\pi_{\theta_t}}(s, a)|$. We have
\begin{align}
    |A_{\omega_t}(s, a) - A^{\pi_{\theta_t}}(s, a)| &= |Q_{\omega_t}(s, a) - V_{\omega_t}(s) - Q^{\pi_{\theta_t}}(s, a) + V^{\pi_{\theta_t}}(s)| \\
    &= \left|Q_{\omega_t}(s, a) - Q^{\pi_{\theta_t}}(s, a) + \sum_{a'} \pi_{\theta_t}(a'|s) \cdot (Q^{\pi_{\theta_t}}(s, a') - Q_{\omega_t}(s, a'))\right| \\
    &= \left|Q_{\omega_t}(s, a) - Q^{\pi_{\theta_t}}(s, a) + \mathbb{E}_{a'\sim \pi_{\theta_t}}[Q^{\pi_{\theta_t}}(s, a') - Q_{\omega_t}(s, a')]\right| \\
    &\le |Q^{\pi_{\theta_t}}(s, a) - Q_{\omega_t}(s, a)| + |\mathbb{E}_{a'\sim \pi_{\theta_t}}[Q^{\pi_{\theta_t}}(s, a') - Q_{\omega_t}(s, a')]|.
\end{align}
Then, we can derive the bound of $(A^{\pi_{\theta_t}}(s, a) - A_{\omega_t}(s, a))^2$ as follows,
\begin{align}
    (A^{\pi_{\theta_t}}(s, a) - A_{\omega_t}(s, a))^2 &\le 2(Q^{\pi_{\theta_t}}(s, a) - Q_{\omega_t}(s, a))^2 + 2(\mathbb{E}_{a'\sim \pi_{\theta_t}}[Q^{\pi_{\theta_t}}(s, a') - Q_{\omega_t}(s, a')])^2 \label{eq:PE_error 1}\\
    &\le 2(Q^{\pi_{\theta_t}}(s, a) - Q_{\omega_t}(s, a))^2 + 2\mathbb{E}_{a'\sim \pi_{\theta_t}}[Q^{\pi_{\theta_t}}(s, a') - Q_{\omega_t}(s, a')^2],\label{eq:PE_error 2}
\end{align}
{where (\ref{eq:PE_error 2}) holds by Jensen's inequality.}
By taking the expectation of (\ref{eq:PE_error 1})-(\ref{eq:PE_error 2}) over the state-action distribution $\sigma_t$, we have
\begin{align}
    \mathbb{E}_{\sigma_t}[(&A^{\pi_{\theta_t}}(s, a) - A_{\omega_t}(s, a))^2] \label{eq:PE_error 3}\\
    &\le 2\mathbb{E}_{\sigma_t}[(Q^{\pi_{\theta_t}}(s, a) - Q_{\omega_t}(s, a))^2] + 2\mathbb{E}_{\sigma_t}[\mathbb{E}_{ a'\sim \pi_{\theta_t}}[(Q^{\pi_{\theta_t}}(s, a') - Q_{\omega_t}(s, a'))^2]] \label{eq:PE_error 4}\\
    &= 4\mathbb{E}_{\sigma_t}[(Q^{\pi_{\theta_t}}(s, a) - Q_{\omega_t}(s, a))^2].\label{eq:PE_error 5},
\end{align}
where the last equality in (\ref{eq:PE_error 5}) is obtained by the actions are directly sampled by $\pi_{\theta_t}$ so we can ignore it in the latter term. Last, we leverage Lemma \ref{thm:4.6_liu} to obtain the result of Lemma \ref{lm:PE_error}.
\end{proof}

\begin{lemma}[Policy Improvement Error]
\label{lm:PI_error}
The output $f_{\bar{\theta}}$ of Algorithm \ref{algo:4} satisfies
\begin{align}
    \mathbb{E}_{\text{init,}\tilde{\sigma}_t}&[(f_{\bar{\theta}}(s, a) - \tau_{t+1} \cdot (C_t(s, a) \cdot A_{\omega_t}(s, a) + \tau_t^{-1} f_{\theta_t}(s, a)))^2] \\
    = \ &O(R_f^2 T_{\text{upd}}^{-1/2} + R_f^{5/2} m_f^{-1/4} + R_f^3 m_f^{-1/2}), \nonumber
\end{align}
\end{lemma}

\noindent To prove Lemma \ref{lm:PI_error}, we first state the following useful result noindently proposed by \citep{liu2019neural}.
\begin{theorem}[\citep{liu2019neural}, Meta-Algorithm of Neural Networks]
\label{thm:meta}
    Consider a meta-algorithm with the following update:
    \begin{align}
        \alpha(t + 1/2) &\leftarrow \alpha(t) - \eta_{\text{upd}} \cdot (u_{\alpha(t)}(s, a) - v(s, a) - \mu \cdot u_{\alpha(t)}(s', a')) \cdot \nabla_{\alpha} u_{\alpha(t)}(s, a),     \label{eq:meta-algo 1}\\
        \alpha(t + 1) &\leftarrow \prod_{B_{\alpha}}(\alpha(t + 1/2)) = \mathop{\arg \min}_{\alpha \in B_{\alpha}} \lVert\alpha - \alpha(t + 1/2)\rVert_2, \label{eq:meta-algo 2}
    \end{align}
    where $\mu \in [0, 1)$ is a constant, $(s, a, s', a')$ is sampled from some stationary distribution $d$, $u_{\alpha}$ is parameterized as a two-layer neural network $\text{NN}(\alpha;m)$, and $v(s, a)$ satisfies
    \begin{align}
    \label{condition:v}
        \mathbb{E}_{d}[(v(s,a))^2] \le \bar{v}_1 \cdot \mathbb{E}_{d}[(u_{\alpha(0)}(s, a))^2] + \bar{v}_2 \cdot R_{u}^2 + \bar{v}_3,
    \end{align} 
    for some constants $\bar{v}_1, \bar{v}_2, \bar{v}_3 \ge 0$.
    We define the update operator $\gT u(s, a) = \mathbb{E}[v(s, a) + \mu \cdot u(s', a')|s' \sim \mathcal{P}(\cdot|s, a), a' \sim \pi(\cdot|s)]$, and define $\alpha^*$ as the \textit{approximate stationary point} (cf. (D.18) in \citep{liu2019neural}), which inherently have the property $u^0_{\alpha^*} = \prod_{\mathcal{F}_{R_u, m}}\mathcal{T}u^0_{\alpha^*}$, where $u_{\alpha^*}^0$ is the linearization of $u$ at $\alpha^*$.
    Suppose we run the above meta-algorithm in (\ref{eq:meta-algo 1})-(\ref{eq:meta-algo 2}) for $T$ iterations with $T \ge 64 / (1 - \mu)^2$ and set the step size $\eta_{\text{upd}} = T^{-1/2}$. Then, we have
    \begin{align}
        \label{thm:meta:eq1}
        \mathbb{E}_{\text{init,}d}[(u_{\bar{\alpha}}(s, a) - u_{\alpha*}^0(s, a))^2] &= O(R_u^2 T_{\text{upd}}^{-1/2} + R_u^{5/2} m^{-1/4} + R_u^3 m^{-1/2}), \\
        \label{thm:meta:eq2}
        \mathbb{E}_{\text{init,}d}[(u_{\alpha'}(s, a) - u_{\alpha'}^0(s, a))^2] &= O(R_u^3 m^{-1/2}),
    \end{align}
    where $\bar{\alpha} \coloneqq 1/T \cdot (\sum_{t=0}^{T-1} \alpha(t))$
    and $\alpha'$ is a parameter in $B_{\alpha}$.
\end{theorem}

\begin{proof}[Proof of Lemma \ref{lm:PI_error}]
Now we are ready to prove Lemma \ref{lm:PI_error} as follows. To begin with, (\ref{eq:meta-algo 1})-(\ref{eq:meta-algo 2}) match the policy improvement update of {Neural PPO-Clip} if we put $u(s, a) = f(s, a)$, $v(s, a) = \tau_{t+1} (C_t(s, a) \cdot A_{
\omega_t}(s, a) + \tau_t^{-1} f_{\theta_t}(s, a))$, $\mu = 0$, $d = \tilde{\sigma}_t$, and $R_{u} = R_f$. For $\mathbb{E}_{\tilde{\sigma}_t}[(v(s,a))^2]$, we have
\begin{align}
    \mathbb{E}_{\tilde{\sigma}_t}[(v(s,a))^2] &\le 2\tau_{t+1}^2 (U_{C}^2 \cdot \mathbb{E}_{\tilde{\sigma}_t}[(A_{\omega_t}(s, a))^2] + \tau^{-2}_t \mathbb{E}_{\tilde{\sigma}_t}[(f_{\theta_t}(s, a))^2]) \\
    &\le 20 \mathbb{E}_{\tilde{\sigma}_t}[(f_{\theta_0}(s, a))^2] + 20 R_f^2.\label{eq:PI_error 1}
\end{align}
{Here, since $C_t$ and $\bar{C_t}$ are dependent only on the EMDA step size $\eta$ and the indicator function that depends on the sign of the advantage (either under the true advantage $A^{\pi_{\theta_t}}$ or the approximated advantage $A_{\omega_t}$), one can always find one common upper bound $U_C$ for both $C_t$ and $\bar{C_t}$.
In particular, as shown in Corollary \ref{cor:PPO-Clip}, we set $U_{C} = \sum_{k=0}^{K-1} \eta$ for PPO-Clip, which is independent from the advantage function.}
{
The inequality in (\ref{eq:PI_error 1}) holds by the condition that $\tau_{t+1}^2 (U_{C}^2 + \tau_{t}^{-2}) \le 1$, $(a + b)^2 \le 2a^2 + 2b^2$, $\mathbb{E}_{\tilde{\sigma}_t}[(A_{\omega_t}(s, a))^2] \le 4 \mathbb{E}_{\tilde{\sigma}_t}[(Q_{\omega_t}(s, a))^2] $, and $\mathbb{E}_{\tilde{\sigma}_t}[(u_{\alpha_t}(s, a))^2] \le 2 \mathbb{E}_{\tilde{\sigma}_t}[(u_{\alpha_0}(s, a))^2] + 2 R_f^2$ which holds by using the Lipschitz property of neural networks where $u_{\alpha} = f_{\theta}, A_{\omega}$.
The condition $\tau_{t+1}^2 (U_{C}^2 + \tau_{t}^{-2}) \le 1$ can be satisfied by configuring proper $\{\tau_t\}$, as described momentarily in Appendix \ref{app:add:cor}.}
We also use that $\mathbb{E}_{\tilde{\sigma}_t}[Q_{\omega(0)}] = \mathbb{E}_{\tilde{\sigma}_t}[f_{\theta(0)}]$ because they share the same initialization. Thus, we have $\bar{v}_1 = \bar{v}_2 = 20$ and $\bar{v}_3 = 0$ in (\ref{condition:v}).

Due to that $\theta^*$ is the approximate stationary point, we have $f^0_{\theta^*} = \prod_{\mathcal{F}_{R_f, m_f}}\mathcal{T}f^0_{\theta^*} = \prod_{\mathcal{F}_{R_f, m_f}} \tau_{t+1} (C_t \circ A_{\omega_t} + \tau_t^{-1} f_{\theta_t})$. Thus, 
\begin{align}
    f^0_{\theta^*} = \mathop{\arg\min}_{f \in \mathcal{F}_{R_f, m_f}} \lVert f - \tau_{t+1} (C_t \circ A_{\omega_t} + \tau_t^{-1} f_{\theta_t})\rVert_{2, \tilde{\sigma}_t},
\end{align}
where $\lVert \cdot \rVert_{2, \tilde{\sigma}_t} = \mathbb{E}_{\text{init,}\tilde{\sigma}_t}[\lVert \cdot \rVert_{2}]^{1/2}$ is the $\tilde{\sigma}_t$-weighted $\ell_2$-norm.
Then, by the fact that $\tau_{t+1} (C_t(s, a) \cdot A_{\omega_t}^0(s, a) + \tau_t^{-1} f_{\theta_t}^0(s, a))\in\mathcal{F}_{R_f, m_f}$ and that $A_{\omega_t}^0(s, a) = Q_{\omega_t}^0(s, a) - \sum_{a\in\mathcal{A}} \pi(a|s)Q_{\omega_t}^0(s, a)$, we obtain
\begin{align}
    &\mathbb{E}_{\text{init,}\tilde{\sigma}_t}[(f^0_{\theta^*}(s, a) - \tau_{t+1} (C_t(s, a) \cdot A_{\omega_t}(s, a) + \tau_t^{-1} f_{\theta_t}(s, a)))^2] \\
    &\le \mathbb{E}_{\text{init,}\tilde{\sigma}_t}[(\tau_{t+1} (C_t(s, a)  A_{\omega_t}^0(s, a) + \tau_t^{-1} f_{\theta_t}^0(s, a)) - (\tau_{t+1} (C_t(s, a) A_{\omega_t}(s, a) + \tau_t^{-1} f_{\theta_t}(s, a))))^2] \\
    &\le 2\tau_{t+1}^2 U_{C}^2 \mathbb{E}_{\text{init,}\tilde{\sigma}_t}[((Q_{\omega_t}^0(s, a) - \sum_{a' \in\mathcal{A}} \pi(a'|s)Q_{\omega_t}^0(s, a')) - (Q_{\omega_t}(s, a) - \sum_{a'in\mathcal{A}} \pi(a'|s)Q_{\omega_t}(s, a')))^2] \nonumber \\
    &\qquad + 2\tau_{t+1}^2 \tau_{t}^{-2} \mathbb{E}_{\text{init,}\tilde{\sigma}_t}[(f_{\theta_t}^0(s, a) - f_{\theta_t}(s, a))^2] \\
    \label{eq:82}
    &\le 8 \tau_{t+1}^2 U_{C}^2 \mathbb{E}_{\text{init,}\tilde{\sigma}_t}[(Q_{\omega_t}^0(s, a) - Q_{\omega_t}(s, a))^2] + 2\tau_{t+1}^2 \tau_{t}^{-2} \mathbb{E}_{\text{init,}\tilde{\sigma}_t}[(f_{\theta_t}^0(s, a) - f_{\theta_t}(s, a))^2] \\
    \label{short:result}
    &= O(R_f^3 m_f^{-1/2}).
\end{align}
We obtain (\ref{eq:82}) as the same reason in (\ref{eq:PE_error 1})-(\ref{eq:PE_error 5}) in the proof of Lemma \ref{lm:PE_error}. The terms in (\ref{eq:82}) are both the designated form as the (\ref{thm:meta:eq2}), we leverage the (\ref{thm:meta:eq2}) in Theorem \ref{thm:meta} and obtain the result in (\ref{short:result}).

Last, we bound the error of our policy improvement, we have
\begin{align}
    &\mathbb{E}_{\text{init,}\tilde{\sigma}_t}[(f_{\bar{\theta}}(s, a) - \tau_{t+1} \cdot (C_t(s, a) \cdot A_{\omega_t}(s, a) + \tau_t^{-1} f_{\theta_t}(s, a)))^2] \\
    \label{eq1:final:meta}
    &\le 2\mathbb{E}_{\text{init,}\tilde{\sigma}_t}[(f_{\bar{\theta}}(s, a) - f^0_{\theta^*}(s, a))^2] \\
    \label{eq2:final:meta}
    &\qquad + 2 \mathbb{E}_{\text{init,}\tilde{\sigma}_t}[(f^0_{\theta^*}(s, a) - \tau_{t+1} (C_t(s, a) \cdot A_{\omega_t}(s, a) + \tau_t^{-1} f_{\theta_t}(s, a)))^2] \\
    \label{meta:result}
    &= O(R_f^2 T_{\text{upd}}^{-1/2} + R_f^{5/2} m_f^{-1/4} + R_f^3 m_f^{-1/2}),
\end{align}
where (\ref{eq1:final:meta}) is bounded as $O(R_f^2 T_{\text{upd}}^{-1/2} + R_f^{5/2} m_f^{-1/4} + R_f^3 m_f^{-1/2})$ by (\ref{thm:meta:eq1}) of Theorem \ref{thm:meta}, and (\ref{eq2:final:meta}) is bounded as $O(R_f^3 m_f^{-1/2})$ by the derivation of (\ref{short:result}). Thus, we obtain (\ref{meta:result}) and complete the proof.
\end{proof}

\begin{lemma}[Error Probability of Advantage]
\label{lm:EPA}
    {Given the policy $\pi_{\theta_t}$, the probability of the event that the advantage error is greater than $\epsilon_{\text{err}}$ can be bounded as}
    \begin{align}
        \mathbb{P}(|A_{\omega_t}(s, a) - A^{\pi_{\theta_t}}(s, a)| > \epsilon_{\text{err}}) \le \frac{\mathbb{E}_{\text{init,}\sigma_t}[(A_{\omega_t}(s, a) - A^{\pi_{\theta_t}}(s, a))^2]}{\epsilon_{\text{err}}^2}.\label{eq:Markov inequality}
    \end{align}
\end{lemma}
\begin{proof}[Proof of Lemma \ref{lm:EPA}]
    By applying Markov's inequality, we have
    \begin{align}
        \mathbb{P}(|A_{\omega_t}(s, a) - A^{\pi_{\theta_t}}(s, a)| > \epsilon_{\text{err}}) &= \mathbb{P}(|A_{\omega_t}(s, a) - A^{\pi_{\theta_t}}(s, a)|^2 > \epsilon_{\text{err}}^2) \\
        &\le \frac{\mathbb{E}[(A_{\omega_t}(s, a) - A^{\pi_{\theta_t}}(s, a))^2]}{\epsilon_{\text{err}}^2}.
    \end{align}
\end{proof}
\noindent Notice that the randomness of the above event in (\ref{eq:Markov inequality}) comes from the state-action visitation distribution $\sigma_t$ and the initialization of the neural networks.

\subsection{Proof of Lemma \ref{lm:ep}}
\label{app:neural:ep}
For ease of exposition, we restate Lemma \ref{lm:ep} as follows. 
{In the following, we slightly abuse the notations $ \mathbb{E}_{\tilde{\sigma}_t}$, $\mathbb{E}_{\sigma_t}$, and $\mathbb{E}_{\nu^*}$ to denote the expectations (over the respective distribution) conditioned on the policy $\pi_{\theta_t}$.}
\begin{lemmastar}[Error Propagation]
    Let $\pi_{t+1}$ be the target policy obtained by EMDA with the true advantage. Suppose the policy improvement error satisfies 
    \begin{align}
    \label{lm:ep:proof eq1}
        \mathbb{E}_{\tilde{\sigma}_t}&[(f_{\theta_{t+1}}(s, a) - \tau_{t+1} \cdot (C_t(s, a) \cdot A_{\omega_t}(s, a) + \tau_t^{-1} f_{\theta_t}(s, a)))^2] \le \epsilon_{t+1},
    \end{align}
    and the policy evaluation error satisfies
    \begin{align}
    \label{lm:ep:proof eq2}
        \mathbb{E}_{\sigma_t}[(A_{\omega_t}(s, a) - A^{\pi_{\theta_t}}(s, a))^2] \le \epsilon_t'.
    \end{align}
    Then, the following holds, 
    \begin{align}
    \label{lm:ep:proof eq3}
        |\mathbb{E}_{\nu^*}[\langle \log\pi_{\theta_{t+1}}(\cdot|s) - \log  \pi_{t+1}(\cdot|s), \pi^*(\cdot|s) - \pi_{\theta_t}(\cdot|s) \rangle]| \le \varepsilon_t + \varepsilon_{\text{err}}
    \end{align}
    where $\varepsilon_t = C_{\infty} \tau_{t+1}^{-1} \phi^* \epsilon_{t+1}^{1/2} + U_{C} X^{1/2} \psi^* \epsilon_t'^{1/2}$ and $\varepsilon_{\text{err}} = \sqrt{2} U_{C} \epsilon_{\text{err}} \psi^*$, and $X = \left[(2 / \epsilon_{\text{err}}^2)(M' + (R_{\max} / (1 - \gamma))^2 - \epsilon_t'/2)\right]$, and $M' = 4\mathbb{E}_{\nu_t}[\max_{a} (Q_{\omega_0}(s, a))^2] + 4R_f^2$.
\end{lemmastar}

\begin{remarkapp}
\label{remark:varepislon}
    Notice that $\epsilon_{t+1}$ in (\ref{lm:ep:proof eq1}) and $\epsilon_t'$ in (\ref{lm:ep:proof eq2}) can be controlled by the width of neural networks and the number of iteration for each SGD and TD updates based on Lemma \ref{lm:PE_error} and \ref{lm:PI_error}. Therefore, $\varepsilon_t$ could be made sufficiently small per our requirement.
\end{remarkapp}

\begin{proof}[Proof of Lemma \ref{lm:ep}]
{For ease of exposition, let us first fix a policy $\pi_{\theta_t}$. Through the analysis, we will show that one can derive an upper bound (in the form of (\ref{lm:ep:proof eq3})) that holds regardless of the policy $\pi_{\theta_t}$. 
Recall that $C_t(s, a)= -\sum_{k=0}^{K^{(t)}-1} \eta g_{s,a}^{(k)}$, where $g_{s,a}^{(k)}$ is obtained in the EMDA subroutine and depends on the sign of the estimated advantage $A_{\omega_t}$. Similarly, we define $\bar{C_t}(s, a)$ as the counterpart of $C_t(s,a)$ by replacing $A_{\omega_t}$ with the true advantage $A^{\pi_{\theta_t}}$.}
We first simplify $\langle \log\pi_{\theta_{t+1}}(\cdot|s) - \log  \pi_{t+1}(\cdot|s), \pi^*(\cdot|s) - \pi_{\theta_t}(\cdot|s) \rangle$. The normalizing factor $Z$ of the policies $\pi_{\theta_{t+1}}$ and $\pi_{t+1}$ is state-dependent, {and the inner product between any state-dependent function and the policy difference $\pi^*(\cdot|s) - \pi_{\theta_t}(\cdot|s)$ is always zero.}
Thus, we have
\begin{align}
    \langle \log\pi_{\theta_{t+1}}(\cdot|s) &- \log  \pi_{t+1}(\cdot|s), \pi^*(\cdot|s) - \pi_{\theta_t}(\cdot|s) \rangle  \\
    &= \langle \tau_{t+1}^{-1} f_{\theta_{t+1}}(s, \cdot) - (\bar{C_t}(s, \cdot) \circ A^{\pi_{\theta_t}}(s, \cdot) + \tau_{t}^{-1} f_{\theta_{t}}(s, \cdot)), \pi^*(\cdot|s) - \pi_{\theta_t}(\cdot|s) \rangle.
\end{align}
Then, {we decompose the above equation into two terms:} (i) the error in the policy improvement and (ii) the error between the true advantage and the approximated advantage, i.e., 
\begin{align}
    \langle \tau_{t+1}^{-1} &f_{\theta_{t+1}}(s, \cdot) - (C_t(s, \cdot) A^{\pi_{\theta_t}}(s, \cdot) + \tau_{t}^{-1} f_{\theta_{t}}(s, \cdot)), \pi^*(\cdot|s) - \pi_{\theta_t}(\cdot|s) \rangle \label{eq:ep proof 1}\\
    &= \langle \tau_{t+1}^{-1} f_{\theta_{t+1}}(s, \cdot) - (\bar{C}_t(s, \cdot) \circ A_{\omega_t}(s, \cdot) + \tau_{t}^{-1} f_{\theta_{t}}(s, \cdot)), \pi^*(\cdot|s) - \pi_{\theta_t}(\cdot|s) \rangle  \label{eq:ep proof 2}\\
    &\qquad + \langle C_t(s, \cdot) \circ A_{\omega_t}(s, \cdot) - \bar{C_t}(s, a) \circ A^{\pi_{\theta_t}}(s, \cdot), \pi^*(\cdot|s) - \pi_{\theta_t}(\cdot|s) \rangle \label{eq:ep proof 3}
\end{align}
We first bound the expectation of (i) over $\nu^*$ as follows.
\begin{align}
    &|\mathbb{E}_{\nu^*}[\langle \tau_{t+1}^{-1} f_{\theta_{t+1}}(s, \cdot) - (C_t(s, \cdot) A_{\omega_t}(s, \cdot) + \tau_{t}^{-1} f_{\theta_{t}}(s, \cdot)), \pi^*(\cdot|s) - \pi_{\theta_t}(\cdot|s) \rangle]|  \label{eq:ep proof 4}\\
    &= \left|\int_{\mathcal{S}} \langle \tau_{t+1}^{-1} f_{\theta_{t+1}}(s, \cdot) - (C_t(s, \cdot) \circ A_{\omega_t}(s, \cdot) + \tau_{t}^{-1} f_{\theta_{t}}(s, \cdot)), \pi^*(\cdot|s) - \pi_{\theta_t}(\cdot|s) \rangle \cdot \nu^*(s) ds \right|  \label{eq:ep proof 5}\\
    &= \left|\int_{\mathcal{S} \times \mathcal{A}} (\tau_{t+1}^{-1} f_{\theta_{t+1}}(s, a) - (C_t(s, a) A_{\omega_t}(s, a) + \tau_{t}^{-1} f_{\theta_{t}}(s, a)))  \left(\frac{\pi^*(a|s)}{\pi_0(a|s)} - \frac{\pi_{\theta_t}(a|s)}{\pi_0(a|s)}\right) \frac{\nu^*(s)}{\nu_t (s)} d \tilde{\sigma}_t(s, a) \right|  \label{eq:ep proof 6}\\
    &\le C_{\infty}\mathbb{E}_{\tilde{\sigma}_t}\left[(\tau_{t+1}^{-1} f_{\theta_{t+1}}(s, a) - (C_t(s, a) A_{\omega_t}(s, a) + \tau_{t}^{-1} f_{\theta_{t}}(s, a)))^2\right]^{1/2} \cdot \mathbb{E}_{\tilde{\sigma}_t}\left[\left|\frac{d \pi^*}{d\pi_0} - \frac{d \pi_{\theta_t}}{d \pi_0}\right|^2\right]^{1/2}  \label{eq:ep proof 7}\\
    &\le C_{\infty} \tau_{t+1}^{-1}\epsilon_{t+1}^{1/2} \phi^*_t, \label{eq:ep proof 8}
\end{align}
{where (\ref{eq:ep proof 6}) follows from the definition of $\tilde{\sigma}_t$, (\ref{eq:ep proof 7}) is obtained by Cauchy-Schwarz inequality and Assumption \ref{assump:con}, and the last inequality in (\ref{eq:ep proof 8}) holds by the condition in (\ref{lm:ep:eq1}) and that $\lVert \nu^* / \nu\rVert_{\infty}< C_{\infty}$.}

{Similarly, we consider the expectation of (ii) over $\nu^*$ as follows.}
\begin{align}
    &|\mathbb{E}_{\nu^*}[\langle C_t(s, \cdot) \circ A_{\omega_t}(s, \cdot) - \bar{C_t}(s, \cdot) \circ A^{\pi_{\theta_t}}(s, \cdot), \pi^*(\cdot|s) - \pi_{\theta_t}(\cdot|s) \rangle]|  \label{eq:ep proof 9}\\
    &= \left|\int_{\mathcal{S}} \langle C_t(s, \cdot) \circ A_{\omega_t}(s, \cdot) - \bar{C_t}(s, \cdot) \circ A^{\pi_{\theta_t}}(s, \cdot), \pi^*(\cdot|s) - \pi_{\theta_t}(\cdot|s) \rangle \nu^*(s) ds \right|  \label{eq:ep proof 10}\\
    &= \left|\int_{\mathcal{S} \times \mathcal{A}} (C_t(s, a) A_{\omega_t}(s, a) - \bar{C_t}(s, a)  A^{\pi_{\theta_t}}(s, a)) \left(\frac{\pi^*(a|s)}{\pi_{\theta_t}(a|s)} - \frac{\pi_{\theta_t}(a|s)}{\pi_{\theta_t}(a|s)}\right) \frac{\nu^*(s)}{\nu_t(s)} d\sigma_t(s, a)\right|  \label{eq:ep proof 11}\\
    &= \left|\int_{\mathcal{S} \times \mathcal{A}} (C_t(s, a) A_{\omega_t}(s, a) - \bar{C_t}(s, a)  A^{\pi_{\theta_t}}(s, a)) \left(\frac{\sigma^*(s, a)}{\sigma_t(s, a)} - \frac{\nu^*(s)}{\nu_t(s)}\right) d\sigma_t(s, a)\right|  \label{eq:ep proof 12}\\
    &\le \mathbb{E}_{\sigma_t}[(C_t(s, a) A_{\omega_t}(s, a) - \bar{C_t}(s, a)  A^{\pi_{\theta_t}}(s, a))^2]^{1/2} \cdot \mathbb{E}_{\sigma_t}\left[\left|\frac{d \sigma^*}{d \sigma_t} - \frac{d \nu^*}{d \nu_t}\right|^2\right]^{1/2}, \label{eq:ep proof 13}
\end{align}
where (\ref{eq:ep proof 13}) holds by the Cauchy-Schwarz inequality.
Next, we bound for the term $\mathbb{E}_{\sigma_t}[(C_t(s, a) A_{\omega_t}(s, a) - \bar{C_t}(s, a)  A^{\pi_{\theta_t}}(s, a))^2]$.
For ease of notation, let $D = (C_t(s, a) A_{\omega_t}(s, a) - \bar{C_t}(s, a)  A^{\pi_{\theta_t}}(s, a))^2$ and simply write $\mathbb{E}_{\text{init, }\sigma_t}$ as $\mathbb{E}$. 
{Also, we slightly abuse the notation by using $A_{\omega_t}$ as the random variable $A_{\omega_t}(s, a)$, whose randomness results from the state-action pairs sampled from $\sigma_t$ and the initialization of neural networks, and using $A^{\pi_{\theta_t}}$ as the random variable $A^{\pi_{\theta_t}}(s, a)$, whose randomness comes from the state-action pairs sampled from $\sigma_t$.} 
To establish the bound of $\mathbb{E}[D]$, we consider two different cases for $\mathbb{E}[D]$: one is that the error is greater than $\epsilon_{\text{err}}$, and the other is that the error is less than or equal to $\epsilon_{\text{err}}$. 
Specifically,
\begin{align}
    \mathbb{E}[D] &= \mathbb{E}[D \mid |A_{\omega_t} - A^{\pi_{\theta_t}}| > \epsilon_{\text{err}}] \cdot \mathbb{P}(|A_{\omega_t} - A^{\pi_{\theta_t}}| > \epsilon_{\text{err}}) \nonumber\\
    \label{lm:ep:proof:eq1}
    &\quad + \mathbb{E}[D \mid |A_{\omega_t} - A^{\pi_{\theta_t}}| \le \epsilon_{\text{err}}] \cdot \mathbb{P}(|A_{\omega_t} - A^{\pi_{\theta_t}}| \le \epsilon_{\text{err}})
\end{align}
Then, we upper bound the two terms in \ref{lm:ep:proof:eq1} separately. 
Regarding the first term in \ref{lm:ep:proof:eq1}, we have
\begin{align}
    \mathbb{E}[&D \mid |A_{\omega_t} - A^{\pi_{\theta_t}}| > \epsilon_{\text{err}}] \cdot \mathbb{P}(|A_{\omega_t} - A^{\pi_{\theta_t}}| > \epsilon_{\text{err}}) \nonumber \\
    &\le 2 U_{C}^2 (\mathbb{E}_{\nu_t}[\lVert A_{\omega_t}(s,\cdot)\rVert_{\infty}^{2}] + (A^{\pi_{\theta_t}}_{\max})^2) \cdot \mathbb{P}(|A_{\omega_t} - A^{\pi_{\theta_t}}| > \epsilon_{\text{err}}),\label{eq:ep proof 14}
\end{align}
where (\ref{eq:ep proof 14}) holds by that $(a + b)^2 \le 2a^2 + 2b^2$.
Next, regarding the second term in \ref{lm:ep:proof:eq1}, we further consider two cases based on whether the absolute value of $A^{\pi_{\theta_t}}$ is greater than $\epsilon_{\text{err}}$ or not. 
Specifically,
\begin{align}
    &\mathbb{E}[D \mid |A_{\omega_t} - A^{\pi_{\theta_t}}| \le \epsilon_{\text{err}}] \nonumber \\
    &= \mathbb{E}[D \mid |A_{\omega_t} - A^{\pi_{\theta_t}}| \le \epsilon_{\text{err}}, |A^{\pi_{\theta_t}}| > \epsilon_{\text{err}}] \cdot \mathds{1}\{|A^{\pi_{\theta_t}}| > \epsilon_{\text{err}}\} \nonumber \\
    &\quad + \mathbb{E}[D \mid |A_{\omega_t} - A^{\pi_{\theta_t}}| \le \epsilon_{\text{err}}, |A^{\pi_{\theta_t}}| \le \epsilon_{\text{err}}] \cdot \mathds{1}\{|A^{\pi_{\theta_t}}| \le \epsilon_{\text{err}}\} \label{eq:ep proof 15}\\
    &\le \mathbb{E}[D \mid |A_{\omega_t} - A^{\pi_{\theta_t}}| \le \epsilon_{\text{err}}, |A^{\pi_{\theta_t}}| > \epsilon_{\text{err}}] + \mathbb{E}[D \mid |A_{\omega_t} - A^{\pi_{\theta_t}}| \le \epsilon_{\text{err}}, |A^{\pi_{\theta_t}}| \le \epsilon_{\text{err}}] \label{eq:ep proof 16}\\
    &\le U_{C}^2 \cdot \mathbb{E}[(A_{\omega_t}(s, a) - A^{\pi_{\theta_t}}(s, a))^2] + 4 U_{C}^2 \epsilon_{\text{err}}^2 \label{eq:ep proof 17}
\end{align}
{where (\ref{eq:ep proof 15}) holds by the fact that we fix a policy $\pi_{\theta_t}$ as described in the beginning of Appendix \ref{app:neural:ep} and hence $A^{\pi_{\theta_t}}$ is determined, (\ref{eq:ep proof 16}) holds by that the indicator function is no larger than 1, the first term in (\ref{eq:ep proof 17}) holds by the fact that $A_{\omega_t}$ and $A^{\pi_{\theta_t}}$ have the same sign and hence $C_t$ is equal to $\bar{C}_t$, and the second term in (\ref{eq:ep proof 17}) follows from that $(a + b)^2 \le 2a^2 + 2b^2$.}
Then, by combining the above terms, we have
\begin{align}
    \mathbb{E}[D] &\le 2 U_{C}^2 (\mathbb{E}_{\nu_t}[\lVert A_{\omega_t}(s,\cdot)\rVert_{\infty}^{2}] + (A^{\pi_{\theta_t}}_{\max})^2) \cdot \mathbb{P}(|A_{\omega_t} - A^{\pi_{\theta_t}}| > \epsilon_{\text{err}}) \nonumber \\
    &\quad + [U_{C}^2 \cdot \mathbb{E}[(A_{\omega_t}(s, a) - A^{\pi_{\theta_t}}(s, a))^2] + 4 U_{C}^2 \epsilon_{\text{err}}^2] \cdot \mathbb{P}(|A_{\omega_t} - A^{\pi_{\theta_t}}| \le \epsilon_{\text{err}}) \\
    &= 2 U_{C}^2 (\mathbb{E}_{\nu_t}[\lVert A_{\omega_t}(s,\cdot)\rVert_{\infty}^{2}] + (A^{\pi_{\theta_t}}_{\max})^2) \cdot \mathbb{P}(|A_{\omega_t} - A^{\pi_{\theta_t}}| > \epsilon_{\text{err}}) \nonumber \\
    &\quad + [U_{C}^2 \cdot \mathbb{E}[(A_{\omega_t}(s, a) - A^{\pi_{\theta_t}}(s, a))^2] + 4 U_{C}^2 \epsilon_{\text{err}}^2] \cdot (1 - \mathbb{P}(|A_{\omega_t} - A^{\pi_{\theta_t}}| > \epsilon_{\text{err}}))
\end{align}
Recall that $\epsilon_t' = \mathbb{E}[(A_{\omega_t}(s, a) - A^{\pi_{\theta_t}}(s, a))^2]$. 
{As we could choose an $\epsilon_{\text{err}}$ small enough and use the neural network power to make $\epsilon_t'$ is also small by Lemma \ref{lm:PE_error} such that we have $2 U_{C}^2 (\mathbb{E}_{\nu_t}[\lVert A_{\omega_t}(s,\cdot)\rVert_{\infty}^{2}] + A^{\pi_{\theta_t}}_{\max}) > U_{C}^2 \epsilon_t' + 4 U_{C}^2 \epsilon_{\text{err}}^2$, then by Lemma \ref{lm:EPA} we have}
\begin{align}
\label{lm4:eq}
    \mathbb{E}[D] &\le 2 U_{C}^2 (\mathbb{E}_{\nu_t}[\lVert A_{\omega_t}(s,\cdot)\rVert_{\infty}^{2}] + (A^{\pi_{\theta_t}}_{\max})^2) \cdot \frac{\epsilon_t'}{\epsilon_{\text{err}}^2} + [U_{C}^2 \epsilon_t' + 4 U_{C}^2 \epsilon_{\text{err}}^2]\cdot(1 - \frac{\epsilon_t'}{\epsilon_{\text{err}}^2}).
\end{align}
Rearranging the terms in (\ref{lm4:eq}), we have
\begin{align}
    \mathbb{E}[D] &\le \epsilon_t' U_{C}^2 \cdot \left[\frac{2}{\epsilon_{\text{err}}^2}(M' + (A^{\pi_{\theta_t}}_{\max})^2 - \frac{\epsilon_t'}{2}) - 1\right] + 4 U_{C}^2 \epsilon_{\text{err}}^2 \\
    &\le \epsilon_t' U_{C}^2 \cdot \left[\frac{2}{\epsilon_{\text{err}}^2}(M' + (A^{\pi_{\theta_t}}_{\max})^2 - \frac{\epsilon_t'}{2})\right] + 4 U_{C}^2 \epsilon_{\text{err}}^2 
\end{align}
where $M' \coloneqq 4\mathbb{E}_{\nu_t}[\max_{a} (Q_{\omega_0}(s, a))^2] + 4R_f^2$. 
By introducing the notation $X = \left[(2 / \epsilon_{\text{err}}^2)(M' + (A^{\pi_{\theta_t}}_{\max})^2 - \epsilon_t'/2)\right]$ and combining all the above results, we have
\begin{align}
    |\mathbb{E}_{\nu^*}[\langle &\log\pi_{\theta_{t+1}}(\cdot|s) - \log  \pi_{t+1}(\cdot|s), \pi^*(\cdot|s) - \pi_{\theta_t}(\cdot|s) \rangle]| \\
    &\le C_{\infty} \tau_{t+1}^{-1}\epsilon_{t+1}^{1/2} \phi^*_t + (\epsilon_t' U_{C}^2 X + 4 U_{C}^2  \epsilon_{\text{err}}^2)^{1/2} \psi^*_t \\
    &\le \epsilon_{t+1}^{1/2} C_{\infty} \tau_{t+1}^{-1} \phi^*_t + \epsilon_t'^{1/2} U_{C} X^{1/2} \psi^*_t + 2 U_{C}  \epsilon_{\text{err}} \psi^*_t,\label{eq:ep proof 19} \\
    &< \epsilon_{t+1}^{1/2} C_{\infty} \tau_{t+1}^{-1} \phi^* + \epsilon_t'^{1/2} U_{C} X^{1/2} \psi^* + 2 U_{C}  \epsilon_{\text{err}} \psi^*,\label{eq:ep proof 20}
\end{align}
where (\ref{eq:ep proof 19}) follows from the inequality $\sqrt{a+b} \le \sqrt{a} + \sqrt{b}$ and that $\varepsilon_t = \epsilon_{t+1}^{1/2} C_{\infty} \tau_{t+1}^{-1} \phi^* + \epsilon_t'^{1/2} U_{C} X^{1/2} \psi^*$ and $\varepsilon_{\text{err}} = 2 U_{C}  \epsilon_{\text{err}} \psi^*$. The proof is complete.
\end{proof}

\subsection{Proof of Lemma \ref{lm:sed}}
For ease of exposition, we restate Lemma \ref{lm:sed} as follows.
\begin{lemmastar}[Stepwise Energy $\ell_{\infty}$-Difference]
    \begin{align}
        \mathbb{E}_{\nu^*}[\lVert\tau_{t+1}^{-1} f_{\theta_{t+1}}(s,\cdot) - \tau_{t}^{-1}f_{\theta_{t}}(s,\cdot)\rVert_{\infty}^2] \le 2\varepsilon'_t + 2 U_{C}^2 M,
    \end{align}
    where $\varepsilon'_t = |\mathcal{A}| \cdot C_{\infty} \tau_{t+1}^{-2} \epsilon_{t+1}$ and $M = 4\mathbb{E}_{\nu^*}[\max_{a} (Q_{\omega_0}(s, a))^2] + 4 R_f^2$.
\end{lemmastar}
\begin{remarkapp}
    As described in Remark \ref{remark:varepislon}, $\epsilon_{t+1}$ can be sufficiently small due to Lemma \ref{lm:PI_error}. Similarly, $\varepsilon_t'$ can also be made arbitrarily small.
\end{remarkapp}

\begin{proof}[Proof of Lemma \ref{lm:sed}]
We first find an explicit bound for $\lVert \tau_{t+1}^{-1} f_{\theta_{t+1}}(s,\cdot) - \tau_{t}^{-1}f_{\theta_{t}}(s,\cdot)\rVert_{\infty}^2$. 
Note that
\begin{align}
\label{lm:sed:eq1}
    \lVert \tau_{t+1}^{-1} f_{\theta_{t+1}}(s,\cdot) - \tau_{t}^{-1}f_{\theta_{t}}(s,\cdot)\rVert_{\infty}^2 &\le 2 \lVert\tau_{t+1}^{-1} f_{\theta_{t+1}}(s,\cdot) - \tau_{t}^{-1}f_{\theta_{t}}(s,\cdot) - C_t(s, \cdot)\circ A_{\omega_t}(s,\cdot)\rVert_{\infty}^2 \\
    &\qquad + 2 \lVert C_t(s, \cdot)\circ A_{\omega_t}(s,\cdot)\rVert_{\infty}^{2}. \nonumber
\end{align}
Next, we consider the expectation of (\ref{lm:sed:eq1}) over $\nu^*$:
For the first term in (\ref{lm:sed:eq1}), we have
\begin{align}
    \mathbb{E}_{\nu^*}&[\lVert\tau_{t+1}^{-1} f_{\theta_{t+1}}(s,\cdot) - \tau_{t}^{-1}f_{\theta_{t}}(s,\cdot) - C_t(s, \cdot)\circ A_{\omega_t}(s,\cdot)\rVert_{\infty}^2] \\
    &= \int_{\mathcal{S}} \lVert\tau_{t+1}^{-1} f_{\theta_{t+1}}(s,\cdot) - \tau_{t}^{-1}f_{\theta_{t}}(s,\cdot) - C_t(s, \cdot)\circ A_{\omega_t}(s,\cdot)\rVert_{\infty}^2 \nu^*(s) ds \\
    &= \int_{\mathcal{S} \times \mathcal{A}} \frac{1}{\pi_0(a|s) }\cdot (\tau_{t+1}^{-1} f_{\theta_{t+1}}(s,a) - \tau_{t}^{-1}f_{\theta_{t}}(s,a) - C_t(s, a)\cdot A_{\omega_t}(s,a))^2 \frac{\nu^*(s)}{\nu_t(s)} d\tilde{\sigma_t}(s, a) \\
    &< |\mathcal{A}| \cdot C_{\infty} \tau_{t+1}^{-2} \epsilon_{t+1},\label{lm:sed:proof eq2}
\end{align}
where (\ref{lm:sed:proof eq2}) holds by the condition in (\ref{lm:ep:eq1}), the definition of the concentrability coefficient, and the fact that $\pi_0$ is a uniform policy.
Furthermore, we bound $ \mathbb{E}_{\nu^*}[\lVert C_t(s, \cdot)\circ A_{\omega_t}(s,\cdot)\rVert_{\infty}^{2}]$, we have
\begin{align}
    \mathbb{E}_{\nu^*}[\lVert C_t(s, \cdot)\circ A_{\omega_t}(s,\cdot)\rVert_{\infty}^{2}] &\le U_{C}^2 \cdot \mathbb{E}_{\nu^*}[\lVert A_{\omega_t}(s,\cdot)\rVert_{\infty}^{2}] \\
    &= U_{C}^2 \cdot \mathbb{E}_{\nu^*}[\lVert Q_{\omega_t}(s,\cdot) - \sum_{a} \pi_{\theta_t}(a|s) Q_{\omega_t}(s,a)\rVert_{\infty}^{2}] \\
    &= U_{C}^2 \cdot \mathbb{E}_{\nu^*}[\lVert Q_{\omega_t}(s,\cdot) - \mathbb{E}_{a \sim \pi_{\theta_t}}[ Q_{\omega_t}(s,a)]\rVert_{\infty}^{2}] \\
    &\le 2 U_C^2 \mathbb{E}_{\nu^*}[\lVert Q_{\omega_t}(s,\cdot)\rVert_{\infty}^2] + 2U_C^2\mathbb{E}_{\nu^*}[\mathbb{E}_{a \sim \pi_{\theta_t}}[ (Q_{\omega_t}(s,a))]^2] \\
    &\le 2 U_C^2 \mathbb{E}_{\nu^*}[\lVert Q_{\omega_t}(s,\cdot)\rVert_{\infty}^2] + 2 U_C^2 \mathbb{E}_{\nu^*}[\lVert Q_{\omega_t}(s,\cdot)\rVert_{\infty}^2] \label{eq:Q_t}\\
    &\le U_{C}^2 \cdot 4\mathbb{E}_{\nu^*}[\lVert Q_{\omega_t}(s,\cdot)\rVert_{\infty}^2] \\
    &\le 4 U_{C}^2 \cdot [\mathbb{E}_{\nu^*}[\max_{a} (Q_{\omega_0}(s, a))^2] + R_f^2],\label{lm:sed:proof eq3}
\end{align}
where (\ref{eq:Q_t}) holds by using Jensen's inequality and leveraging the $\ell_{\infty}$-norm instead of the expectation $\mathbb{E}_{a\sim\pi_{\theta_t}}[\cdot]$, and the last inequality in (\ref{lm:sed:proof eq3}) holds by the 1-Lipschitz property of neural networks with respect to the weights. 
By setting $\varepsilon_t' = |\mathcal{A}| \cdot C_{\infty} \tau_{t+1}^{-2} \epsilon_{t+1}$ and $M = 4\mathbb{E}_{\nu^*}[\max_{a} (Q_{\omega_0}(s, a))^2] + 4R_f^2$, we complete the proof of Lemma \ref{lm:sed}.
\end{proof}

\subsection{Proof of the Lemma \ref{lm:OSD}}
For ease of exposition, we restate Lemma \ref{lm:OSD} as follows.
\begin{lemmastar}[Stepwise KL Difference]
The KL difference is as follows,
\begin{align}
    &\text{KL}(\pi^*(\cdot|s) \rVert \pi_{\theta_{t+1}}(\cdot|s)) - \text{KL}(\pi^*(\cdot|s) \rVert \pi_{\theta_t}(\cdot|s)) \\
    &\le \langle \log\pi_{\theta_{t+1}}(\cdot|s) - \log  {\pi}_{t+1}(\cdot|s), \pi_{\theta_t}(\cdot|s) -  \pi^*(\cdot|s) \rangle  - \langle \bar{C_t}(s, \cdot) \circ A^{\pi_{\theta_t}}(s, \cdot), \pi^*(\cdot|s) - \pi_{\theta_t}(\cdot|s) \rangle \nonumber \\
    &\qquad  - \frac{1}{2} \lVert\pi_{\theta_{t+1}}(\cdot|s) - \pi_{\theta_t}(\cdot|s)\rVert_1^2 - \langle \log\pi_{\theta_{t+1}}(\cdot|s) - \log  \pi_{\theta_t}(\cdot|s), \pi_{\theta_t}(\cdot|s) - \pi_{\theta_{t+1}}(\cdot|s) \rangle
\end{align}
\end{lemmastar}
\begin{proof}[Proof of Lemma \ref{lm:OSD}]
We directly expand the one-step KL divergence difference as
\begin{align}
    \text{KL}(\pi^*(&\cdot|s) \rVert \pi_{\theta_{t+1}}(\cdot|s)) - \text{KL}(\pi^*(\cdot|s) \rVert \pi_{\theta_t}(\cdot|s)) = \left\langle \log \frac{\pi_{\theta_{t}}(\cdot|s)}{\pi_{\theta_{t+1}}(\cdot|s)}, \pi^*(\cdot|s) \right\rangle \\
    &= \left\langle \log \frac{\pi_{\theta_{t+1}}(\cdot|s)}{\pi_{\theta_t}(\cdot|s)}, \pi_{\theta_{t+1}}(\cdot|s) - \pi^*(\cdot|s)  \right\rangle + \text{KL}(\pi_{\theta_{t+1}}(\cdot|s) \rVert \pi_{\theta_t}(\cdot|s)) \\
    &= \left\langle \log \frac{\pi_{\theta_{t+1}}(\cdot|s)}{\pi_{\theta_t}(\cdot|s)} - \bar{C_t}(s, \cdot) \circ A^{\pi_{\theta_t}}(s, \cdot), \pi_{\theta_{t}}(\cdot|s) - \pi^*(\cdot|s) \right\rangle \\
    &\qquad - \langle \bar{C_t}(s, \cdot) \circ A^{\pi_{\theta_t}}(s, \cdot), \pi^*(\cdot|s) - \pi_{\theta_{t}}(\cdot|s) \rangle - \text{KL}(\pi_{\theta_{t+1}}(\cdot|s) \rVert \pi_{\theta_t}(\cdot|s)) \nonumber \\
    &\qquad - \left\langle \log \frac{\pi_{\theta_{t+1}}(\cdot|s)}{\pi_{\theta_t}(\cdot|s)}, \pi_{\theta_{t}}(\cdot|s) - \pi_{\theta_{t+1}}(\cdot|s) \right\rangle. \nonumber
\end{align}
Then, by Pinsker's inequality, we have
\begin{align}
    \text{KL}(\pi^*(&\cdot|s) \rVert \pi_{\theta_{t+1}}(\cdot|s)) - \text{KL}(\pi^*(\cdot|s) \rVert \pi_{\theta_t}(\cdot|s)) \\
    & = \left\langle \log \frac{\pi_{\theta_{t+1}}(\cdot|s)}{\pi_{\theta_t}(\cdot|s)} - \bar{C_t}(s, \cdot) \circ A^{\pi_{\theta_t}}(s, \cdot), \pi_{\theta_{t}}(\cdot|s) - \pi^*(\cdot|s) \right\rangle \\
    &\qquad - \langle \bar{C_t}(s, \cdot) \circ A^{\pi_{\theta_t}}(s, \cdot), \pi^*(\cdot|s) - \pi_{\theta_{t}}(\cdot|s) \rangle - \text{KL}(\pi_{\theta_{t+1}}(\cdot|s) \rVert \pi_{\theta_t}(\cdot|s)) \nonumber \\
    &\qquad - \left\langle \log \frac{\pi_{\theta_{t+1}}(\cdot|s)}{\pi_{\theta_t}(\cdot|s)}, \pi_{\theta_{t}}(\cdot|s) - \pi_{\theta_{t+1}}(\cdot|s) \right\rangle \nonumber \\
    &\le \langle \log\pi_{\theta_{t+1}}(\cdot|s) - \log  \pi_{\theta_t}(\cdot|s) - \bar{C_t}(s, \cdot) \circ A^{\pi_{\theta_t}}(s, \cdot), \pi_{\theta_t}(\cdot|s) -  \pi^*(\cdot|s) \rangle \label{eq:OSD proof 1}\\
    &\qquad - \langle \bar{C_t}(s, \cdot) \circ A^{\pi_{\theta_t}}(s, \cdot), \pi^*(\cdot|s) - \pi_{\theta_t}(\cdot|s) \rangle - \frac{1}{2} \lVert\pi_{\theta_{t+1}}(\cdot|s) - \pi_{\theta_t}(\cdot|s)\rVert_1^2 \nonumber \\
    &\qquad - \langle \log\pi_{\theta_{t+1}}(\cdot|s) - \log  \pi_{\theta_t}(\cdot|s), \pi_{\theta_t}(\cdot|s) - \pi_{\theta_{t+1}}(\cdot|s) \rangle. \nonumber
\end{align}
Finally, by Proposition \ref{pp:PI}, we have $\log \pi_{t+1}(\cdot|s) = \log  \pi_{\theta_t}(\cdot|s) + \bar{C_t}(s, \cdot) \circ A^{\pi_{\theta_t}}(s, \cdot)$ and then apply this to the first term in (\ref{eq:OSD proof 1}). The proof is complete.
\end{proof}
\subsection{Proof of Lemma \ref{lm:PDL}}
For ease of exposition, we restate Lemma \ref{lm:PDL} as follows.
\begin{lemmastar}[Performance Difference Using Advantage]
Recall that $\mathcal{L}(\pi) = \mathbb{E}_{\nu^*}[V^{\pi}(s)]$. We have 
\begin{align}
    \mathcal{L}(\pi^*) - \mathcal{L}(\pi) = (1 - \gamma) ^ {-1} \cdot \mathbb{E}_{\nu^*}[\langle A^{\pi}(s, \cdot), \pi^*(\cdot|s) - \pi(\cdot|s)\rangle].
\end{align}
\end{lemmastar}
Before proving Lemma \ref{lm:PDL}, we first state the following property.
\begin{lemma}[\citep{liu2019neural}, Lemma 5.1]
\label{lm:PDQ}
\begin{align}
    \mathcal{L}(\pi^*) - \mathcal{L}(\pi) = (1 - \gamma) ^ {-1} \cdot \mathbb{E}_{\nu^*}[\langle Q^{\pi}(s, \cdot), \pi^*(\cdot|s) - \pi(\cdot|s)\rangle].
\end{align}
\end{lemma}
\begin{proof}[Proof of Lemma \ref{lm:PDL}]
As the value function $V^{\pi}(\cdot)$ is state-dependent, we have
\begin{align}
    \mathbb{E}_{\nu^*}[\langle V^{\pi}(s), \pi^*(\cdot|s) - \pi(\cdot|s)\rangle] &= \mathbb{E}_{\nu^*}[V^{\pi}(s) \cdot \sum_{a \in \mathcal{A}} (\pi^*(a|s) - \pi(a|s))] \\
    &= \mathbb{E}_{\nu^*}\left[V^{\pi}(s) \cdot \left(\sum_{a \in \mathcal{A}} \pi^*(a|s) - \sum_{a \in \mathcal{A}}\pi(a|s)\right)\right] = 0.\label{eq:performance difference 1}
\end{align}
Therefore, by (\ref{eq:performance difference 1}) and Lemma \ref{lm:PDQ}, we have
\begin{align}
    \mathcal{L}(\pi^*) - \mathcal{L}(\pi)  &= (1 - \gamma) ^ {-1} \cdot\mathbb{E}_{\nu^*}[\langle Q^{\pi}(s, \cdot) - V^{\pi}(s), \pi^*(\cdot|s) - \pi(\cdot|s)\rangle] \\
    &= (1 - \gamma) ^ {-1} \cdot \mathbb{E}_{\nu^*}[\langle A^{\pi}(s, \cdot), \pi^*(\cdot|s) - \pi(\cdot|s)\rangle].
\end{align}
\end{proof}
\subsection{Proof of Theorem \ref{thm:main}}

By taking expectation of the KL difference in Lemma \ref{lm:OSD} over $\nu^*$, we obtain
\begin{align}
    &\mathbb{E}_{\nu^*}[\text{KL}(\pi^*(\cdot|s) || \pi_{\theta_{t+1}}(\cdot|s)) - \text{KL}(\pi^*(\cdot|s) || \pi_{\theta_t}(\cdot|s))] \\ 
    &\le \varepsilon_t + \varepsilon_{\text{err}}
    - \mathbb{E}_{\nu^*}[\langle \bar{C}_t(s,\cdot) \circ A^{\pi_{\theta_t}}(s, \cdot), \pi^*(\cdot | s) - \pi_{\theta_t}(\cdot|s)\rangle] - \frac{1}{2} \mathbb{E}_{\nu^*}[\lVert\pi_{\theta_{t+1}}(\cdot|s) - \pi_{\theta_{t}}(\cdot|s)\rVert_{1}^{2}] \nonumber \\
    &\quad -\mathbb{E}_{\nu^*}[\langle \tau_{t+1}^{-1} f_{\theta_{t+1}}(s, \cdot) - \tau_{t}^{-1} f_{\theta_{t}}(s, \cdot), \pi_{\theta_{t}}(\cdot|s) - \pi_{\theta_{t+1}}(\cdot|s)\rangle] \\
    &\le \varepsilon_t + \varepsilon_{\text{err}}
    - \mathbb{E}_{\nu^*}[\langle \bar{C}_t(s,\cdot) \circ A^{\pi_{\theta_t}}(s, \cdot), \pi^*(\cdot | s) - \pi_{\theta_t}(\cdot|s)\rangle] - \frac{1}{2} \mathbb{E}_{\nu^*}[\lVert\pi_{\theta_{t+1}}(\cdot|s) - \pi_{\theta_{t}}(\cdot|s)\rVert_{1}^{2}] \nonumber \\
    &\quad + \mathbb{E}_{\nu^*}[\lVert\tau_{t+1}^{-1} f_{\theta_{t+1}}(s, \cdot) - \tau_{t}^{-1} f_{\theta_{t}}(s, \cdot)\rVert_{\infty} \cdot \lVert\pi_{\theta_{t+1}}(\cdot|s) - \pi_{\theta_{t}}(\cdot|s)\rVert_{1}] \\
    &\le \varepsilon_t + \varepsilon_{\text{err}}
    - \mathbb{E}_{\nu^*}[\langle \bar{C}_t(s,\cdot) \circ A^{\pi_{\theta_t}}(s, \cdot), \pi^*(\cdot |s) - \pi_{\theta}(\cdot|s) \rangle] \nonumber \\
    &\quad + \frac{1}{2} \mathbb{E}_{\nu^*}[\lVert\tau_{t+1}^{-1} f_{\theta_{t+1}}(s, \cdot) - \tau_{t}^{-1} f_{\theta_{t}}(s, \cdot)\rVert_{\infty}^{2}],
\end{align}
where the first inequality follows from Lemma \ref{lm:OSD} and Lemma \ref{lm:ep}, the second inequality holds by the Hölder's inequality, and the last inequality holds by the fact that $2xy - x^2 \le y^2$ and merging the last two terms. Then, by Lemma \ref{lm:sed} and rearranging the terms, we obtain that 
\begin{align}
\label{proof:eq:each_t}
    \mathbb{E}_{\nu^*}&[\langle \bar{C}_t(s,\cdot) \circ A^{\pi_{\theta_t}}(s, \cdot), \pi^*(\cdot | s) - \pi_{\theta_t}(\cdot|s)\rangle] \nonumber \\
    &\le \mathbb{E}_{\nu^*}[\text{KL}(\pi^*(\cdot|s) \rVert \pi_{\theta_{t}}(\cdot|s)) - \text{KL}(\pi^*(\cdot|s) \rVert \pi_{\theta_{t+1}}(\cdot|s))] + \varepsilon_t + \varepsilon_{\text{err}} + \varepsilon_t' + U_{C}^2 M.
\end{align}
By the first condition of (\ref{suff:1}), we have $L_{C}  \mathbb{E}_{\nu^*}[\langle A^{\pi_{\theta_t}}(s, \cdot), \pi^*(\cdot|s) - \pi_{\theta_t}(\cdot|s)\rangle] \le \mathbb{E}_{\nu^*}[\langle \bar{C}_t(s,\cdot) \circ A^{\pi_{\theta_t}}(s, \cdot), \pi^*(\cdot | s) - \pi_{\theta_t}(\cdot|s)\rangle]$. 
By obtaining the performance difference via Lemma \ref{lm:PDL}, we have
\begin{align}
\label{app:proof:thm:main:eq_t}
    &(1 - \gamma) L_{C} (\mathcal{L}(\pi^*) - \mathcal{L}(\pi_{\theta_t}))\nonumber \\
    &\le \mathbb{E}_{\nu^*}[\text{KL}(\pi^*(\cdot|s) \rVert \pi_{\theta_{t}}(\cdot|s)) - \text{KL}(\pi^*(\cdot|s) \rVert \pi_{\theta_{t+1}}(\cdot|s))] + \varepsilon_t + \varepsilon_{\text{err}} + \varepsilon_t' + U_{C}^2 M.
\end{align}
Then, by taking the telescoping sum of (\ref{app:proof:thm:main:eq_t}) from $t = 0$ to $T-1$, we have
\begin{align}
    &(1 - \gamma) L_{C} \sum_{t=0}^{T-1} (\mathcal{L}(\pi^*) - \mathcal{L}(\pi_{\theta_t})) &\\
    &\le \mathbb{E}_{\nu^*}[\text{KL}(\pi^*(\cdot|s) \rVert \pi_{\theta_{0}}(\cdot|s))] - \mathbb{E}_{\nu^*}[\text{KL}(\pi^*(\cdot|s) \rVert \pi_{\theta_{T}}(\cdot|s))] + \sum_{t=0}^{T-1} (\varepsilon_t + \varepsilon_{\text{err}} + \varepsilon_t') + T U_{C}^2 M. &
\end{align}
By the facts that (i) $\mathbb{E}_{\nu^*}[\text{KL}(\pi^*(\cdot|s) \rVert \pi_{\theta_{0}}(\cdot|s))] \le \log |\mathcal{A}|$, (ii) KL divergence is nonnegative, (iii) $\sum_{t=0}^{T-1} (\mathcal{L}(\pi^*) - \mathcal{L}(\pi_{\theta_t})) \ge T \cdot \min_{0\le t \le T} \{\mathcal{L}(\pi^*) - \mathcal{L}(\pi_{\theta_t})\}$, we have
\begin{align}
\label{proof:thm:main:pre_result}
    \min_{0\le t \le T} \{\mathcal{L}(\pi^*) - \mathcal{L}(\pi_{\theta_t})\} \le \frac{\log |\mathcal{A}| + \sum_{t=0}^{T-1} (\varepsilon_t + \varepsilon_t') + T (\varepsilon_{\text{err}} + M U_{C}^2)}{T L_{C} (1 - \gamma)}.
\end{align}
Since we have $\varepsilon_{\text{err}} = 2 U_{C} \epsilon_{\text{err}} \psi^*$ and the condition of (\ref{suff:2}), we know that if we set $\epsilon_{\text{err}} = U_C$ and $T$ to be sufficiently large, {$\epsilon_{\text{err}}$ shall be sufficiently small and hence satisfy the condition required by (\ref{lm4:eq}).} 
Thus, by plugging $\epsilon_{\text{err}} = U_C$ into (\ref{proof:thm:main:pre_result}), we have $\varepsilon_{\text{err}} = 2 U_C^2 \psi^*$ and $\varepsilon_t = \epsilon_{t+1}^{1/2} C_{\infty} \tau_{t+1}^{-1} \phi^* + \epsilon_t'^{1/2} U_{C} \left[\left[(2 / U_C^2)(M + (A^{\pi_{\theta_t}}_{\max})^2 - \epsilon_t'/2)\right]\right]^{1/2} \psi^* = \epsilon_{t+1}^{1/2} C_{\infty} \tau_{t+1}^{-1} \phi^* + \epsilon_t'^{1/2} U_{C} Y^{1/2} \psi^*$, where $Y = 2M + 2(R_{\max} / (1 - \gamma))^2 - \epsilon_t' \le 2M + 2(R_{\max} / (1 - \gamma))^2$. Finally, we have
\begin{align}
    \min_{0\le t \le T} \{\mathcal{L}(\pi^*) - \mathcal{L}(\pi_{\theta_t})\} \le \frac{\log |\mathcal{A}| + \sum_{t=0}^{T-1} (\varepsilon_t + \varepsilon_t') + T U_{C}^2 (2 \psi^* + M)}{T L_{C} (1 - \gamma)}.\label{eq:rate proof 1}
\end{align}
By the condition (\ref{suff:2}), $U_C^2$ can always cancel out $T$ in the numerator of (\ref{eq:rate proof 1}).
Moreover, in the denominator of (\ref{eq:rate proof 1}), $L_C = \omega(T^{-1})$ is large enough to attain convergence, and we complete the proof.
\hfill \qedsymbol

\begin{remarkapp}
    {As mentioned in Remark \ref{remark:choice}, the choices of $\eta$ and $\{\tau_t\}$ would affect the convergence rate and need to be configured properly for Neural PPO-Clip with different classifiers. As will be shown in Appendix \ref{app:add:cor}, this fact can be further explained through the bounds $U_C$ and $L_C$ obtained in (\ref{eq:U_C, L_C for PPO-Clip}) and (\ref{eq:U_C, L_C for NeuralHPO-sub}).}
\end{remarkapp}

\section{Additional Corollaries and Proofs}
\label{app:add:cor}

\subsection{Proof of Corollary \ref{cor:PPO-Clip}}
\label{proof:Cor:PPO-Clip}
For ease of exposition, we restate the corollary as follows.
\begin{corollarystar}[Global Convergence of {Neural PPO-Clip} with Convergence Rate]
    Consider Neural PPO-Clip with the standard PPO-Clip classifier $\rho_{s, a}(\theta) - 1$ and the objective function $L^{(t)}(\theta)$ in each iteration $t$ as 
    \begin{align}
         \mathbb{E}_{\nu_t}[\langle \pi_{\theta_t}(\cdot|s), |A^{\pi_{\theta_t}}(s, \cdot)| \circ \ell (\sgn(A^{\pi_{\theta_t}}(s, \cdot)), \rho_{s, \cdot}(\theta) - 1, \epsilon) \rangle].
    \end{align}
    We specify the EMDA step size $\eta = 1 / \sqrt{T}$ and the temperature parameter $\tau_t = \sqrt{T} / (Kt)$. Recall that $K$ is the maximum number of EMDA iterations. Let the neural networks' widths $m_f = \Omega(R_f^{10} \phi^{*8} K^8 C_{\infty}^8 T^{12} + R_f^{10} K^8 T^8 C_{\infty}^4 |\mathcal{A}|^4)$, $m_Q = \Omega(R_Q^{10} \psi^{*8} Y^4 T^8)$, and the SGD and TD updates $T_{\text{upd}} = \Omega(R_f^4 \phi^{*4} K^4 C_{\infty}^4 T^6 + R_Q^4 \psi^{*4} Y^2 T^4 + R_f^4 T^4 K^4 C_{\infty}^2 |\mathcal{A}|^2)$, we have
    \begin{align}
        \min_{0\le t \le T} \{\mathcal{L}(\pi^*) - \mathcal{L}(\pi_{\theta_t})\} \le \frac{\log |\mathcal{A}| + K^2 (2 \psi^* + M) + O(1)}{\sqrt{T} (1 - \gamma)},
    \end{align}
     Hence, we provide the $O(1 / \sqrt{T})$ convergence rate of PPO-Clip.
\end{corollarystar}
\begin{proof}[Proof of Corollary \ref{cor:PPO-Clip}]
We find the lower and upper bounds $L_C, U_C$ for PPO-Clip. We first consider the derivative $g_{s, a}$ of the objective with the true advantage function $A^{\pi_{\theta_t}}$. 
\begin{align}
    g_{s, a} = \left.\frac{\partial L(\theta)}{\partial \theta}\right|_{\theta = \tilde{\theta}_{s, a}} = -A^{\pi_{\theta_t}}(s, a) \cdot \mathds{1}\left\{\left(\frac{\tilde{\theta}_{s, a}}{\pi_{\theta_t}(a|s)} - 1\right) \cdot \sgn (A^{\pi_{\theta_t}}(s, a)) < \epsilon \right\}.
\end{align}

Then, we check the sufficient conditions (\ref{suff:1}) and (\ref{suff:2}). Recall that $K$ is the maximum number of EMDA iteration for each $t$. We sum up the gradients with $\eta$ and rearrange the terms into $\bar{C_t}(s, a)$. Then, we have the upper bound as
\begin{align}
    \bar{C_t}(s ,a) \cdot |A^{\pi_{\theta_t}}(s, a)| \le \left[\sum_{k=0}^{K^{(t)}-1} \eta\right] \cdot |A^{\pi_{\theta_t}}(s, a)| \le K \eta \cdot |A^{\pi_{\theta_t}}(s, a)|.
\end{align}
Regarding the lower bound, as we know that under PPO-Clip, the first step of EMDA shall always make an update, i.e., it will never be clipped, and hence we have
\begin{align}
    \eta \cdot |A^{\pi_{\theta_t}}(s, a)| \le \bar{C_t}(s ,a) \cdot |A^{\pi_{\theta_t}}(s, a)|.
\end{align}
Lastly, by setting $\eta = 1 / \sqrt{T}$ and {selecting the temperature as $\tau_{t} = \sqrt{T} / (K t)$ to satisfy the condition $\tau_{t+1}^2 (U_{C}^2 + \tau_{t}^{-2}) \le 1$ that we use in (\ref{eq:PI_error 1})}, we obtain 
\begin{align}
    \omega(T^{-1}) = T^{-1/2} |A^{\pi_{\theta_t}}(s, a)| \le \bar{C_t}(s, a) \cdot |A^{\pi_{\theta_t}}(s, a)| \le K T^{-1/2} \cdot |A^{\pi_{\theta_t}}(s, a)| = O(T^{-1/2}).\label{eq:U_C, L_C for PPO-Clip}
\end{align}
We have checked the sufficient conditions of Theorem \ref{thm:main}. Thus, we obtain,
\begin{align}
    \min_{0\le t \le T} \{\mathcal{L}(\pi^*) - \mathcal{L}(\pi_{\theta_t})\} \le \frac{\log |\mathcal{A}| + \sum_{t=0}^{T-1} (\varepsilon_t + \varepsilon_t') + K^2 (2 \psi^* + M)}{\sqrt{T} (1 - \gamma)}.
\end{align}
Then, we show the minimum widths and the number of iterations of SGD and TD updates to attain convergence. We must force the summation of errors $\varepsilon_t, \varepsilon_t'$ to be $O(1)$. By Lemma \ref{lm:PE_error}, \ref{lm:PI_error}, where $\epsilon_{t+1} = O(R_f^2 T_{\text{upd}}^{-1/2} + R_f^{5/2} m_f^{-1/4} + R_f^3 m_f^{-1/2}), \epsilon_t' = O(R_Q^2 T_{\text{upd}}^{-1/2} + R_Q^{5/2} m_Q^{-1/4} + R_Q^3 m_Q^{-1/2})$, we have
\begin{align}
    C_{\infty} \tau_{t+1}^{-1} \phi^* \epsilon_{t+1}^{1/2} = &O(C_{\infty} Kt T^{-1/2} \phi^* \cdot (R_f^2 T_{\text{upd}}^{-1/2} + R_f^{5/2}m_f^{-1/4})^{1/2}), \\
    Y^{1/2} \psi^* \epsilon_t'^{1/2} = &O(Y^{1/2} \psi^* (R_Q^2 T_{\text{upd}}^{-1/2} + R_Q^{5/2} m_Q^{-1/4})^{1/2}) \\
    |\mathcal{A}|C_{\infty} \tau_{t+1}^2 \epsilon_{t+1} = &O(|\mathcal{A}|C_{\infty} K^2 t^2 T^{-1} (R_f^2 T_{\text{upd}}^{-1/2} + R_f^{5/2}m_f^{-1/4})),
\end{align}
when $m_f = \Omega(R_f^2)$ and $m_Q = \Omega(R_Q^2)$. Then, by taking $m_f = \Omega(R_f^{10} \phi^{*8} K^8 C_{\infty}^8 T^{12}), m_Q = \Omega(R_Q^{10} \psi^{*8} Y^4 T^8),$ and $T_{\text{upd}} = \Omega(R_f^4 \phi^{*4} K^4 C_{\infty}^4 T^6 + R_Q^4 \psi^{*4} Y^2 T^4)$, we have
\begin{equation}
\label{eq:cor1:vart}
\varepsilon_t = C_{\infty} \tau_{t+1}^{-1} \phi^* \epsilon_{t+1}^{1/2} + Y^{1/2} \psi^* \epsilon_t'^{1/2} = O(T^{-1}).
\end{equation}
Moreover, we further put $m_f = \Omega(R_f^{10} T^8 K^8 C_{\infty}^4 |\mathcal{A}|^4)$ and $T_{\text{upd}} = \Omega(R_f^4 T^4 K^4 C_{\infty}^2 |\mathcal{A}|^2)$, we have
\begin{equation}
\label{eq:cor1:vartprime}
    \varepsilon_t' = |\mathcal{A}|C_{\infty} \tau_{t+1}^2 \epsilon_{t+1} = O(T^{-1}).
\end{equation}
Last, we add up the lower bound of each term of $m_f, m_Q$, and $T_{\text{upd}}$, and then sum the errors in (\ref{eq:cor1:vart}) and (\ref{eq:cor1:vartprime}) for all $t$ from $0$ to $T-1$, we obtain
\begin{equation}
    \min_{0\le t \le T} \{\mathcal{L}(\pi^*) - \mathcal{L}(\pi_{\theta_t})\} \le \frac{\log |\mathcal{A}| + K^2 (2 \psi^* + M) + O(1)}{\sqrt{T} (1 - \gamma)},
\end{equation}
which complete the proof and obtain the $O(1/\sqrt{T})$ convergence rate.
\end{proof}

\subsection{Convergence Rate of Neural PPO-Clip With an Alternative Classifier}

\begin{corollary}[Global Convergence of {Neural PPO-Clip} with subtraction classifier with Convergence Rate]
\label{cor:sub}
    Consider Neural PPO-Clip with the subtraction classifier $\pi_{\theta}(a|s) - \pi_{\theta_t}(a|s)$ (\textit{termed Neural PPO-Clip-sub}) and the objective function $L^{(t)}(\theta)$ in each iteration $t$ as 
    \begin{align}
         \mathbb{E}_{\sigma_t}[|A^{\pi_{\theta_t}}(s, a)| \cdot \ell (\sgn(A^{\pi_{\theta_t}}(s, a)), \pi_{\theta}(a|s) - \pi_{\theta_t}(a|s), \epsilon)].
    \end{align}
    We specify the EMDA step size $\eta = 1 / \sqrt{T}$ and the temperature parameter $\tau_t = \sqrt{T} / (Kt)$. Recall that $K$ is the maximum number of EMDA iterations. Let the neural networks' widths $m_f = \Omega(R_f^{10} \phi^{*8} K^8 C_{\infty}^8 T^{12} + R_f^{10} K^8 T^8 C_{\infty}^4 |\mathcal{A}|^4)$, $m_Q = \Omega(R_Q^{10} \psi^{*8} Y^4 T^8)$, and the SGD and TD updates $T_{\text{upd}} = \Omega(R_f^4 \phi^{*4} K^4 C_{\infty}^4 T^6 + R_Q^4 \psi^{*4} Y^2 T^4 + R_f^4 T^4 K^4 C_{\infty}^2 |\mathcal{A}|^2)$, we have
    \begin{align}
        \min_{0\le t \le T} \{\mathcal{L}(\pi^*) - \mathcal{L}(\pi_{\theta_t})\} \le \frac{\log |\mathcal{A}| + K^2 (2 \psi^* + M) + O(1)}{\sqrt{T} (1 - \gamma)},
    \end{align}
     Hence, we provide the $O(1 / \sqrt{T})$ convergence rate of Neural PPO-Clip-sub.
\end{corollary}

\begin{proof}[Proof of Corollary \ref{cor:sub}]
    Similar to Corollary \ref{cor:PPO-Clip}, we derive the gradient of our objective with the true advantage function $A^{\pi_{\theta_t}}(s, a)$. Specifically, we have
    \begin{align}
        g_{s, a} = \left.\frac{\partial L(\theta)}{\partial \theta}\right|_{\theta = \tilde{\theta}_{s, a}} = -A^{\pi_{\theta_t}}(s, a) \cdot \mathds{1}\left\{\left(\tilde{\theta}_{s, a} - \pi_{\theta_t}(a|s)\right) \cdot \sgn (A^{\pi_{\theta_t}}(s, a)) < \epsilon \right\}.
    \end{align}
    Thus, similar to \ref{proof:Cor:PPO-Clip}, we have
    \begin{align}
        \eta \cdot |A^{\pi_{\theta_t}}(s, a)| \le C_t(s ,a) \cdot |A^{\pi_{\theta_t}}(s, a)| \le K \eta \cdot |A^{\pi_{\theta_t}}(s, a)|.
    \end{align}
    We also set $\eta = 1 / \sqrt{T}$ and {pick $\tau_{t} = \sqrt{T} / (K t)$ to satisfy the condition $\tau_{t+1}^2 (U_{C}^2 + \tau_{t}^{-2}) \le 1$ that we use in (\ref{eq:PI_error 1})}. Accordingly, we obtain 
    \begin{align}
        \omega(T^{-1}) = T^{-1/2} |A^{\pi_{\theta_t}}(s, a)| \le C_t(s, a) \cdot |A^{\pi_{\theta_t}}(s, a)| \le K T^{-1/2} \cdot |A^{\pi_{\theta_t}}(s, a)| = O(T^{-1/2}).\label{eq:U_C, L_C for NeuralHPO-sub}
    \end{align}
    We have checked the sufficient condition of Theorem \ref{thm:main}. Therefore, by plugging in $L_C$ and $U_C$, we obtain
    \begin{align}
        \min_{0\le t \le T} \{\mathcal{L}(\pi^*) - \mathcal{L}(\pi_{\theta_t})\} \le \frac{\log |\mathcal{A}| + \sum_{t=0}^{T-1} (\varepsilon_t + \varepsilon_t') + K^2 (2 \psi^* + M )}{\sqrt{T} (1 - \gamma)}.
    \end{align}
    Similar to the proof of Corollary \ref{proof:Cor:PPO-Clip}, we set the same minimum widths and number of iterations to attain convergence, which directly implies 
    \begin{equation}
        \min_{0\le t \le T} \{\mathcal{L}(\pi^*) - \mathcal{L}(\pi_{\theta_t})\} \le \frac{\log |\mathcal{A}| + K^2 (2 \psi^* + M) + O(1)}{\sqrt{T} (1 - \gamma)}.
    \end{equation}
    Then, we complete the proof and obtain the $O(1/\sqrt{T})$ convergence rate of PPO-Clip with a subtraction classifier.
\end{proof}

\section{Tabular PPO-Clip and Proof}
\label{app:mini-batch thm}

\subsection{Tabular Case: PPO-Clip with Direct Policy Parameterization}

In this section, we study the global convergence of PPO-Clip with direct parameterization, i.e., 
policies are parameterized by $\pi(a|s)=\theta_{s,a}$, where $ \theta_s\in\Delta(\mathcal{A})$ denotes the vector $\theta_{s,\cdot}$ and $\theta\in\Delta(\mathcal{A})^{\lvert\mathcal{S}\rvert}$.
We use $V^{(t)}(s)$ and $A^{(t)}(s,a)$ as the shorthands for $V^{\pi^{(t)}}(s)$ and $A^{\pi^{(t)}}(s,a)$, respectively.

\vspace{2mm}
\noindent
\textbf{PPO-Clip with Direct Policy Parameterization.}
PPO-Clip proceeds iteratively. 
Define the generalized sample-based loss function of PPO-Clip for each iteration $t$ as
  \begin{align}
      \label{eq:sampleavg}
       &\hat{L}^{(t)}(\theta):=\frac{1}{\lvert \cD^{(t)}\rvert}\sum_{(s,a)\in\mathcal{D}^{(t)}} {W}^{(t)}({s,a}) \ell\big(\sgn({A}^{(t)}(s,a)), h(\rho_{s,a}^{(t)}(\theta)), \epsilon\big),
  \end{align}
where $\cD^{(t)}$ denotes the batch of state-action pairs, $h(\rho_{s,a}^{(t)}(\theta))$ denotes the classifier, and $W^{(t)}(s,a)\in (0,W_{\max}]$ denotes the weight in (\ref{eq:HPO loss}) for each $(s,a)$.
Note that by choosing $W^{(t)}(s,a)=\lvert {A}^{(t)}(s,a)\rvert$ and $h(\rho_{s,a}^{(t)}(\theta))=\rho_{s,a}^{(t)}(\theta)-1$, the generalized objective would recover the form of the objective of PPO-Clip.
PPO-Clip with direct parameterization is shown in Algorithm \ref{algo:HPO-AM}. 
In each iteration, PPO-Clip updates the policy by minimizing the loss in (\ref{eq:sampleavg}) via the entropic mirror descent algorithm (EMDA) \citep{beck2003mirror}. 
While there are alternative ways to minimize the loss $\hat{L}^{(t)}(\theta)$ over $\Delta(\mathcal{A})^{\lvert \mathcal{S}\rvert}$ (e.g., the projected subgradient method), we leverage EMDA for the following two reasons:
(i) PPO-Clip achieves policy improvement by increasing or decreasing the probability of those state-action pairs in $\cD^{(t)}$ based on the sign of ${{A}^{(t)}(s,a)}$ as well as properly reallocating the probabilities of those state-action pairs \textit{not} contained in the batch (to ensure the probability sum is one). Using EMDA allows us to enforce a proper reallocation in PPO-Clip, as shown later in the proof of Theorem \ref{thm:mini-batch} in Appendix \ref{app:mini-batch thm}; (ii) The exponentiated gradient scheme of EMDA ensures that $\pi^{(t)}$ remains strictly positive for all state-action pairs in each iteration $t$, and this ensures that the probability ratio $\rho_{s,a}(\theta)$ used by PPO-Clip is always well-defined.
In this section, we consider the stylized setting with tabular policy and true advantage mainly for motivating the PPO-Clip method and its analysis. 


\vspace{2mm}
\noindent
\textbf{Global Convergence of PPO-Clip With Direct Parameterization.} 
We first make the following assumptions. 


\begin{assumption}[Infinite Visitation to Each State-Action Pair]
\label{assumption:infinite visit}
    Each state-action pair $(s,a)$ appears infinitely often in $\{\cD^{(\tau)}\}$, i.e., $\lim_{t\rightarrow \infty} \sum_{\tau=0}^{t}\mathds{1}\{(s,a)\in \cD^{(\tau)}\}=\infty$, with probability one.
\end{assumption}

\vspace{-2mm}
\begin{assumption}
\label{assumption:distinct states}
    In each iteration $t$, the state-action pairs in $\cD^{(t)}$ have distinct states.
\end{assumption}

Assumption \ref{assumption:infinite visit} resembles the standard infinite-exploration condition commonly used in the temporal-difference methods, such as Sarsa \citep{singh2000convergence}. Assumption \ref{assumption:distinct states} is rather mild: 
(i) This can be met by post-processing the mini-batch of state-action pairs via an additional sub-sampling step; (ii) In most RL problems with discrete actions, the state space is typically much larger than the action space. 

\begin{theorem}[Global Convergence of PPO-Clip]
\label{thm:mini-batch}
Under PPO-Clip, we have $V^{(t)}(s)\rightarrow V^{\pi^*}(s)\text{ as }t\rightarrow\infty,\ \forall s\in\mathcal{S}$, with probability one.

\end{theorem}

The proof of Theorem \ref{thm:mini-batch} is provided in Appendix \ref{app:mini-batch thm}. 
We highlight the main ideas behind the proof of Theorem \ref{thm:mini-batch}:
(i) \textit{State-wise policy improvement:} Through the lens of generalized objective, we show that PPO-Clip enjoys state-wise policy improvement in every iteration with the help of the EMDA subroutine. 
This property greatly facilitates the rest of the convergence analysis.
(ii) \textit{Quantifying the probabilities of those actions with positive or negative advantages in the limit}: By (i), we know the limits of the value functions and the advantage function all exist. Then, we proceed to show that the actions with positive advantages in the limit cannot exist by establishing a contradiction.

The above also manifests how reinterpreting PPO-Clip helps with establishing the convergence guarantee.

\noindent
\textbf{Global Convergence of PPO-Clip With Alternative Classifiers.}
Notably, the PPO-Clip update scheme and the convergence analysis of Theorem~\ref{thm:mini-batch} can be readily extended to various other classifiers, where we show the guarantees in Appendix \ref{app:secondthm}.





\begin{algorithm}
\caption{Tabular PPO-Clip}
\label{algo:HPO-AM}
    \begin{algorithmic}[1]
        \State {\bfseries Initialization:} policy $\pi^{(0)}=\pi(\theta^{(0)})$, initial state distribution $\mu$, step size of EMDA $\eta$, number of EMDA iterations $K^{(t)}$\;
        \For{$t=0, 1, \cdots$}
            \State Collect a set of trajectories $\tau \in \mathcal{D}^{(t)}$ under policy $\pi^{(t)}=\pi(\theta^{(t)})$\;
            \State Find ${A}^{(t)}$ by a policy evaluation method\;
            \State Compute $\hat{L}^{(t)}(\theta)$ based on ${A}^{(t)}$ and the collected samples in $\mathcal{D}^{(t)}$\;
            \State Update the policy by $\theta^{(t+1)}=\text{EMDA}(\hat{L}^{(t)}(\theta),\eta,K^{(t)},\cD^{(t)},\theta^{(t)})$\;
        \EndFor
        \State {\bfseries Output:} Learned policy $\pi^{(\infty)}$\;
    \end{algorithmic}
\end{algorithm}

\begin{algorithm}
\caption{EMDA$(L(\theta),\eta,K,\cD, \theta_{\text{init}})$}
\label{algo:EMD}
    \begin{algorithmic}[1]
        \State {\bfseries Input:} Objective $L(\theta)$, step size $\eta$, number of iteration $K$,  dataset $\cD$, and initial parameter $\theta_{\text{init}}$\;
        \State {\bfseries Initialization:} $\widetilde{\theta}^{(0)}=\theta_{\text{init}}$, $\widetilde{\theta}=\theta_{\text{init}}$\;
        \For{$k=0,\cdots,K-1$}
            \For{each state $s$ in $\cD$}
                \vspace{1mm}
                \State Find $g_{s,a}^{(k)}=\frac{\partial L(\theta)}{\partial \theta_{s,a}}\rvert_{\theta=\widetilde{\theta}^{(k)}}$, for each $a$\;
                \State Let $w_s=(e^{-\eta g_{s,1}^{(k)}},\cdots,e^{-\eta g_{s,\lvert \mathcal{A}\rvert}^{(k)}})$\;
                \State $\widetilde{\theta}_s^{(k+1)}=\frac{1}{\langle w_s, \widetilde{\theta}_s^{(k)}\rangle}(w_s \circ \widetilde{\theta}_s^{(k)})$\;
            \EndFor
        \EndFor
        \State {\bfseries Output:} Learned parameter $\tilde{\theta}$\;
    \end{algorithmic}
\end{algorithm}

\subsection{Supporting Lemmas for the Proof of Theorem~\ref{thm:mini-batch}}

As described in Section \ref{section:HPO}, one major component of the proof of Theorem \ref{thm:mini-batch} is the state-wise policy improvement property of PPO-Clip.
For ease of exposition, we introduce the following definition regarding the partial ordering over policies.
\begin{definition}[Partial ordering over policies]
    Let $\pi_{1}$ and $\pi_{2}$ be two policies. 
    Then, $\pi_{1}\geq\pi_{2}$, called \textit{$\pi_{1}$ improves upon $\pi_{2}$}, if and only if $V^{\pi_{1}}(s)\geq V^{\pi_{2}}(s),\ \forall s\in\mathcal{S}$. Moreover, we say $\pi_1>\pi_2$, called \textit{$\pi_{1}$ strictly improves upon $\pi_{2}$}, if and only if $\pi_{1}\geq\pi_{2}$ and there exists at least one state $s$ such that $V^{\pi_{1}}(s)> V^{\pi_{2}}(s)$.
\end{definition}

\begin{lemma}[Sufficient condition of state-wise policy improvement]
\label{prop:first}
    Given any two policies $\pi_{1}$ and $\pi_{2}$, we have ${\pi_{1}}\geq {\pi_{2}}$ if the following condition holds:
    \begin{equation}
        \sum_{a\in\mathcal{A}}\pi_{1}(a|s)A^{\pi_{2}}(s,a)\geq0,\ \forall s\in\mathcal{S}.
    \end{equation}
\end{lemma}
\begin{proof}[Proof of Lemma \ref{prop:first}]
This is a direct result of the performance difference lemma \cite{kakade2002}.
\end{proof}
Next, we present two critical properties that hold under PPO-Clip for every sample path.
\begin{lemma}[Strict improvement and strict positivity of policy under PPO-Clip with direct tabular parameterization]
\label{lemma:policy improvement mini-batch}
In any iteration $t$, suppose $\pi^{(t)}$ is strictly positive in all state-action pairs, i.e., $\pi^{(t)}(a\rvert s)>0$, for all $(s,a)$. 
Under PPO-Clip in Algorithm \ref{algo:HPO-AM}, $\pi^{(t+1)}$ satisfies that (i) $\pi^{(t+1)}>\pi^{(t)}$ and (ii) $\pi^{(t+1)}(a\rvert s)>0$, for all $(s,a)$.
\end{lemma}
\begin{proof}[Proof of Lemma \ref{lemma:policy improvement mini-batch}]
Consider the $t$-th iteration of PPO-Clip (cf. Algorithm \ref{algo:HPO-AM}) and the corresponding update from $\pi^{(t)}$ to $\pi^{(t+1)}$.
Regarding (ii), recall from Algorithm \ref{algo:EMD} that $K^{(t)}$ denotes the number of iterations undergone by the EMDA subroutine for the update from $\pi^{(t)}$ to $\pi^{(t+1)}$ and that $K^{(t)}$ is designed to be finite.
Therefore, it is easy to verify that $\pi^{(t+1)}(a\rvert s)>0$ for all $(s,a)$ by the exponentiated gradient update scheme of EMDA and the strict positivity of $\pi^{(t)}$.

Next, for ease of exposition, for each $k\in \{0,1,\cdots,K^{(t)}\}$ and for each state-action pair $(s,a)$, let $\widetilde{\theta}^{(k)}_{s,a}$ denote the policy parameter after $k$ EMDA iterations.
Regarding (i), recall that we define $g^{(k)}_{s,a}:=\frac{\partial \mathcal{L}(\theta)}{\partial \theta_{s,a}}\big\rvert_{\theta=\widetilde{\theta}^{(k)}_{s}}$ and $w_s^{(k)}:=(e^{-\eta g^{(k)}_{s,1}},\cdots,e^{-\eta g^{(k)}_{s,\lvert \mathcal{A}\rvert}})$.
Note that as the weights in the loss function only affects the effective step sizes of EMDA, we simply set the weights of PPO-Clip to be one, without loss of generality.
By EMDA in Algorithm \ref{algo:EMD}, for every $(s,a)\in \cD^{(t)}$, we have
\begin{equation}
\label{eq:pi_t+1 and pi_t}
    \pi^{(t+1)}(a\rvert s)=\frac{\prod_{k=0}^{K^{(t)}-1}\exp({-\eta}g_{s,a}^{(k)})}{{\prod_{k=0}^{K^{(t)}-1}}\langle w_s^{(k)},\widetilde{\theta}_s^{(k)}\rangle}\cdot\pi^{(t)}(a\rvert s).
\end{equation}
Note that $g^{(k)}_{s,a}$ can be written as
\begin{align}
\label{eq:g_sa^k}
    g^{(k)}_{s,a}=
    \begin{cases}
        -\frac{1}{\pi^{(t)}(a\rvert s)}\sgn({A}^{(t)}(s,a))&, \text{if }\big(\frac{\widetilde{\theta}_{s,a}^{(k)}}{\pi^{(t)}(a\rvert s)}-1\big)\sgn({A}^{(t)}(s,a))< \epsilon, (s,a)\in \cD^{(t)}\\ 
        0&, \text{otherwise }    
    \end{cases}
\end{align}
By (\ref{eq:pi_t+1 and pi_t})-(\ref{eq:g_sa^k}), it is easy to verify that for those $(s,a)\in \cD^{(t)}$ with positive advantage, we must have $\prod_{k=0}^{K^{(t)}-1}\exp({-\eta}g_{s,a}^{(k)})>1$.
Similarly, for those $(s,a)\in \cD^{(t)}$ with negative advantage, we have $\prod_{k=0}^{K^{(t)}-1}\exp({-\eta}g_{s,a}^{(k)})<1$.
Now we are ready to check the condition of strict policy improvement given by Lemma \ref{prop:first}: For each $s\in \cS$, we have
\begin{align}
\label{eq:sufficient condition mini-batch}
    \sum_{a\in \cA}\pi^{(t+1)}(a\rvert s)A^{(t)}(a\rvert s)=\frac{1}{\prod_{k=0}^{K^{(t)}-1}\langle w_s^{(k)},\widetilde{\theta}_s^{(k)}\rangle}\sum_{a\in \cA}\Big(\prod_{k=0}^{K^{(t)}-1}\exp({-\eta}g_{s,a}^{(k)})\Big)\pi^{(t)}(a\rvert s)A^{(t)}(a\rvert s)>0.
\end{align}
Hence, we conclude that the strict state-wise policy improvement property indeed holds, i.e., $\pi^{(t+1)}>\pi^{(t)}$.
\end{proof}

Note that Lemma \ref{lemma:policy improvement mini-batch} implies that the limits $V^{(\infty)}(s)$, $Q^{(\infty)}(s,a)$, $A^{(\infty)}(s,a)$ exist, for every sample path: By the strict policy improvement shown in Lemma \ref{lemma:policy improvement mini-batch}, we know that the sequence of state values is point-wise monotonically increasing, i.e., $V^{(t+1)}(s)\geq V^{(t)}(s),\ \forall s\in\mathcal{S}$. 
Moreover, by the bounded reward and the discounted setting, we have $-\frac{R_{\max}}{1-\gamma}\leq V^{(t)}(s)\leq \frac{R_{\max}}{1-\gamma}$. 
The above monotone increasing property and boundedness imply convergence, i.e., $V^{(t)}(s)\rightarrow V^{(\infty)}(s)$, for each sample path.
Similarly, we also know that $Q^{(t)}(s,a)\rightarrow Q^{(\infty)}(s,a)$, and thus $A^{(t)}(s,a)\rightarrow A^{(\infty)}(s,a)$.
As a result, we can define the three sets $I_{s}^{+}$, $I_{s}^{0}$ and $I_{s}^{-}$ as
\begin{align}
    I_{s}^{+} & :=\{a\in\mathcal{A}|A^{(\infty)}(s,a)>0\}, \\
    I_{s}^{0} & :=\{a\in\mathcal{A}|A^{(\infty)}(s,a)=0\}, \\
    I_{s}^{-} & :=\{a\in\mathcal{A}|A^{(\infty)}(s,a)<0\}.
\end{align}
Note that for each sample path, the sets $I_{s}^{+}$, $I_{s}^{0}$ and $I_{s}^{-}$ are well-defined, based on the limit $A^{(\infty)}(s,a)$.

\begin{lemma}
\label{lemma:pi convergence mini-batch}
Conditioned on the event that each state-action pair occurs infinitely often in $\{\cD^{(t)}\}$, if $I_s^{+}$ is not an empty set, then we have $\sum_{a\in I_{s}^{-}}\pi^{(t)}(a\rvert s)\rightarrow 0$, as $t\rightarrow \infty$. 
\end{lemma}
\begin{proof}[Proof of Lemma \ref{lemma:pi convergence mini-batch}]
We discuss each state separately as it is sufficient to show that for each state $s$, given some fixed $a'\in I_{s}^{+}$, for any $a''\in I_{s}^{-}$, we have $\frac{\pi^{(t)}(a''\rvert s)}{\pi^{(t)}(a'\rvert s)}\rightarrow 0$, as $t\rightarrow \infty$.
For ease of exposition, we reuse some of the notations from the proof of Lemma \ref{lemma:policy improvement mini-batch}.
Recall that we let $K^{(t)}$ denote the number of iterations undergone by the EMDA subroutine for the update from $\pi^{(t)}$ to $\pi^{(t+1)}$, and $K^{(t)}$ is designed to be finite.
For each $k\in \{0,1,\cdots,K^{(t)}\}$ and for each state-action pair $(s,a)$, let $\widetilde{\theta}^{(k)}_{s,a}$ denote the policy parameter after $k$ EMDA iterations.
Recall from Algorithm \ref{algo:EMD} that $g^{(k)}_{s,a}:=\frac{\partial \mathcal{L}(\theta)}{\partial \theta_{s,a}}\big\rvert_{\theta=\widetilde{\theta}^{(k)}_{s}}$ and $w_s^{(k)}:=(e^{-\eta g^{(k)}_{s,1}},\cdots,e^{-\eta g^{(k)}_{s,\lvert \mathcal{A}\rvert}})$.
Define $\Delta_*:=\min_{a\in I_{s}^{+}\cup I_{s}^{-}}\lvert A^{(\infty)}(s,a)\rvert>0$ (and here $\Delta_*$ is a random variable as $A^{(\infty)}(s,a)$ is defined with respect to each sample path).
By the definition of $I_s^{+}$, $I_s^{-}$ and $\Delta_*$, we know that for each sample path, there must exist finite $T_{+}$ and $T_{-}$ such that: (i) for every $a\in I_{s}^{+}$, $A^{(t)}(s,a)\geq \frac{\Delta_*}{2}$, for all $t>T_{+}$, and (ii) for every $a\in I_{s}^{-}$, $A^{(t)}(s,a)\leq -\frac{\Delta_*}{2}$, for all $t>T_{-}$.
Under Assumption \ref{assumption:distinct states}, at each iteration $t$ with $t>\max\{T_{+},T_{-}\}$, there are three possible cases regarding the state-action pairs $(s,a')$ and $(s,a'')$:

\vspace{-2mm}
\begin{itemize}[leftmargin=*]
    \item \textbf{Case 1:} $(s,a')\in \cD^{(t)}$, $(s,a'')\notin \cD^{(t)}$\\
    By the EMDA subroutine and (\ref{eq:pi_t+1 and pi_t}), we have
    \begin{align}
    \label{eq:case 1}
        \frac{\pi^{(t+1)}(a''\rvert s)}{\pi^{(t+1)}(a'\rvert s)}=\frac{\pi^{(t)}(a''\rvert s)}{\pi^{(t)}(a'\rvert s)}\cdot \prod_{k=0}^{K^{(t)}-1}\exp({\eta}g_{s,a'}^{(k)})\leq \frac{\pi^{(t)}(a''\rvert s)}{\pi^{(t)}(a'\rvert s)}\cdot \underbrace{\exp(-\eta)}_{<1},
    \end{align}
    where the last inequality holds by (\ref{eq:g_sa^k}), $a'\in I_{s}^{+}$, and $\pi^{(t)}(a'\rvert s)\leq 1$.
    \item \textbf{Case 2:} $(s,a')\notin \cD^{(t)}$, $(s,a'')\in \cD^{(t)}$\\
    By the EMDA subroutine, we have $-g_{s,a''}^{(0)}<0$ and $-g_{s,a''}^{(k)}\leq 0$ for all $k\in \{1,\cdots, K^{(t)}\}$. Therefore, we have
    \begin{align}
    \label{eq:case 2}
        \frac{\pi^{(t+1)}(a''\rvert s)}{\pi^{(t+1)}(a'\rvert s)}<\frac{\pi^{(t)}(a''\rvert s)}{\pi^{(t)}(a'\rvert s)}.
    \end{align}    
    \item \textbf{Case 3:} $(s,a')\notin \cD^{(t)}$, $(s,a'')\notin \cD^{(t)}$\\
    Under EMDA, as neither $(s,a')$ nor $(s,a'')$ is in $\notin \cD^{(t)}$, the action probability ratio between these two actions remains unchanged (despite that the values of $\pi^{(t)}(a''\rvert s)$ and $\pi^{(t)}(a''\rvert s)$ can still change if there is an action $a'''$ such that $a'''\neq a'$, $a'''\neq a''$, and $(s,a''')\in \cD^{(t)}$), i.e.,
    \begin{align}
    \label{eq:case 3}
        \frac{\pi^{(t+1)}(a''\rvert s)}{\pi^{(t+1)}(a'\rvert s)}=\frac{\pi^{(t)}(a''\rvert s)}{\pi^{(t)}(a'\rvert s)}.
    \end{align}
\end{itemize}
Conditioned on the event that each state-action pair occurs infinitely often in $\{\cD^{(t)}\}$, we know Case 1 and (\ref{eq:case 3}) must occur infinitely often. 
By (\ref{eq:case 1})-(\ref{eq:case 3}),  we conclude that $\frac{\pi^{(t)}(a''\rvert s)}{\pi^{(t)}(a'\rvert s)}\rightarrow 0$, as $t\rightarrow \infty$, for every $a''\in I_{s}^{-}$.
\end{proof}

\begin{lemma}
\label{lemma:epsilon lower bound mini-batch}
Conditioned on the event that each state-action pair occurs infinitely often in $\{\cD^{(t)}\}$, if $I_s^{+}$ is not an empty set, then there exists some constant $c>0$ such that $\sum_{a\in I_{s}^{-}}\pi^{(t)}(a\rvert s)\geq c$, for infinitely many $t$.
\end{lemma}
\begin{proof}[Proof of Lemma \ref{lemma:epsilon lower bound mini-batch}]
For each $(s,a)$, define $\sT_{s,a}:=\{t: (s,a)\in \cD^{(t)}\}$ to be the index set that collects the time indices at which $(s,a)$ is contained in the mini-batch.
Given that each state-action pair occurs infinitely often, we know $\sT_{s,a}$ is a (countably) infinite set.

For ease of exposition, define a positive constant $\chi$ as
\begin{equation}
    \chi:=\frac{e\cdot \eta}{e\cdot \eta+1}<1.
\end{equation}
Define $\Delta:=\min_{a\in I_{s}^{+}}A^{(\infty)}(s,a)>0$ (and here $\Delta$ is a random variable as $A^{(\infty)}(s,a)$ is defined with respect to each sample path).
By the definition of $I_s^{+}$ and $\Delta$, we know that there must exist a finite $T^{(+)}$ such that for every $a\in I_{s}^{+}$, $A^{(t)}(s,a)\geq \frac{3\Delta}{4}$, for all $t>T^{(+)}$.
Similarly, by the definition of $I_s^{0}$, there must exist a finite $T^{(0)}$ such that for every $a\in I_{s}^{0}$, $\lvert A^{(t)}(s,a)\rvert \leq \frac{\chi\Delta}{4}$, for all $t>T^{(0)}$.
We also define $T^*:=\max\{T^{(+)}, T^{(0)}\}$.

We reuse some of the notations from the proof of Lemma \ref{lemma:policy improvement mini-batch}.
Recall that we let $K^{(t)}$ denote the number of iterations undergone by the EMDA subroutine for the update from $\pi^{(t)}$ to $\pi^{(t+1)}$, and $K^{(t)}$ is a finite positive integer.
For ease of exposition, for each $k\in \{0,1,\cdots,K^{(t)}\}$ and for each state-action pair $(s,a)$, let $\widetilde{\theta}^{(k)}_{s,a}$ denote the policy parameter after $k$ EMDA iterations.
Recall that we define $g^{(k)}_{s,a}:=\frac{\partial \mathcal{L}(\theta)}{\partial \theta_{s,a}}\big\rvert_{\theta=\widetilde{\theta}^{(k)}_{s}}$ and $w_s^{(k)}:=(e^{-\eta g^{(k)}_{s,1}},\cdots,e^{-\eta g^{(k)}_{s,\lvert \mathcal{A}\rvert}})$.
If $I_{s}^{+}$ is not an empty set, then we can select an arbitrary action $a'\in I_{s}^{+}$.
For any $t$ with $t>T^{(+)}$ and $t\in \sT_{s,a'}$, by (\ref{eq:pi_t+1 and pi_t}) we have
\begin{align}
    \pi^{(t+1)}(a'\rvert s)& =\frac{\prod_{k=0}^{K^{(t)}-1}\exp({-\eta}g_{s,a'}^{(k)})}{{\prod_{k=0}^{K^{(t)}-1}}\langle w_s^{(k)},\widetilde{\theta}_s^{(k)}\rangle}\cdot\pi^{(t)}(a'\rvert s)\label{eq:pi_t+1 1}\\
    &\geq \frac{\pi^{(t)}(a'\rvert s) \exp(-\eta g_{s,a'}^{(0)})}{\pi^{(t)}(a'\rvert s) \exp(-\eta g_{s,a'}^{(0)})+1}\label{eq:pi_t+1 2}\\
    &\geq \frac{\pi^{(t)}(a'\rvert s) \exp(\eta/\pi^{(t)}(a'\rvert s))}{\pi^{(t)}(a'\rvert s) \exp(\eta/ \pi^{(t)}(a'\rvert s))+1}\label{eq:pi_t+1 3}\\
    &\geq \frac{e\cdot\eta}{e\cdot\eta+1}=\chi, \label{eq:pi_t+1 4}
\end{align}
where (\ref{eq:pi_t+1 2}) holds due to the fact that $\widetilde{\theta}_{s,a}^{(k)}$ is non-decreasing with $k$ under Assumption \ref{assumption:distinct states} and that $K^{(t)}\geq 1$, (\ref{eq:pi_t+1 3}) follows from (\ref{eq:g_sa^k}) and that $a'\in I_{s}^{+}$, and (\ref{eq:pi_t+1 4}) holds by that the function $q(z)=z\cdot \exp(\eta/z)$ has a unique minimizer at $z=\eta$ with minimum value $e\cdot\eta$.
For all $t$ that satisfies $(t-1)\in \sT_{s,a}$ and $t>T^*$, we have
\begin{align}
    \sum_{a\in I_{s}^{-}}\pi^{(t)}(a|s)&\geq\frac{\sum_{a\in I_{s}^{+}}\pi^{(t)}(a|s)A^{(t)}(s,a) + \sum_{a\in I_{s}^{0}}\pi^{(t)}(a|s)A^{(t)}(s,a)}{\max_{a\in I_{s}^{-}} \lvert A^{(t)}(s,a)\rvert}\label{eq:lower bound mini-batch case 1-2}\\
    &\geq\frac{\chi(3\Delta/4) - 1\cdot (\chi\Delta/4) }{\frac{2R_{\max}}{1-\gamma}}\label{eq:lower bound mini-batch case 1-3}\\
    &=\frac{\chi\Delta }{\frac{4R_{\max}}{1-\gamma}},\label{eq:lower bound mini-batch case 1-4}
\end{align}
where (\ref{eq:lower bound mini-batch case 1-2}) follows from that $\sum_{a\in \cA}\pi^{(t)}(a\rvert s)=0$ and $A^{(t)}(s,a)<0$ for all $a\in I_{s}^{-}$, and (\ref{eq:lower bound mini-batch case 1-3}) follows from the definition of $T^{(+)}, T^{(0)}$ as well as the boundedness of rewards.
Since $\sT_{s,a}$ is a countably infinite set, we know $ \sum_{a\in I_{s}^{-}}\pi^{(t)}(a|s)\geq \frac{\chi\Delta}{\frac{4R_{\max}}{1-\gamma}}$, for infinitely many $t$.

\end{proof}

\subsection{Proof of Theorem~\ref{thm:mini-batch}}
Now we are ready to show Theorem~\ref{thm:mini-batch}. 
For ease of exposition, we restate Theorem~\ref{thm:mini-batch} as follows.
\begin{theoremstar}[Global Convergence of PPO-Clip]
Under PPO-Clip, we have $V^{(t)}(s)\rightarrow V^{\pi^*}(s)\text{ as }t\rightarrow\infty,\ \forall s\in\mathcal{S}$, with probability one.
\end{theoremstar}
\begin{proof}
We establish that $\pi^{(t)}$ converges to an optimal policy by showing that $I_{s}^{+}$ is an empty set for all $s$.
Under Assumption~\ref{assumption:infinite visit}, the analysis below is presumed to be conditioned on the event that each state-action pair occurs infinitely often in $\{\cD^{(t)}\}$.
The proof proceeds by contradiction as follows:
Suppose $I_{s}^{+}$ is non-empty.
From Lemma~\ref{lemma:pi convergence mini-batch}, we have that $\sum_{a\in I_{s}^{-}}\pi^{(t)}(a\rvert s)\rightarrow 0$, as $t\rightarrow \infty$. 
However, Lemma \ref{lemma:epsilon lower bound mini-batch} suggests that there exists some constant $c>0$ such that $\sum_{a\in I_{s}^{-}} \pi^{(t)}(a\rvert s)\geq c$ infinitely often.
This leads to a contraction, and thus completes the proof.  
\end{proof}

\section{{Global Convergence of Tabular PPO-Clip With Alternative Classifiers}}
\label{app:secondthm}
\begin{theorem}
\label{thm:second}
    Theorem \ref{thm:mini-batch} also holds under the following algorithms: 
    (i) {PPO-Clip} with the classifier $\log(\pi_{\theta}(a|s))-\log(\pi(a|s))$ (termed {PPO-Clip-log});
    (ii) {PPO-Clip} with the classifier $\sqrt{\rho_{s,a}(\theta)}-1$ (termed {PPO-Clip-root}).
\end{theorem}
\begin{proof}[Proof of Theorem \ref{thm:second}]
We show that Theorem~\ref{thm:mini-batch} can be extended to these two alternative classifiers by following the proof procedure of Theorem~\ref{thm:mini-batch}. Specifically, we extend the supporting lemmas (cf. Lemma \ref{lemma:policy improvement mini-batch}, Lemma \ref{lemma:pi convergence mini-batch}, and Lemma \ref{lemma:epsilon lower bound mini-batch}) as follows:
\begin{itemize}[leftmargin=*]
    \item To extend Lemma \ref{lemma:policy improvement mini-batch} to the alternative classifiers, we can reuse (\ref{eq:pi_t+1 and pi_t}) and rewrite (\ref{eq:g_sa^k HPO-AM-log}) for each classifier. That is, for PPO-Clip-log, we have
    \begin{align}
    \label{eq:g_sa^k HPO-AM-log}
    g^{(k)}_{s,a}=
    \begin{cases}
        -\frac{1}{\widetilde{\theta}_{s,a}^{(k)}}\sgn({A}^{(t)}(s,a))&, \text{if }\log\big(\frac{\widetilde{\theta}_{s,a}^{(k)}}{\pi^{(t)}(a\rvert s)}\big)\sgn({A}^{(t)}(s,a))< \epsilon, (s,a)\in \cD^{(t)}\\ 
        0&, \text{otherwise }    
    \end{cases}
    \end{align}
    On the other hand, for PPO-Clip-root, we have
    \begin{align}
    \label{eq:g_sa^k HPO-AM-root}
    g^{(k)}_{s,a}=
    \begin{cases}
        -\frac{1}{2\sqrt{\widetilde{\theta}_{s,a}^{(k)}\pi^{(t)}(a\rvert s)}}\sgn({A}^{(t)}(s,a))&, \text{if }\Big(\sqrt{\frac{\widetilde{\theta}_{s,a}^{(k)}}{\pi^{(t)}(a\rvert s)}}-1\Big)\sgn({A}^{(t)}(s,a))< \epsilon, (s,a)\in \cD^{(t)}\\ 
        0&, \text{otherwise }    
    \end{cases}
    \end{align}
    As the sign of $g_{s,a}^{(k)}$ depends only on the sign of the advantage, it is easy to verify that (\ref{eq:sufficient condition mini-batch}) still goes through and hence the sufficient condition of Lemma \ref{prop:first} is satisfied under these two alternative classifiers.
    Moreover, by using the same argument of EMDA as that in Lemma \ref{lemma:policy improvement mini-batch}, it is easy to verify that $\pi^{(t+1)}(a\rvert s)>0$ for all $(s,a)$.
    \vspace{-2pt}
    \item Regarding Lemma~\ref{lemma:pi convergence mini-batch}, we can extend this result again by considering the three cases as in Lemma~\ref{lemma:pi convergence mini-batch}. For Case 1, given the $g_{s,a}^{(k)}$ in (\ref{eq:g_sa^k HPO-AM-log}) and (\ref{eq:g_sa^k HPO-AM-root}), we have: For PPO-Clip-log,
    \begin{align}
    \label{eq:case 1 HPO-AM-log}
        \frac{\pi^{(t+1)}(a''\rvert s)}{\pi^{(t+1)}(a'\rvert s)}=\frac{\pi^{(t)}(a''\rvert s)}{\pi^{(t)}(a'\rvert s)}\cdot \prod_{k=0}^{K^{(t)}-1}\exp({\eta}g_{s,a'}^{(k)})\leq \frac{\pi^{(t)}(a''\rvert s)}{\pi^{(t)}(a'\rvert s)}\cdot \underbrace{\exp(-\eta)}_{<1}.
    \end{align}
    Similarly, for PPO-Clip-root, we have
    \begin{align}
    \label{eq:case 1 HPO-AM-root}
        \frac{\pi^{(t+1)}(a''\rvert s)}{\pi^{(t+1)}(a'\rvert s)}=\frac{\pi^{(t)}(a''\rvert s)}{\pi^{(t)}(a'\rvert s)}\cdot \prod_{k=0}^{K^{(t)}-1}\exp({\eta}g_{s,a'}^{(k)})\leq \frac{\pi^{(t)}(a''\rvert s)}{\pi^{(t)}(a'\rvert s)}\cdot \underbrace{\exp(-\frac{\eta}{2})}_{<1}.
    \end{align}
    Moreover, it is easy to verify that the arguments in Case 2 and Case 3 still hold under these two alternative classifiers. Hence, Lemma~\ref{lemma:pi convergence mini-batch} still holds.
    \item Regarding Lemma \ref{lemma:epsilon lower bound mini-batch}, we can reuse all the setup and slightly revise (\ref{eq:pi_t+1 1})-(\ref{eq:pi_t+1 4}) for the two alternative classifiers: For PPO-Clip-log, by (\ref{eq:g_sa^k HPO-AM-log}), we have
    \begin{align}
    \pi^{(t+1)}(a'\rvert s)& =\frac{\prod_{k=0}^{K^{(t)}-1}\exp({-\eta}g_{s,a'}^{(k)})}{{\prod_{k=0}^{K^{(t)}-1}}\langle w_s^{(k)},\widetilde{\theta}_s^{(k)}\rangle}\cdot\pi^{(t)}(a'\rvert s)\label{eq:pi_t+1 HPO-AM-log 1}\\
    &\geq \frac{\pi^{(t)}(a'\rvert s) \exp(-\eta g_{s,a'}^{(0)})}{\pi^{(t)}(a'\rvert s) \exp(-\eta g_{s,a'}^{(0)})+1}\label{eq:pi_t+1 HPO-AM-log 2}\\
    &\geq \frac{\pi^{(t)}(a'\rvert s) \exp(\eta/\pi^{(t)}(a'\rvert s))}{\pi^{(t)}(a'\rvert s) \exp(\eta/ \pi^{(t)}(a'\rvert s))+1}\label{eq:pi_t+1 HPO-AM-log 3}\\
    &\geq \frac{e\cdot\eta}{e\cdot\eta+1}. \label{eq:pi_t+1 HPO-AM-log 4}
    \end{align}
    Similarly, for PPO-Clip-root, by (\ref{eq:g_sa^k HPO-AM-root}), we have
    \begin{align}
    \pi^{(t+1)}(a'\rvert s)& =\frac{\prod_{k=0}^{K^{(t)}-1}\exp({-\eta}g_{s,a'}^{(k)})}{{\prod_{k=0}^{K^{(t)}-1}}\langle w_s^{(k)},\widetilde{\theta}_s^{(k)}\rangle}\cdot\pi^{(t)}(a'\rvert s)\label{eq:pi_t+1 HPO-AM-root 1}\\
    &\geq \frac{\pi^{(t)}(a'\rvert s) \exp(-\eta g_{s,a'}^{(0)})}{\pi^{(t)}(a'\rvert s) \exp(-\eta g_{s,a'}^{(0)})+1}\label{eq:pi_t+1 HPO-AM-root 2}\\
    &\geq \frac{\pi^{(t)}(a'\rvert s) \exp(\eta/2\pi^{(t)}(a'\rvert s))}{\pi^{(t)}(a'\rvert s) \exp(\eta/2 \pi^{(t)}(a'\rvert s))+1}\label{eq:pi_t+1 HPO-AM-root 3}\\
    &\geq \frac{e\cdot\frac{\eta}{2}}{e\cdot\frac{\eta}{2}+1}. \label{eq:pi_t+1 HPO-AM-root 4}
    \end{align}
    Accordingly, (\ref{eq:lower bound mini-batch case 1-2})-(\ref{eq:lower bound mini-batch case 1-4}) still go through and hence Lemma \ref{lemma:epsilon lower bound mini-batch} indeed holds under PPO-Clip-log and PPO-Clip-root.
\end{itemize}
In summary, since all the supporting lemmas hold for these alternative classifiers, we complete this part of the proof by obtaining a contradiction similar to that in Theorem \ref{thm:mini-batch}.
\end{proof}

\section{Detailed Configuration of the Experiments}
\label{app:Experiments}

\subsection{Experimental Settings}
For our experiments, we implement Neural PPO-Clip with different classifiers on the open-source RL baseline3-zoo framework \citep{rl-zoo3}. Specifically, we consider four different classifiers as follows: (i) $\rho_{s,a}(\theta) - 1$ (the standard PPO-Clip classifier); (ii) $\pi_{\theta}(a|s) - \pi_{\theta_t}(a|s)$ (PPO-Clip-sub); (iii) $\sqrt{\rho_{s,a}(\theta)} - 1$ (PPO-Clip-root); (iv) $\log(\pi_{\theta}(a|s)) - \log(\pi_{\theta_t}(a|s))$ (PPO-Clip-log). We test these variants in the MinAtar environments \citep{young19minatar} such as Breakout and Space Invaders. On the other hand, we evaluate them in OpenAI Gym environments \citep{brockman2016openai}, which are LunarLander, Acrobot, and CartPole, as well. For the comparison with other benchmark methods, we consider A2C and Rainbow. The training curves are drawn by the averages over 5 random seeds. For the computing resources we use to run the experiment, we use (i) CPU: Intel(R) Xeon(R) CPU E5-2630 v4 @ 2.20GHz; (ii)
GPU: NVIDIA GeForce GTX 1080.

\subsection{Model Parameters}

The neural networks architecture of policy and value function in the experiments share two full-connected layers and connect to respective output layers. We provide the parameters of the algorithms for each environment in the following tables \ref{tab:params in Breakout}-\ref{tab:params in Acrobot}. 
Notice that lin\_5e-4 means that the learning rate decays linearly from $5 \times 10^{-4}$ to 0. Also, the \texttt{vf\_coef} is the weight of the value loss and \texttt{temperature\_lambda} is the pre-constant of the adaptive temperature parameter for energy-based neural networks. We also give the parameters searching range in table \ref{tab: params search}.

\begin{table}[!htbp]
    \centering
    \caption[Short Title]%
    {Parameters for MinAtar Breakout experiments.}
    \begin{tabular}{|l|c|c|c|c|c|}
    \hline
    Hyperparameters & PPO-Clip & PPO-Clip-sub & PPO-Clip-root & PPO-Clip-log & A2C \\
    \hline
    batch\_size & 256 & 256 & 256 & 256 & 80 \\
    \hline
    learning\_rate & lin\_1e-3 & lin\_1e-3 & lin\_1e-3 & lin\_1e-3 & 7e-4\\
    \hline
    vf\_coef & 0.00075 & 0.00075 & 0.00075 & 0.00075 & 0.25 \\
    \hline
    EMDA step size & 0.005 & 0.005 & 0.005 & 0.005 & -\\
    \hline
    EMDA iteration & 2 & 2 & 2 & 2 & - \\
    \hline
    clipping range & 0.3 & 0.3 & 0.3 & 0.3 & - \\
    \hline
    temperature\_lambda & 25 & 25 & 25 & 25 & -\\
    \hline
    \end{tabular}
    \label{tab:params in Breakout}
\end{table}
\begin{table}[!htbp]
    \centering
    \caption[Short Title]%
    {Parameters for MinAtar Space Invaders experiments.}
    \begin{tabular}{|l|c|c|c|c|c|}
    \hline
    Hyperparameters & PPO-Clip & PPO-Clip-sub & PPO-Clip-root & PPO-Clip-log & A2C \\
    \hline
    batch\_size & 256 & 256 & 256 & 256 & 80 \\
    \hline
    learning\_rate & lin\_1e-3 & lin\_1e-3 & lin\_1e-3 & lin\_1e-3 & 7e-4\\
    \hline
    vf\_coef & 0.00075 & 0.00075 & 0.00075 & 0.00075 & 0.25 \\
    \hline
    EMDA step size & 0.005 & 0.005 & 0.005 & 0.005 & -\\
    \hline
    EMDA iteration & 5 & 5 & 2 & 5 & - \\
    \hline
    clipping range & 0.5 & 0.5 & 0.5 & 0.5 & - \\
    \hline
    temperature\_lambda & 10 & 10 & 10 & 10 & -\\
    \hline
    \end{tabular}
    \label{tab:params in Space Invaders}
\end{table}

\begin{table}[!htbp]
    \centering
    \caption[Short Title]%
    {Parameters for OpenAI Gym LunarLander-v2 experiments.}
    \begin{tabular}{|l|c|c|c|c|c|}
    \hline
    Hyperparameters & PPO-Clip & PPO-Clip-sub & PPO-Clip-root & PPO-Clip-log & A2C \\
    \hline
    batch\_size & 64 & 8 & 64 & 64 & 40 \\
    \hline
    learning\_rate & lin\_5e-4 & lin\_5e-4 & lin\_5e-4 & lin\_5e-4 & lin\_8.3e-4\\
    \hline
    vf\_coef & 0.75 & 0.75 & 0.75 & 0.75 & 0.5 \\
    \hline
    EMDA step size & 0.01 & 0.002 & 0.01 & 0.01 & -\\
    \hline
    EMDA iteration & 5 & 5 & 5 & 5 & - \\
    \hline
    clipping range & 0.3 & 0.5 & 0.3 & 0.3 & - \\
    \hline
    temperature\_lambda & 10 & 10 & 10 & 10 & -\\
    \hline
    \end{tabular}
    \label{tab:params in LunarLander}
\end{table}
\begin{table}[!htbp]
    \centering
    \caption[Short Title]%
    {Parameters for OpenAI Gym Acrobot-v1 experiments.}
    \begin{tabular}{|l|c|c|c|c|c|}
    \hline
    Hyperparameters & PPO-Clip & PPO-Clip-sub & PPO-Clip-root & PPO-Clip-log & A2C \\
    \hline
    batch\_size & 64 & 64 & 64 & 64 & 40 \\
    \hline
    learning\_rate & lin\_7.5e-4 & lin\_7.5e-4 & lin\_7.5e-4 & lin\_7.5e-4 & lin\_8.3e-4\\
    \hline
    vf\_coef & 0.5 & 0.5 & 0.5 & 0.5 & 0.5 \\
    \hline
    EMDA step size & 0.01 & 0.01 & 0.01 & 0.01 & -\\
    \hline
    EMDA iteration & 5 & 5 & 5 & 5 & - \\
    \hline
    clipping range & 0.3 & 0.3 & 0.3 & 0.3 & - \\
    \hline
    temperature\_lambda & 10 & 10 & 10 & 10 & -\\
    \hline
    \end{tabular}
    \label{tab:params in Acrobot}
\end{table}
\begin{table}[!htbp]
    \centering
    \caption[Short Title]%
    {Parameters for OpenAI Gym CartPole-v1 experiments.}
    \begin{tabular}{|l|c|c|c|c|c|}
    \hline
    Hyperparameters & PPO-Clip & PPO-Clip-sub & PPO-Clip-root & PPO-Clip-log & A2C \\
    \hline
    batch\_size & 64 & 64 & 64 & 64 & 40 \\
    \hline
    learning\_rate & lin\_7.5e-4 & lin\_7.5e-4 & lin\_7.5e-4 & lin\_7.5e-4 & lin\_8.3e-4\\
    \hline
    vf\_coef & 0.5 & 0.5 & 0.5 & 0.5 & 0.5 \\
    \hline
    EMDA step size & 0.01 & 0.01 & 0.01 & 0.01 & -\\
    \hline
    EMDA iteration & 5 & 5 & 5 & 5 & - \\
    \hline
    clipping range & 0.3 & 0.3 & 0.3 & 0.3 & - \\
    \hline
    temperature\_lambda & 10 & 10 & 10 & 10 & -\\
    \hline
    \end{tabular}
    \label{tab:params in CartPole}
\end{table}
\begin{table}[!htbp]
    \centering
    \caption[Short Title]%
    {Parameters searching range for the experiments.}
    \begin{tabular}{|l|c|}
    \hline
    Hyperparameters & Searching Range \\
    \hline
    batch\_size & 64, 128, 256 \\
    \hline
    learning\_rate & lin\_1e-3, lin\_7.5e-4, lin\_5e-4, lin\_2.5e-4\\
    \hline
    vf\_coef & 0.00075, 0.0005, 0.3, 0.5, 0.75, 0.8 \\
    \hline
    EMDA step size & 0.001, 0.005, 0.075, 0.02, 0.05, 0.01, 0.1 \\
    \hline
    EMDA iteration & 2, 5, 10 \\
    \hline
    clipping range & 0.3, 0.5, 0.7 \\
    \hline
    temperature\_lambda & 0.1, 0.5, 1, 5, 10, 25, 40, 60, 75\\
    \hline
    \end{tabular}
    \label{tab: params search}
\end{table}

\section{Supplementary Related Works}

\noindent \textbf{RL as Classification.} Regarding the general idea of casting RL as a classification problem, it has been investigated by the existing literature \citep{lagoudakis2003reinforcement,lazaric2010analysis,farahmand2014classification}, which view the one-step greedy update (e.g. in Q-learning) as a binary classification problem.
However, a major difference is the labeling: classification-based approximate policy iteration labels the action with the largest Q value as positive; Generalized PPO-Clip labels the actions with positive advantage as positive.
Despite the high-level resemblance, our paper is fundamentally different from the prior works \citep{lagoudakis2003reinforcement,lazaric2010analysis,farahmand2014classification} as our paper is meant to study the theoretical foundation of PPO-Clip, from the perspective of hinge loss.

\section{{Comparison of the Clipped Objective and the Generalized PPO-Clip Objective}}
\label{app:compare PPO-Clip and HPO}
Recall that the original objective of PPO-Clip is
\begin{equation}
    L^{\text{clip}}(\theta) =\mathbb{E}_{s\sim d_{\mu_{0}}^{\pi},a\sim\pi(\cdot|s)}\big[\min\{\rho_{s,a}(\theta)A^{\pi}(s,a), \text{clip}(\rho_{s,a}(\theta),1-\epsilon,1+\epsilon)A^{\pi}(s,a)\}\big],
\end{equation}
where $\rho_{s,a}(\theta)=\frac{\pi_{\theta}(a\rvert s)}{\pi(a\rvert s)}$.
In practice, $L^{\text{clip}}(\theta)$ is approximated by the sample average as
\begin{align}
    L^{\text{clip}}(\theta) &\approx \hat{L}^{\text{clip}}(\theta)=\frac{1}{\lvert \cD\rvert} \sum_{(s,a)\in \cD}\min\{\rho_{s,a}(\theta){A}^{\pi}(s,a), \text{clip}(\rho_{s,a}(\theta),1-\epsilon,1+\epsilon){A}^{\pi}(s,a)\}\\
    &=\frac{1}{\lvert \cD\rvert} \sum_{(s,a)\in \cD}\lvert {A}^{\pi}(s,a)\rvert \cdot\underbrace{\min\{\rho_{s,a}(\theta)\sgn({A}^{\pi}(s,a)), \text{clip}(\rho_{s,a}(\theta),1-\epsilon,1+\epsilon)\sgn({A}^{\pi}(s,a))\}}_{=:{H}^{\text{clip}}_{s,a}(\theta)}.
\end{align}
Note that ${H}^{\text{clip}}_{s,a}(\theta)$ can be further written as
\[    {H}^{\text{clip}}_{s,a}(\theta)=
\begin{cases}
        1+\epsilon&, \text{if } {A}^{\pi}(s,a)>0 \text{ and } \rho_{s,a}(\theta)\geq 1+\epsilon  \\
        \rho_{s,a}(\theta)&, \text{if } {A}^{\pi}(s,a)>0 \text{ and } \rho_{s,a}(\theta)<1+\epsilon\\
        -\rho_{s,a}(\theta)&, \text{if } {A}^{\pi}(s,a)<0 \text{ and } \rho_{s,a}(\theta)>1-\epsilon\\
        -(1-\epsilon)&, \text{if } {A}^{\pi}(s,a)<0 \text{ and } \rho_{s,a}(\theta)\leq 1-\epsilon\\
        0&, \text{otherwise}
    \end{cases}\]
{Recall that the generalized objective of PPO-Clip with hinge loss takes the form as}
\begin{align}
    L(\theta) &\approx \hat{L}(\theta)=\frac{1}{\lvert \cD\rvert} \sum_{(s,a)\in \cD}\lvert {A}^{\pi}(s,a)\rvert \cdot\underbrace{\max\big\{0,\epsilon-(\rho_{s,a}(\theta)-1)\sgn({A}^{\pi}(s,a))\big\}}_{=:{H}_{s,a}(\theta)}.
\end{align}
Similarly, ${H}_{s,a}(\theta)$ can be further written as
\[    {H}_{s,a}(\theta)=
\begin{cases}
        0&, \text{if } {A}^{\pi}(s,a)>0 \text{ and } \rho_{s,a}(\theta)\geq 1+\epsilon  \\
        -\rho_{s,a}(\theta)+(1+\epsilon)&, \text{if } {A}^{\pi}(s,a)>0 \text{ and } \rho_{s,a}(\theta)<1+\epsilon\\
        \rho_{s,a}(\theta)-(1-\epsilon)&, \text{if } {A}^{\pi}(s,a)<0 \text{ and } \rho_{s,a}(\theta)>1-\epsilon\\
        0&, \text{if } {A}^{\pi}(s,a)<0 \text{ and } \rho_{s,a}(\theta)\leq 1-\epsilon\\
        \epsilon&, \text{otherwise}
    \end{cases}\]
Therefore, it is easy to verify that $\hat{L}^{\text{clip}}(\theta)$ and $-\hat{L}(\theta)$ only differ by a constant with respect to $\theta$. This also implies that $\nabla_{\theta}\hat{L}^{\text{clip}}(\theta)= -\nabla_{\theta}\hat{L}(\theta)$.

\end{document}